\documentclass{article}
\usepackage[accepted]{icml2018x}

\usepackage{multicol}
\usepackage{subcaption}
\usepackage{amssymb, amsmath, amsthm}
\usepackage{microtype}
\usepackage{graphicx}
\usepackage{booktabs} %

\usepackage[hyphens]{url}

\usepackage[inline]{enumitem}
\usepackage{breakurl}
\usepackage{xcolor}
\usepackage[colorlinks,bookmarks=true,breaklinks=true,citecolor=olive,linkcolor=violet]{hyperref}
\usepackage[capitalize,nameinlink]{cleveref}
\usepackage{mymacros}
\usepackage{bbm}
\crefname{thm}{\text{Thm}}{\text{Thms}}
\crefname{assm}{\text{Assm}}{\text{Assms}}
\crefname{defn}{\text{Defn}}{\text{Defns}}
\crefname{cor}{\text{Cor}}{\text{Cors}}
\crefname{lemma}{\text{Lem}}{\text{Lems}}

\newcommand{\Gnodes}{\mathcal{G}}
\newcommand{\Hnodes}{\mathcal{H}}
\newcommand{\Gnodesin}{\Gnodes^{\mathrm{in}}}

\usepackage{thmtools}

\newcommand\gy[1]{}

\usepackage{bbding}

\newcommand{\distto}{\xrar{\mathrm{d}}}

\newcommand{\limas}{\lim^{\mathrm{a.s.}}}
\newcommand{\Vt}{\mathrm{V}}

\newcommand{\onev}{\mathbbm{1}} %
\newcommand{\onem}{\onev \onev^T} %

\newcommand{\trsp}{\top}
\newcommand{\gvar}{\mathtt{g}}
\newcommand{\avar}{\mathtt{a}}
\newcommand{\hvar}{\mathtt{h}}
\newcommand{\Avar}{\mathtt{A}}

\newcommand{\fvar}{\mathtt{f}}
\newcommand{\nvar}{\mathtt{n}}

\newcommand{\batchnorm}{\mathfrak{B}}

\newcommand{\disteq}{\overset{\mathrm{d}}{=}}

\newcommand{\cdc}{\mathfrak{c}}
\newcommand{\Cdc}{\mathfrak{C}}
\newcommand{\lno}{\mathsf{line}}

\newcommand{\varphih}{\varphi_{\hvar}}
\newcommand{\varphig}{\varphi_{\gvar}}
\newcommand{\defeq}{\mathbin{\overset{\mathrm{def}}{=}}}

\newcommand{\Jac}[2]{\f{\dd #1}{\dd #2}}

\newcommand{\asto}{\xrar{\mathrm{a.s.}}}
\newcommand{\probto}{\xrar{\mathrm{p}}}

\newcommand{\PSD}{\mathrm{PSD}}

\newcommand{\rmin}{{\mathrm{in}}}
\newcommand{\cdcin}{{\cdc_\rmin}}

\makeatletter
\let\orgdescriptionlabel\descriptionlabel
\newcommand*{\@restrictlabeltext}[1]{#1\protected@edef\@currentlabel{#1}}
\newcommand*{\nolabel}[1]{#1}%
\renewcommand*{\descriptionlabel}[1]{%
  \let\orglabel\label
  \let\label\@gobble
  \let\orig@hfil\hfil
  \def\hfil{}%
  \let\nolabel\@gobble
  \let\restrictlabeltext\@firstofone
  \phantomsection
  \protected@edef\@currentlabel{#1}%
  \let\hfil\orig@hfil
  \let\label\orglabel
  \let\restrictlabeltext\@restrictlabeltext
  \orgdescriptionlabel{#1}%
}
\makeatother

\icmltitlerunning{Scaling Limits of Wide Neural Networks with Weight Sharing}

\begin{document}

\onecolumn
\icmltitle{Scaling Limits of Wide Neural Networks with Weight Sharing:\\
Gaussian Process Behavior, Gradient Independence, and Neural Tangent Kernel Derivation}

\begin{icmlauthorlist}
\icmlauthor{Greg Yang}{MSR}
\end{icmlauthorlist}

\icmlaffiliation{MSR}{Microsoft Research AI}

\icmlcorrespondingauthor{$\la$gregyang@microsoft.com$\ra$}{}

\vskip 0.3in

\setlength{\footskip}{20pt}
\thispagestyle{plain}
\pagestyle{plain}

\printAffiliationsAndNotice{}  %

\begin{abstract}
Several recent trends in machine learning theory and practice, from the design of state-of-the-art Gaussian Process to the convergence analysis of deep neural nets (DNNs) under stochastic gradient descent (SGD), have found it fruitful to study wide random neural networks.
Central to these approaches are certain scaling limits of such networks.
We unify these results by introducing a notion of a straightline {\it tensor program} that can express most neural network computations, and we characterize its scaling limit when its tensors are large and randomized.
From our framework follows
\begin{enumerate*}
    \item 
        the convergence of random neural networks to Gaussian processes for architectures such as recurrent neural networks, convolutional neural networks, residual networks, attention, and any combination thereof, with or without batch normalization;
    \item
        conditions under which the \emph{gradient independence assumption} -- that weights in backpropagation can be assumed to be independent from weights in the forward pass -- leads to correct computation of gradient dynamics, and corrections when it does not;
    \item
        the convergence of the Neural Tangent Kernel, a recently proposed kernel used to predict training dynamics of neural networks under gradient descent, at initialization for all architectures in (1) without batch normalization.
\end{enumerate*}
Mathematically, our framework is general enough to rederive classical random matrix results such as the semicircle and the Marchenko-Pastur laws, as well as recent results in neural network Jacobian singular values.
We hope our work opens a way toward design of even stronger Gaussian Processes, initialization schemes to avoid gradient explosion/vanishing, and deeper understanding of SGD dynamics in modern architectures.
\end{abstract}

\section{Introduction}

Several recent trends in machine learning theory and practice have found it fruitful to study wide random neural networks, such as neural network inspired Gaussian Processes, signal propagation in DNNs, small learning rate SGD dynamics, and even, in some sense, the celebrated Approximate Message Passing algorithm for compressed sensing.
We review these subjects and more in \cref{sec:related}.
All of these works involve some theory that derives, rigorously or semirigorously, some scaling limit of a neural network as its width goes to infinity.
In this paper, we give a unifying treatment to such scaling limits:
\begin{itemize*}
    \item
        We define a notion of {\it tensor programs} which can express most neural network computations, and a natural notion of tensor program scaling that corresponds to increasing width with Glorot initialization \cite{glorot_understanding_2010}.
        Our main theorems characterize the scaling limits in the two most common scenarios that roughly correspond to DNN inference and backpropagation, as well as in the general tensor program case.
        They are proved via a Gaussian conditioning technique first used in \citet{bolthausen_iterative_2012} for analyzing the TAP equations in spin glass theory.
    \item
        We obtain corollaries that fully justify semirigorous derivations in prior works and strengthen previous results in the different strands of research mentioned above.
        In the next section we highlight the most important corollaries and discuss other briefly, leaving their details to the appendix.
\end{itemize*}

By \emph{standard architecture} we mean any DNN architecture that is some composition of multilayer perceptrons (MLP)s, recurrent neural networks (RNNs) (e.g., Long-Short Term Memory (LSTM) \cite{hochreiter_long_1997} or Gated Recurrent Unit (GRU) \cite{cho_learning_2014}), skip connections \cite{he_deep_2016,huang_densely_2016}, (self-)attention \cite{bahdanau_neural_2014,vaswani_attention_2017}, convolution
\cite{lecun_gradient-based_1998,lecun_object_1999}, and/or batch normalization (batchnorm) \cite{ioffe_batch_2015}.
We use \emph{readout layer} to mean any linear layer converting some hidden states to an output vector.
While most of our corollaries are stated for standard architectures, they are typically more general, but we just highlight the most relevant cases for a deep learning audience.

\section{Related Works and Our Corollaries}
\label{sec:related}
We formulate informal versions of our main corollaries and comment on other results inline, marked by a star $\bigstar$.
\subsection{Gaussian Behavior of Wide Neural Networks}
    \label{subsec:GPNN}
    In 1995, \citeauthor{neal_bayesian_1995} first discovered the Gaussian Process behavior of wide neural networks.
    He showed that under certain conditions, a single-layer neural network with random parameters can converge in distribution to a Gaussian process as its width goes to infinity.
    Later works extended the conditions under which this scaling limit takes place \citep{williams_computing_1997,le_roux_continuous_2007,hazan_steps_2015}.
    Recently, \citet{lee_deep_2018,matthews_gaussian_2018} empirically and/or theoretically investigated analogous correspondences for infinite width, finite depth \emph{deep} MLPs, and likewise \citet{novak_bayesian_2018}, for deep \emph{convolution} networks.
    \citet{daniely_toward_2016} also proved similar results in the framework of kernel methods, where they introduced a notion of ``computational skeleton,'' similar to \emph{tensor programs} introduced here, that covers feedforward computation with no weight-sharing (so that, for example, it can express locally connected networks but not convolutional networks)\footnote{even though they claim that dealing with weight-tying is straightforward.
    It's unclear what they had in mind, however, as there is a significant difference in the scaling behavior of sharing matrix transposes vs sharing no matrix transposes (see \cref{thm:gradIndep} and \cref{thm:generalTensorP})}.
    
    Many previous works have exploited this DNN-GP correspondence implicitly or explicitly to build new models \citep{cho_kernel_2009,lawrence_hierarchical_2007,damianou_deep_2013,wilson_stochastic_2016,wilson_deep_2016,bradshaw_adversarial_2017,van_der_wilk_convolutional_2017,kumar_deep_2018,blomqvist_deep_2018,borovykh_gaussian_2018}.
    In particular,  \citet{lee_deep_2018,garriga-alonso_deep_2018,novak_bayesian_2018} directly converted DNN to GP using this correspondence.
    \citet{lee_deep_2018} constructed the state-of-the-art (SOTA) permutation-invariant GP on MNIST, and \citet{novak_bayesian_2018} achieved SOTA on CIFAR10 for any GP with untrainable kernel.
    
    In this paper, we generalize the DNN-GP correspondence to \emph{standard architectures} and very general nonlinearities.
    \begin{cor}[DNN-GP correspondence, informal]\label{cor:DNNGP}
    Let $f$ be a network of fixed standard architecture, with linear readout layer, and with nonlinearities bounded uniformly by $\exp(O(x^{2-\epsilon}))$ for some $\epsilon>0$.
    Fix a finite input set $\mathcal X$ of the right signature (e.g. set of batches for batchnorm network; set of sequences for RNN).
    Sampling $f$'s parameters from iid Gaussians induces a distribution of functions on $\mathcal X$.
    If the readout layer weights are sampled independently from hidden parameters, then this distribution weakly converges to a Gaussian process as the network widths go to infinity (with fixed input and output dimensions).
    See \cref{subsec:warmupConseq,subsec:DNNGP}.
    \end{cor}

    In contrast, \citet{matthews_gaussian_2018} requires $\phi$ to be linearly bounded in norm;
    \citet{daniely_toward_2016} requires $\phi$ be twice-differentiable with $|\phi|,|\phi'|, |\phi''|$ all bounded, or that $\phi=$ ReLU;
    and a sufficient condition given in \citet{novak_bayesian_2018} is that $\phi'$ exists and is bounded by $\exp(O(x^{2-\epsilon}))$, though it is unclear how the more general set of 3 conditions given there compares with ours.

    We hope this corollary opens the door to the design of more powerful GPs, in the way of \citet{lee_deep_2018,novak_bayesian_2018} by converting state-of-the-art DNNs. \footnote{For a gentler introduction to these GP results and several extensions, we recommend the reader to look at \citet{yang_GP_2019}.}
    
\subsection{Signal Propagation in Neural Networks}
    \label{subsec:signalprop}
    \citet{glorot_understanding_2010,he_delving_2015} derived the popular \emph{Glorot} and \emph{He} initializations from consideration of hidden state norms in a DNN with random weights.
    A recent line of work generalizes their studies significantly by examining the evolution with depth of covariance between $f(x^i), f(x^j)$ and between $\nabla_x f (x^i), \nabla_x f (x^j)$ for distinct inputs $x^i$ and $x^j$, when the DNN $f$ is wide and parameters of $f$ are randomized.
    This evolution is referred to as \emph{(forward and backward) signal propagation} in the literature \citep{poole_exponential_2016,schoenholz_deep_2017,yang_mean_2017,yang_deep_2018,hanin_how_2018,chen_dynamical_2018,yang_mean_2018,pennington_resurrecting_2017}.
    It has been used to optimize initialization hyperparameters to prevent gradient explosion/vanishing, even to allow training of a 10,000 layer CNN without batchnorm or skip connections \citep{xiao_dynamical_2018}.
    
    Suppose $\{x^i\}_i$ is a set of inputs.
    Let $f^l$ be an $l$-layer MLP with activation $\phi$ and uniform width $n$.
    If $\Sigma^l_{ij} = \f 1 n \EV \la f^l(x^i), f^l(x^j)\ra$ and $\Pi^l_{ij} = \f 1 n \EV \la \nabla_x f^l(x^i), \nabla_x f^l(x^j) \ra $, with expectation taken over $W^\ell_{ij} \sim \Gaus(0, \sigma_w^2/n), b^\ell_i \sim \Gaus(0, \sigma_b^2)$ for each layer $\ell$, then the signal propagation literature posits that, in the $n\to \infty$ limit, the dynamics of $\Sigma^l$ and $\Pi^l$ are summarized by
    \begin{align}
        \Sigma^l
            &=
                \sigma_w^2 \EV[\phi(z)\phi(z)^\trsp: z \sim \Gaus(0, \Sigma^{l-1})] + \sigma_b^2
                \label{eq:forwardMF}
                \\
        \Pi^{l}
            &=
                \sigma_w^2\EV[\phi'(z)\phi'(z)^\trsp: z \sim \Gaus(0, \Sigma^l)] \odot \Pi^{l-1}.
                \label{eq:backwardMF}
    \end{align}
    Note that $\Sigma^l$ essentially is the kernel of the corresponding GP, and \cref{eq:forwardMF} is the same one used in the DNN-GP correspondence.
    \citet{pennington_resurrecting_2017} more generally computed the singular value distribution of the input-output Jacobian matrix of an MLP and characterized conditions under which this distribution concentrates around 1.
    To make this computation and to derive \cref{eq:backwardMF}, they and others \citep{schoenholz_deep_2017,yang_mean_2017,yang_deep_2018,chen_dynamical_2018,xiao_dynamical_2018,pennington_resurrecting_2017} relied on 
    \begin{assm}[Gradient Independence Assumption]\label{assm:gradInd}
        In backpropagation, whenever we multiply by $W^\trsp$ for some weight matrix $W$, we multiply by an iid copy instead.
    \end{assm}
    that was first discovered by \citet{schoenholz_deep_2017} to make good predictions and later formulated explicitly by \citet{yang_mean_2017}.
    In this paper we show
    \begin{cor}[\cref{assm:gradInd} is conditionally correct, informal]\label{cor:gradInd}
        In a MLP having nonlinearities with polynomially bounded weak derivatives, \cref{assm:gradInd} leads to the correct equation \cref{eq:backwardMF} and the correct singular value distribution computation from \citet{pennington_resurrecting_2017}, as long as the readout layer is sampled independently from other parameters and has mean 0.
        In general, \cref{assm:gradInd} does not induce correct computations -- for example when the last layer is global mean pooling -- and we rigorously give the correct equations, and more generally a way to compute the singular value distribution of the neural network Jacobian, both generalized to all standard architectures without batchnorm.
        See \cref{{subsec:warmupConseq},subsec:gradIndAssm}.
    \end{cor}
    
    As an example, we computed the scaling limit for the gradient norms of an LSTM and compared it against empirical simulation (\cref{fig:LSTM}).
    The theoretical prediction is very precise already for 1000 neurons, which is typical for applications of LSTM.
    
    Note that this literature also studies the limit of iterating \cref{eq:forwardMF,eq:backwardMF} (large depth limit), but our results only apply to a fixed number of iterations, and so do not rigorously justify such limits.
    
    \citet{chen_dynamical_2018} estimates the signal propagation in tied-weights RNNs with the equations for that in untied-weights RNNs.
    They find this a fantastic approximation for simple RNNs but not quite so for gated RNNs.
    $\bigstar$ As a corollary of \cref{thm:notransposeLimit} we show that, indeed, the tied- and untied-weights theories agree for simple RNNs, but not for general (say, gated) RNNs.
    We give the simplest counterexample of weight-tied residual network.
    See \cref{subsec:signalProp}.
    
    Recently, \citet{li_random_2018} investigated (forward only) signal propagation in weight-tied autoencoders.
    $\bigstar$ A version of their main theorem allowing for arbitrary polynomially bounded activations, without restriction on smoothness, also follows as corollary of \cref{thm:generalTensorP}.
    See \cref{subsec:signalProp}.
    
    We hope \cref{cor:gradInd} will allow future works to optimize initialization hyperparameters and prevent gradient explosion/vanishing problems for modern architectures in the way of \citet{schoenholz_deep_2017,yang_mean_2017,yang_deep_2018,chen_dynamical_2018,xiao_dynamical_2018,yang_mean_2018}.

\begin{figure}
    \centering
    \includegraphics[width=0.5\textwidth]{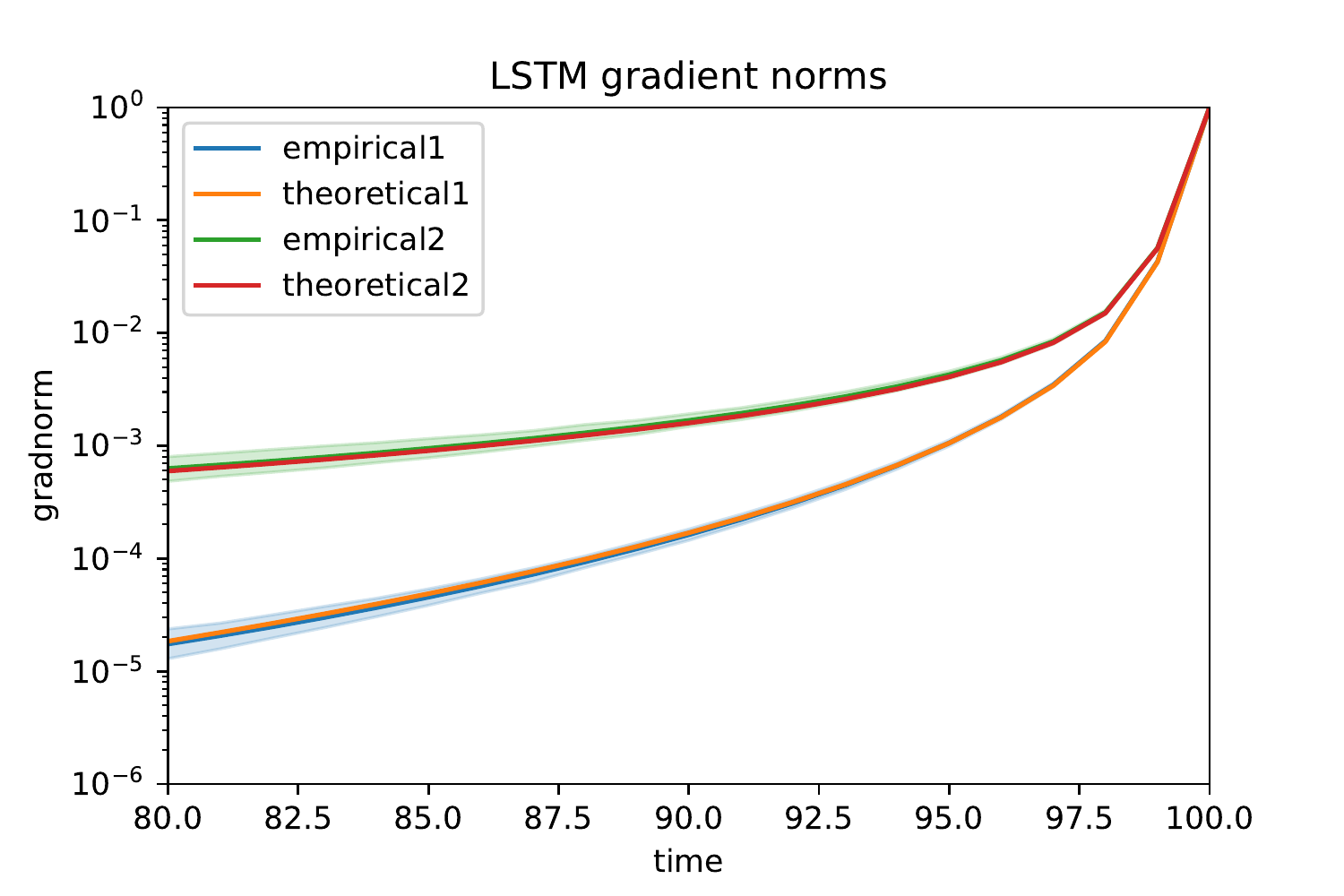}
    \caption{For two different random initializations of the LSTM with 1000 neurons along with a readout layer (at each time step), we run 100 steps forward on inputs of zero vectors and 20 steps of backpropagation-through-time.
    We collect the gradient norms $\|\pd y^{100}/\pd h^{t}\|, t = 80, \ldots, 100,$ where $y$ is the readout and $h$ is the hidden state, and plot its mean and standard deviations (\emph{empirical1} and \emph{empirical2}).
    Then we compute the large width limit of these gradient norms according to \cref{thm:gradIndep} and overlaid them on top (\emph{theoretical1} and \emph{theoretical2}).
    The agreement is so precise that the theoretical curves completely obscure the means of the empirical data.}
    \label{fig:LSTM}
\end{figure}
\subsection{Neural Tangent Kernel}
    \label{subsec:NTK}
    \newcommand{\trainset}{{\mathcal X}}
    For any parametrized function $f(x; \theta)$, the Neural Tangent Kernel can be in general defined as $K_\theta(x, x') = \la \nabla_\theta f(x; \theta), \nabla_\theta f(x'; \theta) \ra$ \citep{jacot_neural_2018}.
    In the case when $f(x; \theta)$ is a feedforward neural network, with parameters appropriately scaled (see \cref{subsec:warmupConseq}), there is a scaling limit of $K_\theta \to K_\infty$ when $\theta$ is randomized and $f$'s widths grow to infinity \cite{jacot_neural_2018}.
    This convergence allows one to predict the evolution of $f(x; \theta)$ due to gradient descent on $\theta$.
    For example, if we apply gradient flow on a training set $\trainset$ and loss function $\f 1 {|\trainset|} \sum_{(x,y) \in \trainset} \f 1 2 (f(x) - y)^2$, for $\mathrm{codomain}(f) = \R$, \citet{jacot_neural_2018} derived
    \begin{align*}
        \pdf {f_t} t
            &=
                - 
                \f 1 {|\trainset|}
                K_{\theta_t}(\trainset, \trainset)(f_t - f^*)
    \end{align*}
    where $f^*$ is the ``ground truth`` function that sends $x \mapsto y$ for every $(x, y) \in \trainset$, and $f$ and $f^*$ are thought of dimension $|\trainset|$ vectors.
    \citet{jacot_neural_2018} proved that under suitable conditions, with training time $T$ fixed and width $\to \infty$, $K_{\theta_t}(\trainset, \trainset) \to K_\infty(\trainset, \trainset)$ for all $0 \le t \le T$.
    This means that, in the large width regime, $f$ (in the function space) evolves approximately according to a linear differential equation under gradient flow.
    In this paper we show
    \begin{cor}[NTK convergence, informal]\label{cor:NTKconverge}
    Fix a finite input set $\mathcal X$.
    Let $f$ be a network of fixed standard architecture, with linear readout layer, and having nonlinearities with polynomially bounded weak derivatives (so in particular cannot have batchnorm).
    Then over $\mathcal X$, $K_\theta \to K_\infty$ almost surely as the widths of $f$ go to infinity and $\theta$ suitably randomized, for some $K_\infty$.
    See \cref{subsec:warmupConseq,subsec:NTKAppendix}.
    \end{cor}

    While \citet{jacot_neural_2018} is groundbreaking in producing an equation to predict the behavior of gradient descent in the small learning rate, large width regime, its proof of the convergence $K_\theta \to K_\infty$ relies on taking the widths to infinity one by one starting from the first layer\footnote{An earlier version of this paper claimed that it assumes gradient independence; this is incorrect, as the sequential limit obviates the need for it}. This is an unrealistic limit, as in practice, the widths are typically of the same order of magnitude\footnote{Indeed, when the last layer is global mean pooling rather than random Gaussians, the sequential limit obtains a different answer than simultaneous limit}.
    In comparison, \cref{cor:NTKconverge} proves that the limit exists as the widths tend to infinity together, and it generalizes to arbitrary standard architectures.
    $\bigstar$ We give an example computation of the NTK for a CNN in \cref{subsec:NTKAppendix}; this is a new result that has not appeared in prior literature.

    \citet{amari_fisher_2018,karakida_universal_2018} recently used \cref{assm:gradInd} to study the empirical Fisher information matrix (FIM), over finitely many datapoints drawn from isotropic Gaussian, of random neural networks, specifically its spectral properties.
    If we let $Jf$ be the $|\mathcal X| \times |\theta|$ matrix whose rows are $\{\nabla_x f(x; \theta)\}_{x \in \mathcal X}$, then (empirical) FIM $\propto Jf\ Jf^\trsp$ while NTK is $Jf^\trsp Jf$ \footnote{\citet{karakida_universal_2018} called NTK the \emph{dual matrix}}.
    Thus the spectral properties of empirical FIM and NTK are identical up to scaling.
    $\bigstar$ By \cref{cor:NTKconverge}, we can then justify the computations of \citet{amari_fisher_2018,karakida_universal_2018} rigorously.

\subsection{Other Works}

Very recently, \citet{du_gradient_2018,du_gradient_2018-1,allen-zhu_convergence_2018,allen-zhu_convergence_2018-1,zou_stochastic_2018} formally proved that GD or SGD can reduce an overparametrized DNN's training error to 0 by showing that random initialization imbues the network with certain good properties\footnote{using tools similar to ones in the signal propagation literature} and, with small learning rate, the network never moves too far from its initialization\footnote{using reasoning similar to \citet{jacot_neural_2018}}.
\citet{allen-zhu_learning_2018,cao_generalization_2019} also show generalization bounds under various data assumptions using a similar reasoning.

There is a long line of work investigating random classic spiking or hopfield networks, for example \citet{landau_coherent_2018,crisanti_path_2018,kadmon_transition_2015,stern_dynamics_2014,rajan_stimulus-dependent_2010,sompolinsky_chaos_1988,amit_spin-glass_1985}.
In the reinforcement learning literature, \citet{osband_randomized_2018,burda_large-scale_2018,burda_exploration_2018} used random DNNs for exploration.
Other than the works discussed above, \citet{li_exploring_2018,giryes_deep_2016,gabrie_entropy_2018,reeves_additivity_2017,fletcher_inference_2017} also considered neural networks with random weights.

$\bigstar$ Our technique is general enough to rederive the semicircle law for the Gaussian Orthogonal Ensemble and the Marchenko-Pastur Law for Wishart matrices \cite{tao_topics_2012}.
See
\cref{warmup:semicirclelaw,warmup:marchenkopasturlaw}.

Approximate Message Passing is an algorithm for recovering a ground truth vector from noisy measurements (\emph{Compressed Sensing}) \citep{donoho_message_2009}.
In one view, the algorithm repeatedly applies a certain neural network to the noisy measurement, and it succeeds if the result eventually converges to the ground truth vector.
Previous works have shown that when the measurement matrix is randomized and the dimension goes to infinity, this algorithm satisfies a set of equations called State Evolution that can be used to reason about its behavior \citep{bayati_dynamics_2011,berthier_state_2017}.
Their proofs are based on the same Gaussian conditioning technique used here.
$\bigstar$ In \cref{subsec:AMP}, we detail the algorithm and State Evolution, and prove the validity of State Evolution equations for arbitrary polynomially bounded nonlinearities and test functions, removing the smoothness assumption of \citet{bayati_dynamics_2011} (in exchange for a stronger moment condition on the measurements).

This concludes the discussion of related works and our corollaries.
We now present the tensor program framework and our main theorems.

\section{Tensor Programs}
\label{sec:tensorPrograms}
Consider programs of the following form, which we call \textit{tensor programs}.
Each line contains an assignment and a dimension annotation and can have the following types.
\label{defn:linetypes}
\begin{description}
    \item[VecIn\label{linetype:G}] (G) a vector input $x$
    $$l: \gvar^l := x \in \R^{\nvar^l}$$
    \item[MatIn\label{linetype:A}] (A) a matrix input $A$
    $$l: \Avar^l := A \in \R^{\nvar^l_1 \times \nvar^l_2}$$
    \item[T\label{linetype:trsp}] (A) transpose of an A-var
    $$l: \Avar^l := (\Avar^j)^\trsp \in \R^{\nvar^l_1 \times \nvar^l_2} = \R^{\nvar^j_2 \times \nvar^j_1}$$
    \item[MatMul\label{linetype:GLinear}] (G) if $\Avar^k$ and $\gvar^j$ have $\nvar^k_2 = \nvar^j$, then an assignment via a linear mapping
    $$l: \gvar^l := \Avar^k \gvar^j \in \R^{\nvar^l} = \R^{\nvar^k_1}$$
    or similarly for H-vars
    $$l: \gvar^l := \Avar^k \hvar^j \in \R^{\nvar^l} = \R^{\nvar^k_1}$$
    where $j,k < l$
    \item[LinComb\label{linetype:lincomb}] (G) if $\nvar^{j_1} = \cdots = \nvar^{j_k}$, then an assignment via linear combination of G-vars that appeared in previous lines: with $\avar^l_{j_i} \in \R$,
    $$l: \gvar^l := \avar^l_{j_1}\gvar^{j_1} + \cdots + \avar^l_{j_k}\gvar^{j_k} \in \R^{\nvar^l} = \R^{\nvar^{j_1}}.$$
    \item[Nonlin\label{linetype:nonlin}] (H) if $\nvar^{j_1} = \cdots = \nvar^{j_k}$, then an assignment via some general (possibly nonlinear) function $\fvar^l: \R^k \to \R$, acting coordinatewise,
    $$l: \hvar^l := \fvar^l(\gvar^{j_1}, \ldots, \gvar^{j_k}) \in \R^{\nvar^l} = \R^{\nvar^{j_1}}.$$
\end{description}
Here (G) marks those variables that we call \textbf{G-vars}, and similarly we have \textbf{A-vars} and \textbf{H-vars}.
Vars introduced by \ref{linetype:G} and \ref{linetype:A} are also called {\it (vector and matrix) input vars}.
The initial ``$l:$'' marks the line number, and each new variable formed from this line is labeled with a superscript $l$.
A partial program with $\nvar^l$ and input G- and A-vars unspecified is called a {\it (program) skeleton}, typically denoted by Greek letters like $\pi$.
This skeleton can be thought of as a generalization of the skeleton in \citet{daniely_toward_2016} in the language of a straightline program that allows weight sharing (transposed or not) and simple type checking.

\subsection{Examples}
Such tensor programs can express the computation in most neural network scenarios.
In \cref{sec:examplePrograms}, we give example programs for computations in
\begin{enumerate*}[label=\emph{(\arabic*)}]
    \item MLP, forward and backward passes (\ref{subsec:exampleFCFF});
    \item batch of input (\ref{subsec:exampleBatch});
    \item residual networks (\ref{subsec:exampleResnet});
    \item simple RNN (\ref{subsec:exampleRNN});
    \item batchnorm (\ref{subsec:exampleBatchnorm});
    \item CNNs (\ref{subsec:exampleConv}).
\end{enumerate*}
It's also clear from these examples that any combination of such computation can be expressed faithfully in a tensor program.
On the other hand, tensor programs don't capture all neural network computations, and one example is layer normalization \cite{ba_layer_2016}, but see \cref{sec:discussion} for how to still compute its scaling limit in this framework.

\subsection{Setup}
\label{subsec:setup}

Lines of type \ref{linetype:trsp}, \ref{linetype:GLinear}, \ref{linetype:lincomb}, and~\ref{linetype:nonlin} induce equality constraints on the dimensions $\nvar^l$.
Given a skeleton $\pi$ and a possible set of additional dimensional constraints $\Lambda \sbe \{`` \nvar^l = \nvar^m"\}_{\gvar^l,\gvar^m}$, consider the smallest equivalence relation $\sim$ on G-vars such that $\gvar^l \sim \gvar^m$ if $`` \nvar^l = \nvar^m" \in \Lambda$ or if they are constrained to have equal dimension by some line of type \ref{linetype:trsp}, \ref{linetype:GLinear}, \ref{linetype:lincomb}, or \ref{linetype:nonlin}.
We call each class a \textit{common dimension class} (CDC) of $(\pi, \Lambda)$ and write $\cdc(\gvar^l)$ for the class of a G-var $\gvar^l$.
The collection of all common dimension classes is written as $\Cdc_{(\pi, \Lambda)},$ or just $\Cdc$ when $\pi$ and $\Lambda$ are understood from context.
An algorithm to compute CDCs is presented in \cref{sec:CDC}.

In this work, for every skeleton $\pi$ (equipped with $\Lambda$), we study the behavior of vars in its realizations when the input vars are appropriately randomized and as the dimensions $\nvar^l \to \infty$.
More precisely, we consider a sequence (in $t \in \N$) of dimensions $\{\nvar^{l t}\}_{\gvar^l or \hvar^l} \cup \{\nvar^{lt}_1, \nvar^{lt}_2\}_{\Avar^l}$ respecting $\sim$, along with input G- and A-vars $\gvar^{lt}$, $\Avar^{lt}$ of appropriate dimensions.
For each $\cdc \in \Cdc$, let $\nvar^{\cdc t} = \nvar^{lt}$ for $\gvar^l \in \cdc$.
We extend the notations $\gvar^{lt}$ and $\hvar^{lt}$ to the non-input G- and H-vars computed from these inputs.

At time $t$, we sample independently $\Avar^{lt}_{ij} \sim \Gaus(0, (\sigma^{lt})^2/\nvar^{lt}_2)$ for a set $\{\sigma^{lt}\}_{\Avar^l}$
\footnote{
We could as well assume that there is an infinite 2D array of independent Gaussian variables $\{\mathring A^l_{ij} \sim \Gaus(0, 1)\}_{i,j=1}^\infty$, and at time $t$, set $\Avar^{l t}_{ij} = \sigma^{lt} \mathring A^l_{ij} / \sqrt{\nvar^{lt}_2}$.
In that case, we do not need $\nvar^{\cdc t}$ to increase stricty with $t$}.
For each $\cdc \in \Cdc$, we also sample independently $\gvar^{\cdcin t}_i \sim \Gaus(\mu^{\cdcin t}, K^{\cdcin t})$ for each $i, j$.
Here $\cdcin$ is the set of input G-vars in $\cdc$, $\gvar^{\cdcin t}_i = (\gvar^{lt}_i)_{\gvar^l \in\cdcin}$, and $\mu^{\cdcin t}: \cdcin \to \R, K^{\cdcin t}: \cdcin \times \cdcin \to \R$ are specified mean and covariance at time $t$.

Thus given $(\pi, \Lambda)$, the data $\{\nvar^{\cdc t}\}_{\cdc \in \Cdc},$ $\{\sigma^{lt}\}_{\Avar^l}$, $\{\mu^{\cdcin t}\}_{\cdc \in \Cdc}$, and $\{K^{\cdcin t}\}_{\cdc \in \Cdc}$ realize a random program $\pi(\{\nvar^{\cdc t}\}_{\cdc \in \Cdc}, \{\sigma^{lt}\}_{\Avar^l}, \{\mu^{\cdcin t}\}_{\cdc \in \Cdc}, \{K^{\cdcin t}\}_{\cdc \in \Cdc})$.
Its vars are random variables and our theorems will concern certain ``moments'' of them.

We assume that, as $t \to \infty$, for all $\cdc, \cdc' \in \Cdc$:
\begin{enumerate*}[label={(\arabic*)}]
    \item $\nvar^{\cdc t}$ is increasing with $t$ and $\nvar^{\cdc t} \to \infty$.
    \item $\lim_{t \to \infty} \nvar^{\cdc t}/\nvar^{\cdc' t} = \alpha_{\cdc, \cdc'} \in (0, \infty)$ for some constant $\alpha_{\cdc, \cdc'}$ depending only on $\cdc, \cdc'$.
    \item  $\sigma^{lt} \to \sigma^{l\infty}$ for some finite $\sigma^{l\infty} >0$ for each input A-var $\Avar^l$.%
    \item $\mu^{\cdcin t} \to \mu^{\cdcin \infty}$ and $K^{\cdcin t} \to K^{\cdcin \infty}$ for some finite $\mu^{\cdcin \infty}, K^{\cdcin \infty}$, and $\rank K^{\cdcin t} = \rank K^{\cdcin \infty}$ for all large $t$.
\end{enumerate*}

\paragraph{Discussion}
Tensor programs are meant to represent the ``body'' of a neural network where all dimensions are large compared to input and output dimensions.
The CDCs are used to capture the varying widths of practical neural networks; for example, while widths typically increase and decrease in classical networks, they are held constant in blocks in residual networks (see \cref{subsec:exampleResnet} for an example).
For the first read-through, we recommend the reader to assume all dimensions are the same so that there is a single CDC consisting of all G-vars.

The sampling of A-vars reflects variants of Glorot initialization \cite{glorot_understanding_2010} used in practice.
The sampling of input G-vars models the distribution of the first hidden layer across multiple inputs, sampling of the first layer parameters (see \cref{subsec:exampleFCFF} for an example), and/or sampling of bias vectors.
Most often, the vector vars should be thought of as hidden layer quantities whose dimensions go to infinity; neural network inputs (of fixed dimension) are indirectly expressed as above, and outputs (of fixed dimension) are obtained as some coordinates of a vector var.

\cref{thm:notransposeLimit,thm:gradIndep} below say that, under certain conditions, G-vars converge to Gaussians of specific mean and covariances (hence the name ``G-var'').
But \cref{thm:generalTensorP} shows that in general this may not be true.

\paragraph{Notation}
We will often identify functions $Z: A \to B$ with vectors in $B^A$ (which should be thought of as dictionaries with keys in $A$).
Given a subset $U \sbe A$, $Z^U$ is the subvector of $Z$ supported on $U$, or as a function is the restriction of $Z$ on $U$.
For $\psi: B^U \to C$, $\psi(Z) \defeq \psi(Z^U)$, i.e. we automatically ignore the values of $Z$ outside $U$.
We use $\asto$ for convergence almost surely.

\section{Programs with No Transposes}
\label{sec:notrsp}
For any $\cdc \in \Cdc$, we recursively define
\begin{align}
    \mu^\cdc(\gvar^l)
        &=
            \begin{cases}
                \mu^{\cdcin\infty}(\gvar^l)   &   \text{if $\gvar^l \in \cdcin$}\\
                \sum_i \avar^l_{ji} \mu^\cdc(\gvar^{j_i})
                    &   \text{if $\gvar^l := \sum_i \avar^l_{ji} \gvar^{j_i}$}\\
                0   &   \text{if $\gvar^l := \Avar^k \gvar^j$ or $\gvar^l := \Avar^k \hvar^j$}
            \end{cases}
            \label{eqn:muRecur}
\end{align}
and recursively define 
\begin{align}
    K^\cdc(\gvar^l, \gvar^m)
        &=
            \begin{cases}
            K^{\cdcin \infty}(\gvar^l, \gvar^m)   &   \text{if $\gvar^l, \gvar^m \in \cdcin$}\\
            \sum_i \avar^m_{j_i} K^\cdc(\gvar^l, \gvar^{j_i})
                &   \text{if $\gvar^m := \sum_i \avar^m_{j_i} \gvar^{j_i}$}\\
            \sum_i \avar^m_{j_i} K^\cdc(\gvar^{j_i}, \gvar^m)
                &   \text{if $\gvar^l := \sum_i \avar^l_{j_i} \gvar^{j_i}$}\\
            (\sigma^{k\infty})^2\EV_{z}[\fvar^a(z) \fvar^b(z)]
                &   \text{if $\gvar^l := \Avar^k \hvar^a, \gvar^m := \Avar^k \hvar^b$}\\
            0   &   \text{else}
            \end{cases}
            \label{eqn:KRecur}
\end{align}
where $\hvar^a := \fvar^a(\gvar^{j_1}, \ldots, \gvar^{j_{k}}), \hvar^b := \fvar^b(\gvar^{j_1'}, \ldots, \gvar^{j'_{k'}})$, and
$z \sim \Gaus(\mu^\cdc, K^\cdc)$.
We also make branch 4 cover the case when $\gvar^l := \Avar^k \gvar^a$ or $\gvar^m:=\Avar^k \gvar^b$ by ``typecasting'' $\gvar^a$ to an H-var and setting $\fvar^a = \id$ (similarly for $\gvar^b$).
Note that, as discussed in Notations above, $\fvar^a$ will ignore irrelevant components of $a$, and the expectations only depend on the entries of $\mu^\cdc$ and $K^\cdc$ that correspond to already-defined values, so this describes a valid recursion.

\newcommand{\Ss}{\mathcal{S}}
We introduce the following technical assumption.

\begin{assm}[Almost Sure Rank Convergence]
For any $\Avar^k$ and any collection $\Ss \sbe \{\text{G- or H-var $h$}: \gvar^l := \Avar^k h \text{ for some $l$}\}$,
let $H^t \in \R^{\nvar^{\cdc t} \times |\Ss|}$ be the matrix whose columns are $\hvar^{mt}$ or $\gvar^{mt}$ for each $\hvar^m$ or $\gvar^m$ in $\Ss$.
If $\f 1 {\nvar^{\cdc t}} H^t{}^\trsp H^t \in \R^{|\Ss| \times |\Ss|}$ converges almost surely to some $C^*$ with $t \to \infty$, then almost surely $\rank H^t = \rank C^*$ for all large $t$.
\end{assm}

If we don't have lines of type \ref{linetype:lincomb}, and no $\fvar^l$ is a polynomial, then the $C^*$s are all full rank, implying rank convergence by the upper semicontinuity of rank.
\ref{linetype:lincomb} lines may add linear dependencies, but they are constant with $t$, so that $\rank H^t = \rank C^*$ in the limit and we still have rank convergence.

\begin{defn}\label{defn:alphaControlled}
For $\alpha > 0$, a function $\phi: \R^k \to \R$ is said to be \textit{$\alpha$-controlled} if for some $C, c > 0$, $|\phi(x)| \le e^{C \sum_{i=1}^k |x_i|^{\alpha} + c}$ for all $x \in \R^k$.
\end{defn}

\begin{restatable}{thm}{notransposeLimit}\label{thm:notransposeLimit}
Consider dimension constraints $\Lambda$ and a skeleton $\pi$ without \ref{linetype:trsp} lines, i.e. no transpose allowed.
Suppose all $\fvar^l$ are $\alpha$-controlled for some $\alpha < 2$.
Sample all input vars as in \cref{subsec:setup} and assume almost sure rank convergence.
Then for any $\cdc \in \Cdc$ and any $\alpha$-controlled function $\phi: \R^{\cdc} \to \R$, $\alpha < 2$,
\begin{align*}
    \f 1 {\nvar^{\cdc t}} \sum_{i=1}^{\nvar^{\cdc t}}\phi(\gvar^{\cdc t}_i) \asto \EV \phi(Z)
\end{align*}
where $\gvar^{\cdc t}_i = (\gvar^{lt}_i)_{\gvar^l \in \cdc}$ and $\R^\cdc \ni Z = (Z^g)_{g \in \cdc} \sim \Gaus(\mu^\cdc, K^\cdc)$.

\end{restatable}

\paragraph{Discussion}
Roughly speaking, G-vars created from the same matrix $\Avar^k$ have nonzero correlations, but otherwise are asymptotically independent modulo \ref{linetype:lincomb}.
Intuitively, $\gvar^{\cdc t}_i ``\disteq" \Gaus(\mu^\cdc, K^\cdc)$ for large $t$, iid for each $i$.

There is an apparent contradiction in \cref{thm:notransposeLimit} if we consider deep linear networks with tied ($y = W^L x \in \R^n, W \in \R^{n \times n}$) and untied weights ($y = \prod_{l=1}^L \p W l x, \forall l [\p W l \in \R^{n \times n}]$).
Via simple computations of $\mu^\cdc$ and $K^\cdc$, one sees that, by \cref{thm:notransposeLimit}, as the width $n \to \infty$, $y$ is distributed ``similarly'' in either case (in that $\alpha$-controlled moments match asymptotically).
This seems to contradict our intuition that $W^L x$ should blow up or decay exponentially, with $L$, along the direction of the eigenvector of $W$ corresponding to the largest eigenvalue of $W$; whereas in the untied case it's easy to see that each $y_i$ converges in distribution to an i.i.d. Gaussian.

This apparent paradox is resolved by noting that \cref{thm:notransposeLimit} only applies for fixed skeletons (so fixed $L$ in this example), as widths $\to \infty$.
By \citet{rider_limit_2003}, the maximum eigenvalue of $W$ scales like $1 + O(n^{-1/2})$ if $W_{ij} \sim \Gaus(0, 1/n)$, and so does that of $W^L$ for fixed $L$.
Furthermore, as $n$ increases, the components of $x$ corresponding to large eigenvalues ($\ge 1$) of $W$ decrease in magnitude to 0 in probability, by the circular law \cite{tao_topics_2012}.
So the $L$ at which the exponentiating effect of $W^L$ kicks in increases with $n$.

\section{Backprop with Zero Mean Gradients}

\newcommand{\Lnabla}{{L_\nabla}}
\label{sec:zeroMeanGrad}

\begin{figure}[t]
\centering
\begin{subfigure}[b]{0.45\textwidth}
\begin{align*}
    1&: \Avar^1 := A^1\\
        &\vdots\\
    L_A&: \Avar^{L_A} := A^{L_A}\\
    L_A + 1&: \gvar^{L_A+1} := x^1\\
        &\vdots\\
    L_A + L_g&: \gvar^{L_A + L_g} := x^{L_g}\\
    L_A + L_g + 1&: \ldots:=\ldots
        &   \text{noninput line types}\\
        &\vdots\\
    L&: \ldots := \ldots
        &\text{last line}
\end{align*}
\caption{Program $\pi$}
\label{fig:programPi}
\end{subfigure}
~
\begin{subfigure}[b]{0.45\textwidth}
\begin{align*}
    L+1&: \Avar^{L+1}    :=     (\Avar^1)^\trsp\\
                        &\vdots\\
    L+L_A&:  \Avar^{L+L_A}   := (\Avar^{L_A})^\trsp\\
    L+L_A+1&:    \gvar^{L+L_A+1} := v^{1}\\
                                &\vdots\\
    L+L_A+L_\nabla&:    \gvar^{L+L_A+L_\nabla} := v^{\Lnabla}
\end{align*}
\caption{Extended program $\tilde \pi$}
\label{fig:programTildePi}
\end{subfigure}
\caption{}
\end{figure}

Let $\pi$ be a skeleton with $L$ lines but no \ref{linetype:trsp}.
WLOG, suppose all input vars appear at the start, with matrix inputs first, as in \cref{fig:programPi}.
Consider an extension $\tilde \pi$ of $\pi$ in the following way:
The first few appended lines are transposes of A-vars in $\pi$, followed by a series of new vector input vars $\{v^l\}_{l=1}^{\Lnabla}$, as in \cref{fig:programTildePi}.
Lines appended below this can be arbitrary non-input lines except that
\begin{enumerate*}[label={(\arabic*)}]
    \item lines of type \ref{linetype:GLinear} must use a transposed matrix $\Avar^{L+1}$ to $\Avar^{L+L_A}$ and $\hvar^l$ or $\gvar^l$ must have been introduced after $\pi$ (i.e. $l > L$), and
    \item any $\hvar^l$ for $l > L + L_A + \Lnabla$, as a function of $v^1, \ldots, v^\Lnabla$, must be odd: for any fixed values of $\{\gvar^l\}_{l \le L},$
$\hvar^l(-v^1, \ldots, -v^\Lnabla, \{\gvar^l\}_{l \le L}) = -\hvar^l(v^1, \ldots, v^\Lnabla, \{\gvar^l\}_{l \le L})$; likewise $\gvar^l$ must be odd for $l > L + L_A$.
    This in particular means that \ref{linetype:lincomb} lines cannot involve $\gvar^l$ for $l \le L$.
\end{enumerate*}

If $\pi$ expresses the forward computation of an NN $f$ without matrix transposes, then $\tilde \pi$ has enough power to express the backpropagation of $f$ and compute the gradients with respect to hidden states.
For example, if $f(x) = v^\trsp \rho(x)$, so that $\pdf f x = v^\trsp \pdf{\rho}x$, then $\pdf f x$ is an odd function of $v$ and can be expressed as a $\tilde \pi$ as above (see \cref{subsec:exampleFCFF} for a concrete example).
In general, the multiple $\{v^i\}$ allow for multiple NN outputs.

CDCs are naturally defined for $\tilde \pi$ (see \cref{sec:CDC}) just like before.
We extend $\mu^\cdc$ and $K^\cdc$ to vars introduced in $\tilde \pi$:
For $l, m > 0,$ and $k \le L_A,$ set $\mu^\cdc(\gvar^{L+l})  = 0$ and when $l > L$ or $m > L$, set $K^\cdc(\gvar^l, \gvar^m) = $
\begin{align*}
    \begin{cases}
            K^{\cdcin \infty}(\gvar^l, \gvar^m)   &   \text{if $\gvar^l, \gvar^m$ are input vars}\\
            \sum_i \avar^m_{j_i} K^\cdc(\gvar^l, \gvar^{j_i})
                &   \text{if $\gvar^m := \sum_i \avar^m_{j_i} \gvar^{j_i}$}\\
            \sum_i \avar^m_{j_i} K^\cdc(\gvar^{j_i}, \gvar^m)
                &   \text{if $\gvar^l := \sum_i \avar^l_{j_i} \gvar^{j_i}$}\\
            (\sigma^{k\infty})^2 \alpha
                \EV_{z}\fvar^a(z) \fvar^b(z)
                &   \text{if $\gvar^l := \Avar^{L+k} \hvar^a, \gvar^m := \Avar^{L+k} \hvar^b$}\\
            0   &   \text{else}
            \end{cases}
\end{align*}
where $\alpha = \alpha_{\cdc_1(\Avar^k), \cdc_2(\Avar^k)} = \lim_{t \to \infty} \f{\nvar^{\cdc_1(\Avar^k)t}}{\nvar^{\cdc_2(\Avar^k)t}}$, and branch 4 covers the case when $\gvar^l := \Avar^{L+k} \gvar^a$ or $\gvar^m := \Avar^{L+k} \gvar^b$ by taking $\fvar^a$ or $\fvar^b$ to be identity.
Note that covariances between vars of $\pi$ and new vars in $\tilde \pi$ are 0.

\begin{restatable}{thm}{gradIndep}\label{thm:gradIndep}
Sample $\{v^{it}\}_{i=1}^{\Lnabla}$ with zero mean (i.e. $\mu^{\cdcin t}(\gvar^{L+L_A+i}) \to 0$ for all $i \in [\Lnabla]$) and independently from the input vars $\{x^{lt}\}_{l=1}^{L_g}$ (i.e. $K^{\cdcin t}(\gvar^l, \gvar^{l'}) = 0$ if $l > L \ge l'$)
\footnote{In our previous example of $f(x) = v^\trsp \rho(x)$, this corresponds to the readout layer $v$ sampled with zero mean and independently from $x$ and other parameters of $\rho$.}.
Sample all other vars in $\tilde \pi$ according to \cref{subsec:setup}.
Assume all $\fvar^l$ of $\tilde \pi$ are polynomially bounded and $\tilde \pi$ satisfies almost sure rank convergence.
Then for any dimension constraints $\Lambda$, any $\cdc \in \Cdc_{(\tilde \pi, \Lambda)}$, and any polynomially bounded function $\phi: \R^{\cdc} \to \R$,
\begin{align*}
    \f 1 {\nvar^{\cdc t}} \sum_{i=1}^{\nvar^{\cdc t}}\phi(\gvar^{\cdc t}_i) \asto \EV \phi(Z)
\end{align*}
where $\gvar^{\cdc t}_i = (\gvar^{lt}_i)_{\gvar^l \in \cdc}$ and $Z \sim \Gaus(\mu^\cdc, K^\cdc)$.
\end{restatable}

Note that our result here does not apply to batchnorm, whose Jacobian has a singularity on a 1-dimensional affine subspace (and in particular, at the origin).
This theorem allows one to justify the gradient independence assumption rigorously; see \cref{subsec:gradIndAssm}.

\section{General Tensor Programs}

\cref{thm:gradIndep} does not give the correct computation if $\{v^{it}\}_i$ do not have zero mean:
Consider a one-hidden-layer MLP with quadratic activation, $f(x) = \onev^\trsp \phi(W x), \phi(-) = \f 1 2 (-)^2, x \in \R^{n^0}, W \in \R^{n^1 \times n^0}, \onev \in \R^{n^1}.$
Then $\pdf f x = W^\trsp (\onev \odot (W x)) = W^\trsp W x$.
If $n^1 = n^0$, and $W_{ij} \sim \Gaus(0, 1/n^0)$, then $ \EV \pdf f {x_i} = x_i$.
If we have assumed \cref{thm:gradIndep} is correct, then we would have (incorrectly) computed $\EV \f 1 {n^0} \sum_{i=1}^{n^0} \pdf f {x_i} \asto \EV \f 1 {n^0} \sum_{i=1}^{n^0} (W'{}^\trsp W x)_i = 0$ where $W'$ is an iid copy of $W$.

Below, we develop a theory of the scaling limit of general tensor programs, from which follows the correct way of computing gradients when $v^{it}$ do not have 0 mean.

We first introduce ``extended syntax'' programs, which are equivalent semantically to programs of original syntax, and then show that we can ``compile'' original syntax programs to extended syntax programs with no transposes, with the same scaling limit in a suitable sense.

\begin{defn}
\textit{Extended syntax} programs are those that allow all line types of \cref{defn:linetypes} and in addition
\begin{description}
    \item[Comp\label{linetype:Hcomp}] (H) if $\nvar^{j_1} = \cdots = \nvar^{j_k}$, then an assignment via some general (possibly nonlinear) function $\fvar^l: \R^k \to \R$
    $$l: \hvar^l := \fvar^l(h^1, \ldots, h^k) \in \R^{\nvar^l} = \R^{\nvar^{j_1}}$$
    where $h^1, \ldots, h^k$ are previous G- {\it or} H-vars, and $\fvar^l$ acts coordinatewise.
\end{description}
\end{defn}
So in essence, extended syntax programs just allow \ref{linetype:nonlin} lines to take H-vars in addition to G-vars.
While in the original syntax, H-vars must feed into lines of type~\ref{linetype:GLinear}, in extended syntax they can also be used to create new H-vars via coordinatewise action.

One can define CDCs for extended syntax programs just as before (see \cref{sec:CDC}).
Each extended syntax program is equivalent to an original syntax program, by expanding the definition of each H-var to a function of previous G-vars.
For example, if $\hvar^l := \fvar^l(\hvar^k, \gvar^m)$, and $\hvar^k := \fvar^k(\gvar^r)$, then the expanded definition of $\hvar^l$ is $\fvar^l(\fvar^k(\gvar^r), \gvar^m)$.
We call this the {\it expanded definition} of $\hvar^l$, and write $\fvar^{\hvar^l}$ for this expanded function, so that
$\hvar^l = \fvar^{\hvar^l}(\gvar^{l_1}, \ldots, \gvar^{l_m})$ for some $l_1, \ldots, l_m < l$; for G-vars, we also define $\fvar^{\gvar^l} = \id$.
(In our example above, $\fvar^{\hvar^l}(x, y) = \fvar^l(\fvar^k(x), y)$).
So by replacing each line of type~\ref{linetype:Hcomp} with its expanded definition, we can convert an extended syntax program to an original syntax program with the same semantics for all vector vars.

\begin{defn}\label{defn:detransposition}
Let $\pi$ be an original syntax skeleton with associated sampling data $\{\sigma^{lt}\}_{\Avar^l}, \{\mu^{\cdcin t}, K^{\cdcin t}\}_{\cdc \in \Cdc}.$
We define an extended syntax skeleton $\check \pi$, called the {\it detransposition} of $\pi$, by induction on line number as follows.
During this process, we keep track of an injective mapping $\varphi$ taking vector (resp. matrix) vars of $\pi$ to vector (resp. matrix) vars of $\check \pi$, along with a specialized mapping $\varphig$ taking a G-var of $\pi$ produced by \ref{linetype:GLinear} to a G-var of $\check \pi$.
We use a check $\check{\phantom{x}}$ to denote objects of the detransposition.
We also simultaneously set $\{\check \sigma^{lt}\}_{\check \Avar^l}$, $\{\mu^{\check \cdcin t}\}_\cdc$ and $\{K^{\check \cdcin t}\}_\cdc$ of the detransposition.
They propagate according to the usual rules, \cref{eqn:KRecur,eqn:muRecur}, to determine $\mu^{\check \cdc}$ and $K^{\check \cdc}$.
Let $l$ be the current line number of $\pi$ we are processing, and let $\ell$ denote the 1 + length of the current $\check \pi$ (this is where we are adding new lines in $\check \pi$).
\begin{enumerate}
    \item
        If $\gvar^l := x$ is a \ref{linetype:G} line, then add a line of the same type to $\check \pi$, $\ell: \check \gvar^\ell := x$.
        Set $\varphi(\gvar^l) \gets \check \gvar^\ell$,
        $\mu^{\check \cdcin t}(\check \gvar^\ell) \gets  \mu^{\cdcin t}(\gvar^l),$
        where $\check \cdc = \cdc(\check \gvar^l)$ and $\cdc = \cdc(\gvar^l)$.
        Set $K^{\check \cdcin t}(\check \gvar^\ell, \varphi(\gvar^m)) \gets
        K^{\cdcin t}(\gvar^l, \gvar^m)$ for all input G-vars $\gvar^m$ with $m < l$.
    \item 
        If $\Avar^l:= A$ is a \ref{linetype:A} line, then add to $\check \pi$ the line $\ell: \check \Avar^\ell := A$.
        Set $\varphi(\Avar^l) \gets \check \Avar^\ell$ and $\check \sigma^{\ell t} \gets \sigma^{lt}$ for all $t \in [1, \infty]$.
    \item
        If $\Avar^l := \Avar^m{}^\trsp$ is a \ref{linetype:trsp} line, then add to $\check \pi$ an input line $\ell: \check \Avar^\ell := A'$ for a new input $A'$ sampled iid as $\Avar^m{}^\trsp$.
        Set $\varphi(\Avar^l) \gets \check \Avar^\ell$ and $\check \sigma^{\ell t} \gets \sqrt{\f{\nvar_1(\Avar^{mt})}{\nvar_2(\Avar^{mt})}} \sigma^{m t}, \forall t \in [1, \infty].$

    \item
        If $\gvar^l := \avar^l_{j_1}\gvar^{j_1} + \cdots + \avar^l_{j_k}\gvar^{j_k}$ is an \ref{linetype:lincomb} line in $\pi$, we add a line of the same type in $\check \pi$:
        $$\ell: \check\gvar^\ell := \avar^l_{j_1}\varphi(\gvar^{j_1}) + \cdots + \avar^l_{j_k}\varphi(\gvar^{j_k})$$
        and set $\varphi(\gvar^l) \gets \check \gvar^\ell$
        if each of $\varphi(\gvar^{j_i}) $ is a G-var in $\check \pi$; or we add a line of type~\ref{linetype:Hcomp}
        $$\ell: \check\hvar^\ell := \avar^l_{j_1}\varphi(\gvar^{j_1}) + \cdots + \avar^l_{j_k}\varphi(\gvar^{j_k})$$
        and set $\varphi(\gvar^l) \gets \check \hvar^\ell$
        if some $\varphi(\gvar^{j_i}) $ is an H-var in $\check \pi$.
    \item
        If $\hvar^l := \fvar^l(\gvar^{j_1}, \ldots, \gvar^{j_k})$ is a line of type~\ref{linetype:nonlin}, then we add to $\check \pi$ a line of type~\ref{linetype:Hcomp}
        $$\ell: \check\hvar^\ell := \check\fvar^\ell(\varphi(\gvar^{j_1}), \ldots, \varphi(\gvar^{j_k}))$$
        where $\check \fvar^\ell = \fvar^l$, and we set $\varphi(\hvar^l) \gets \check \hvar^\ell.$
        (If all $\varphi(\gvar^{j_i})$ are G-vars then we also typecast this line to \ref{linetype:nonlin})
    \item \label{detrx:GLinear}
        Suppose $\gvar^l := A \hvar^m$ is a line of type~\ref{linetype:GLinear} in $\pi$, where $A$ is some previous $A$-var.
        Consider the A-var $A'$ where $A' := A^\trsp$ if $A$ is an input A-var, or $A := A'{}^\trsp$ if $A$ is a transposed var.
        Let $g^i := A' h^i, i = 1, \ldots, s,$ be all previous lines of type~\ref{linetype:GLinear} involving $A'$, where $h^i$ can be G- or H-var.
        Define $C \in \R^{s \times s}, v \in \R^s$ by
        \begin{align*}
            C_{ij}
                &=
                    \EV \fvar^{\varphi(h^i)}(Z) \fvar^{\varphi(h^j)}(Z),
                    \\
            v_i
                &=
                    \EV \fvar^{\varphig(g^i)}(Z) \fvar^{\varphi(\hvar^{m})}(Z)
                    ,
        \end{align*}
        where $Z \sim \Gaus(\mu^{\check \cdc}, K^{\check \cdc})$ (the expectation will only depend on components of $Z$ corresponding to previous lines).
        Compute $\avar = \alpha C{}^+ v \in \R^s$, where $\alpha = \lim_{t} \f{\nvar_2(A^t)}{\nvar_1(A^t)}.$
        Then we add the following to $\check \pi$:
        \begin{align*}
            \ell&: \check \gvar^\ell := \varphi(A) \varphi(\hvar^m) & \text{\ref{linetype:GLinear}}\\
            \ell+1&: \check \hvar^{\ell+1} := \check \gvar^{\ell} + \sum_{j=1}^s \avar_j \varphi(h^j) & \text{\ref{linetype:Hcomp}}
        \end{align*}
        If $\varphi(h^j)$ are all G-vars, we typecast line $\ell+1$ to \ref{linetype:lincomb} and write $\check \gvar^{\ell+1}$ instead.
        Set $\varphi(\gvar^l) \gets \check \hvar^{\ell+1}$ and $\varphig(\gvar^l) \gets \check \gvar^\ell.$

\end{enumerate}
\end{defn}

See \cref{subsec:exampleDetransposition} for a concrete example of detransposition .
Below, for any $\cdc \in \Cdc_\pi$, let $\bar \cdc$ be $\cdc \cup \{h: \cdc(h) = \cdc\}$, i.e. the collection of H- or G-vars with the same dimension constraint; see \cref{sec:CDC}.
\begin{restatable}{thm}{generalTensorP}\label{thm:generalTensorP}
Let $\pi$ be an original syntax program with sampling instructions, and $\check \pi$ be the detransposition of $\pi$, with $\varphi$ the mapping from vector vars of $\pi$ to vector vars of $\check \pi$.
Assume all $\fvar^l$ of $\pi$ are polynomially bounded, and that almost sure rank convergence holds for $\check \pi$.
Sample input vars of $\pi$ according to \cref{subsec:setup}.
Then for any dimension constraints $\Lambda$, any $\cdc \in \Cdc_{(\pi, \Lambda)}$, and any polynomially bounded function $\phi: \R^{\bar\cdc} \to \R$,
\begin{align*}
    \f 1 {\nvar^{\cdc t}} \sum_{i=1}^{\nvar^{\cdc t}}
        \phi(\hvar_i^{\cdc t})
            \asto
                \EV \phi((\fvar^{\varphi(h)}(Z))_{h \in \bar\cdc})
\end{align*}
where $\hvar^{\cdc t}_i$ is the sequence of the $i$th coordinates of all vector vars in $\bar\cdc$, and
$Z \sim \Gaus(\mu^{\check \cdc}, K^{\check \cdc}).$

If all $\fvar^l$ in $\pi$ are differentiable, then we can take $\avar$ in \cref{detrx:GLinear} of \cref{defn:detransposition} to be $\alpha \sigma^2 (\EV \pd_{Z^{\varphig(g^i)}} \fvar^{\varphi(\hvar^m)}(Z))_{i \in [s]}$ \footnote{even if not all $\fvar^l$ are differentiable, as long as the covariance $\{K^{\check \cdc}(\varphig(g^i), \varphig(g^j))\}_{i,j\in[s]}$ is nonsingular, we may take the distributional derivatives and interpret the expectation as the application of this (tempered) distribution to the Gaussian density function. Then the theorem holds under this interpretation.
Even if the covariance is singular, we may consider a subset $\{\varphig(g^i)\}_{i \in \mathcal{I}}$ of maximal rank, and still apply the tempered distribution interpretation; see the proof for more details.
}, where $\sigma = \sigma^{r\infty}$ and $r$ is the line number of $\varphi(A')$, and the above almost sure convergence will still hold.
\end{restatable}

See \cref{subsec:warmupConseq,warmup:semicirclelaw,warmup:marchenkopasturlaw} for example applications of this theorem to random MLP, and rederivation of the semicircle and Marchenko-Pastur laws.
\cref{thm:generalTensorP} has the following basic intuition.
Let $g = A^\trsp F((A h^j)_{j=1}^k)$ for $A \in \R^{n \times m}, h^i \in \R^m, F: \R^{k \times n} \to \R^{n}, ((z^j \in \R^n)_{j=1}^k) \mapsto F((z^j)_{j=1}^k)$.
Here $F$ should be thought of the stretch of $\pi$ that separates $g$ from previous G-vars induced by $A$.
Then, semirigorously, as $n,m \to \infty$ and $A_{ij} \sim \Gaus(0, 1/n),$ for any $i \in [m]$,
$$
\EV y_i
    \approx
        \EV A_{:i}^\trsp \left\{ F((A^{\setminus i} h^j_{\setminus i})_{j=1}^k)
        +
        \la F', (A_{:i} h_i^j)_{j=1}^k\ra\right\}
$$
where $A^{\setminus i}$ is $A$ without $i$th column, $h^j_{\setminus i}$ is $h^j$ without $i$th coordinate, and $F'$ is the Jacobian of $F$ at $(A^{\setminus i} h^j_{\setminus i})_{j=1}^k$.
Now $\EV  A_{:i}^\trsp \la F', (A_{:i} h_i^j)_{j=1}^k\ra = \sum_{j=1}^k h^j_i \sum_{m=1}^n \pdf {F_m}{z_m^j}$ because $F'$ is independent of $A_{:i}$.
Likewise, $ A_{:i}^\trsp F((A^{\setminus i} h^j_{\setminus i})_{j=1}^k)$ is approximately a Gaussian with 0 mean.
Then $g$ is roughly a Gaussian plus a linear combination of $\{h^j\}_{j=1}^k$.
So, unlike restricted programs in \cref{thm:notransposeLimit,thm:gradIndep}, we do not expect $g$ to be ``Gaussian'' in the limit, as $h^j$ could be the image of an activation function.
We can, however, still keep track of $g$'s decomposition into a Gaussian and a linear combination part --- this is the key idea behind detransposition, and in particular, its step 6.
There, $\vec a'_j$ are the coefficients of this linear combination, while $\vec a_i$ records the correlation between $g$ and previous G-vars induced by $A^\trsp$.
Each $\vec a'_j$ can be seen to be exactly the coefficient computed heuristically here by applying Stein's lemma (\cref{lemma:stein}) when $r=0$ in step 6.

\section{Proof Techniques}

While our tensor program framework is new and \citet{bayati_dynamics_2011} is concerned with a much simpler setting of AMP algorithms, its technique of Gaussian conditioning is useful to us (\cref{lemma:condTrick}): If $A$ is a Gaussian matrix, then conditioned on $G = AH, G' = A^\trsp H'$, the distribution of $A$ is $E + \Pi \tilde A \Pi'$ for some mean matrix $E$, projection matrices $\Pi, \Pi'$, and $\tilde A$ distributed iid as $A$.
If we let $G$ and $H$ be previous G- and H-vars in \ref{linetype:GLinear} lines involving $A$, and similarly for $G', H'$ with respect to $A^\trsp$, this allows us to induct on line number by conditioning on previous lines.

Compared to \citet{bayati_dynamics_2011}, our input G-vars have all finite moments, whereas the analogous quantities in \citet{bayati_dynamics_2011} just have a bounded number of them.
This allows us to simplify the induction somewhat, and remove the smoothness assumption on (the functions playing the same role as) $\fvar^l$ that is required in \citet{bayati_dynamics_2011}.
The latter is a result of two facts:
\begin{enumerate*}
    \item Gaussian averaging is smooth: $\EV[\Phi(z): z \sim \Gaus(\mu, \Sigma)]$ is generically smooth in $\mu$ and $\Sigma$.
    \item if $\Pi \in \R^{n \times n}$  is a projection matrix of rank $n - O(1)$, then $\Pi z$ for an isotropic Gaussian vector $z$ has approximately independent coordinates, in so far as a law of large number is concerned.
    This is shown by first bounding the off-diagonal correlations of $\Pi$ using linear programming duality (\cref{lemma:projectionCorrelation}) and then bounding the moments of $\sum_i \phi(\Pi z)_i$ using the Hermite expansion of $\phi$ and the previous bound on correlations of $\Pi z$ (\cref{thm:controlHighMoments}).
\end{enumerate*}
Again, these two tools were not accessible to \citet{bayati_dynamics_2011} because of their assumptions of only a finite number of input moments.
Note that a more straightforward application of the logic of \citet{bayati_dynamics_2011} would not allow us to reason about nonsmooth functions such as the step function which appears as the gradient of ReLU.

\section{Discussion}
\label{sec:discussion}

In this paper, we have introduced a notion of a \emph{tensor program} able to express almost all neural network computations in modern deep learning, and characterized its scaling limits.
As corollaries, we generalized the DNN-GP correspondence, rigorously derived correct equations for signal propagation in neural networks, and proved the convergence of NTK, among many other results discussed in \cref{sec:related}.

While our results assume Gaussian sampling, we expect the results to hold when we sample from other ``nice'' distributions (say, with a few finite moments).
In the random matrix and statistical physics literature, this \emph{universality} is very common.
In our case, the central limit intuition for DNN-GP correspondence, for example, would indeed hint at this.

We also believe that the rank convergence assumptions are not necessary, but just side effects of our Gaussian conditioning technique.
One should be able to remove them by considering some stability property of tensor programs.

Our framework, while surprisingly general, as presented doesn't cover a few deep learning layers, but can be easily extended to do so in all but one case:
\begin{itemize*}
    \item
        Dropout.
        Our framework can already cover ``Gaussian dropout''; binary dropout can be incorporated easily as well by introducing Bernoulli random variables in the way of \citet{schoenholz_deep_2017}.
    \item
        Layernorm.
        Our framework only allows $\fvar^l$ with a fixed signature, as ``width'' grows, for example batchnorm with a fixed batch size.
        However, as we take width to infinity, the signature for layernorm also changes.
        But our theorems show that the mean and variance of the layer converges a.s. to a deterministic limit, so that in the forward pass, layernorm is asymptotically the same as a linear, coordinatewise $\fvar^l$.
        A brief computation shows that the gradient of layernorm can also be asymptotically expressed via a tensor program in a similar way.
        Nevertheless, non-``coordinatewise'' $\fvar^l$ is worth investigating in the future, perhaps to take inspiration from the work of \citet{berthier_state_2017} that studied this scenario in the setting of State Evolution for AMP.
    \item
        Batchnorm, when reasoning about gradients.
        Our framework does not allow singularities in $\fvar^l$, whereas during backprop batchnorm's derivative contains a singularity at the origin.
        \citet{yang_mean_2018} demonstrates that empirically our equations should still extend to this case.
        We leave this to future work.
\end{itemize*}

Our scaling limit results only apply to fixed tensor program skeletons.
This would be enough to derive the behavior of a DNN on a dataset which is small compared to the width.
But perhaps more reasonable is when the dataset size is commensurate or perhaps even larger than the width of the network.
This would require taking a joint limit in both the skeleton size, over the data distribution, and over the dimensions $\{\nvar^l\}_l$;
see \citet{pennington_nonlinear_2017,pennington_spectrum_2018} for analogous settings for 2 or 3 layer networks and Gaussian data.
We leave the investigation of this to future work.

The tensor program framework naturally lends to an automation of computations regarding random neural networks, given the program underlying it.
Our community might find valuable a module in PyTorch or Tensorflow that computes the corresponding $\mu^\cdc$ and $K^\cdc$ automatically given the tape (for PyTorch) or the computation graph (for Tensorflow).

We hope the tools presented in this work will be adopted by any researcher interested in studying random neural networks.

\section*{Acknowledgement}
Thanks are due to Jascha Sohl-Dickstein, Sam Schoenholz, Jeffrey Pennington, Raphael Berthier, Ilya Razenshteyn, Pengchuan Zhang, Hadi Salman, and Zeyuan Allen-Zhu for discussions and help on initial drafts

\newpage
\onecolumn
\bibliography{references}
\bibliographystyle{icml2019}

\clearpage
\appendix
\onecolumn

\section{Common Dimension Classes}

\label{sec:CDC}
We present an algorithm below to compute each CDC $\cdc$ of $(\pi, \Lambda)$.
We write $\cdc(g)$ for the CDC associated to the G-var $g$, and we as well associate CDC $\cdc(h)$ to each H-var and left- and right- CDCs $\cdc_1(A), \cdc_2(A)$ to each A-var.
Here $\gvar^l$ and $\gvar^m$ should be interpreted as elements of the set $\cdc$ and not as vectors.
We induct on lines in the skeleton $\pi$.
\newcommand{\Lambdasim}{\overset{\Lambda}{\sim}}
First let $\Lambdasim$ be the smallest equivalence relation on G-vars such that $\gvar^l \Lambdasim \gvar^m$ if $(\nvar^l, \nvar^m) \in \Lambda$.
\begin{enumerate}
    \item \ref{linetype:G} $\gvar^l$.
        Set $\cdc(\gvar^l) \gets \{\gvar^m: \gvar^m \Lambdasim \gvar^l\}$.
    \item \ref{linetype:A} $\Avar^l$.
        Set $\cdc_1(\Avar^l) \gets \{\}, \cdc_2(\Avar^l) \gets \{\}$.
    \item \ref{linetype:trsp} $\Avar^l \defeq (\Avar^k)^\trsp$.
        Set $\cdc_1(\Avar^l) \gets \cdc_2(\Avar^k), \cdc_2(\Avar^l) \gets \cdc_1(\Avar^k).$
    \item \ref{linetype:GLinear}
        \begin{enumerate}
            \item If $\gvar^l := \Avar^k \gvar^j$:
                Merge $\cdc(\gvar^j) \gets \cdc(\gvar^j) \cup \cdc_2(\Avar^k) \to \cdc_2(\Avar^k)$ and $\cdc(\gvar^l) \gets \cdc_1(\Avar^k) \cup \{\gvar^m: \gvar^m \Lambdasim \gvar^l\} \to \cdc_1(\Avar^k)$.
            \item If  $\gvar^l := \Avar^k \hvar^j$:
                Merge $\cdc(\hvar^j) \gets \cdc(\hvar^j) \cup \cdc_2(\Avar^k) \to \cdc_2(\Avar^k)$ and $\cdc(\gvar^l) \gets \cdc_1(\Avar^k) \cup \{\gvar^m: \gvar^m \Lambdasim \gvar^l\} \to \cdc_1(\Avar^k)$.
        \end{enumerate}
    \item \ref{linetype:lincomb} $\gvar^l := \avar^l_{j_1}\gvar^{j_1} + \cdots + \avar^l_{j_k}\gvar^{j_k}$.
        Merge $\cdc(\gvar^{j_i}) \gets \{\gvar^m: \gvar^m \Lambdasim \gvar^l\} \cup \bigcup_{i'} \cdc(\gvar^{j_{i'}}) \to \cdc(\gvar^l)$ for all $i$.
    \item \ref{linetype:nonlin} $\hvar^l := \fvar^l(\gvar^{j_1}, \ldots, \gvar^{j_k}) \in \R^{\nvar^l} = \R^{\nvar^{j_1}}$.
        Merge $\cdc(\gvar^{j_i}) \gets \bigcup_{i'} \cdc(\gvar^{j_{i'}}) \to \cdc(\hvar^l)$ for all $i$.
    \item \ref{linetype:Hcomp} similar to \ref{linetype:nonlin}.
\end{enumerate}

\section{Example Programs}
\label{sec:examplePrograms}
\subsection{MLP}
\label{subsec:exampleFCFF}
\begin{align*}
    1 &: \gvar^1 := W^1 x &   \text{network input multiplied by weight matrix}\\
    2 &: \gvar^2 := b^1 &   \text{layer 1 bias}\\
    3 &: \gvar^3 := \gvar^1 + \gvar^2   &   \text{layer 1 preactivation}\\
    4 &: \hvar^4 := \phi(\gvar^3)   &   \text{layer 1 activation}\\
    5 &: \Avar^5 := W^2 &   \text{layer 2 weights}\\
    6 &: \gvar^6 := b^2 &   \text{layer 2 biases}\\
    7 &: \gvar^7 := \Avar^5 \hvar^4\\
    8 &: \gvar^8 := \gvar^7 + \gvar^6   &   \text{layer 2 preactivations}\\
    9 &: \hvar^9 := \phi(\gvar^8)   &   \text{layer 2 activations}
\end{align*}

In line 1 above, we could also spend a few more lines and equivalently set $\gvar^1 := \sum_{i=1}^{n^0} x_i \gvar^{-i}$ where $\gvar^{-i} := W^1_{:i}$.
For brevity, we adopt the current approach, but later for reasoning about backprop, this way of expression $W^1 x$ is useful.
Note that we can also express $x$ as its own deterministic input G-var and $W^1$ as an input A-var, but the program given here is more consistent with our scaling limit, which takes the hidden layer width to infinity but keeps the input dimension fixed.

Backprop of fully-connected, feedforward:
\begin{align*}
    10 &: \gvar^{10} := \nabla_{\hvar^9} L    &   \text{last layer gradient}\\
    11 &: \hvar^{11} := \phi'(\gvar^8) \odot \gvar^{10} &   \text{layer 2 preactivation gradient}\\
    12 &: \Avar^{12} := (\Avar^5)^\trsp\\
    13 &: \gvar^{13} := \Avar^{12} \hvar^{11}   &   \text{layer 1 activation gradient}\\
    14 &: \hvar^{14} := \phi'(\gvar^3) \odot \gvar^{13}  &   \text{layer 1 preactivation gradient}
\end{align*}
Here $\nabla_{\hvar^9} L$ can be any vector but in the context of neural networks, it can be thought as the gradient of some loss function $L$ obtained through some readout layer.

\subsubsection{Detransposition}
\label{subsec:exampleDetransposition}
We demonstrate the detransposition $\check \pi$ of the above program $\pi$.
Line 1 to line 9 are almost copied verbatim to $\check \pi$ because the only matrix multiplications involve new A-vars.

\begin{align*}
    1 &: \check\gvar^1 := W^1 x & \varphi(\gvar^1) = \check \gvar^1\\
    2 &: \check \gvar^2 := b^1 & \varphi(\gvar^2) = \check \gvar^2\\
    3 &: \check \gvar^3 := \check \gvar^1 + \check \gvar^2 & \varphi(\gvar^3) = \check \gvar^3\\
    4 &: \check\hvar^4 := \phi(\check \gvar^3) & \varphi(\hvar^4) = \check \hvar^4\\
    5 &: \check \gvar^5 := b^2  &   \varphi(\gvar^5) = \check \gvar^5\\
    6 &: \check \Avar^6 := W^2 & \varphi(\Avar^6) = \check \Avar^6\\
    7 &: \check \gvar^7 := \check \Avar^6 \check \hvar^4 & \text{begin \ref{linetype:GLinear} conversion; $\varphig(\gvar^7) = \check \gvar^7$}\\
    8 &: \check \gvar^8 := \check \gvar^7 + 0 & \text{end \ref{linetype:GLinear} conversion; $\varphi(\gvar^7) = \check \gvar^8$}\\
    9 &: \check \gvar^{9} := \check \gvar^8 + \check \gvar^5 & \varphi(\gvar^8) = \check \gvar^{9}\\
    10 &: \check \hvar^{10} := \phi(\check \gvar^{9}) & \varphi(\hvar^9) = \check \hvar^{10}
\end{align*}

Now to convert line 10 to 14 of $\pi$
\begin{align*}
    11 &: \check \gvar^{11} := \nabla_{\hvar^9} L & \varphi(\gvar^{10}) = \check \gvar^{11}\\
    12 &: \check \hvar^{12} := \phi'(\check \gvar^{9}) \odot \check \gvar^{11}
        & \varphi(\hvar^{11}) = \check \hvar^{12}\\
    13 &: \check \Avar^{13} := \text{iid copy of }W^2{}^\trsp & \varphi(\Avar^{12}) = \check \Avar^{13}\\
    14 &: \check \gvar^{14} := \check \Avar^{13} \check \hvar^{12} & \text{begin \ref{linetype:GLinear} conversion; $\varphig(\gvar^{13}) = \check \gvar^{14}$}\\
    15 &: \check \hvar^{15} := \check \gvar^{14} + \avar' \check \hvar^4 & \text{end \ref{linetype:GLinear} conversion; $\varphi(\gvar^{13}) = \check \hvar^{15}$}\\
    16 &: \check \hvar^{16} := \phi'(\check \gvar^3) \odot \check \hvar^{15}
\end{align*}
where 
\begin{align*}
    \avar'
        &=
            \alpha (\EV \fvar^{\varphi(\hvar^4)}(Z)^2)^{-1} \EV \fvar^{\varphig(\gvar^7)}(Z) \fvar^{\varphi(\hvar^{11})}(Z)
            \\
        &=
            \alpha (\EV \fvar^{\check \hvar^4}(Z')^2)^{-1} \EV \fvar^{\check \gvar^7}(Z) \fvar^{\check \hvar^{12}}(Z)
            \\
        &=
            \alpha (\EV \phi(Z'_3)^2)^{-1}\EV Z_7 \phi'(Z_{9}) Z_{11}
\end{align*}
with $Z, Z' \sim \Gaus(\mu^{\check \cdc}, K^{\check \cdc})$ and $\alpha = \lim \nvar_1(\Avar^6)/\nvar_2(\Avar^6).$

\subsection{Batched input}
\label{subsec:exampleBatch}
For the second input $y$ in the batch
\begin{align*}
    15 &: \gvar^{15} := W^1 y^0 &   \text{2nd input multiplied by (same) weight matrix}\\
    16 &: \gvar^{16} := \gvar^{15} + \gvar^2 &  \text{using same bias as before}\\
    17 &: \hvar^{17} := \phi(\gvar^{16}) &  \text{layer 1 activation}\\
    18 &: \gvar^{18} := \Avar^5 \hvar^{17} & \text{using same weight matrix}\\
    19 &: \gvar^{19} := \gvar^{18} + \gvar^6   &   \text{using same bias}\\
    20 &: \hvar^{20} := \phi(\gvar^{19})   &   \text{layer 2 activations}
\end{align*}
Gradients:
\begin{align*}
    21 &: \gvar^{21} := \nabla_{\hvar^{20}} L    &   \text{last layer gradient for input $y$}\\
    22 &: \hvar^{22} := \phi'(\gvar^{19}) \odot \gvar^{21} &   \text{layer 2 preactivation gradient}\\
    23 &: \gvar^{23} := \Avar^{12} \hvar^{22}   &   \text{layer 1 activation gradient; using same weights}\\
    24 &: \hvar^{24} := \phi'(\gvar^{16}) \odot \gvar^{23}  &   \text{layer 1 preactivation gradient}
\end{align*}

\subsection{Residual network}
\label{subsec:exampleResnet}
Style 1: resblock merges after weights
\begin{align*}
    1 &: \gvar^1 := W^1 x &   \text{network input multiplied by weight matrix}\\
    2 &: \gvar^2 := b^1 &   \text{resblock 1 bias}\\
    3 &: \gvar^3 := \gvar^1 + \gvar^2   &   \text{resblock 1 preactivation}\\
    4 &: \hvar^4 := \phi(\gvar^3)   &   \text{resblock 1 activation}\\
    5 &: \Avar^5 := W^2 &   \text{resblock 1 merge weights}\\
    6 &: \gvar^6 := b^2 &   \text{resblock 1 merge biases}\\
    7 &: \gvar^7 := \Avar^5 \hvar^4\\
    8 &: \gvar^8 := \gvar^1 + \gvar^7 + \gvar^6   &   \text{return to main branch}
\end{align*}

Style 2: resblock merges before weights
\begin{align*}
    1 &: \gvar^1 := W^1 x &   \text{network input multiplied by weight matrix}\\
    2 &: \gvar^2 := b^1 &   \text{resblock 1 bias}\\
    3 &: \gvar^3 := \gvar^1 + \gvar^2   &   \text{resblock 1 preactivation}\\
    4 &: \hvar^4 := \phi(\gvar^3) + \gvar^1   &   \text{resblock 1 activation, merge back to main branch}\\
    5 &: \Avar^5 := W^2\\
    6 &: \gvar^6 := b^2\\
    7 &: \gvar^7 := \Avar^5 \hvar^4\\
    8 &: \gvar^8 := \gvar^7 + \gvar^6   &   \text{resblock 2 preactivation}\\
    9 &: \hvar^9 := \phi(\gvar^8) + \phi(\gvar^3) + \gvar^1 &   \text{merge of 2nd resblock; semantically same as $\hvar^9 := \phi(\gvar^8) + \hvar^4$}
\end{align*}

\subsection{Simple RNN}
\label{subsec:exampleRNN}
This is almost the same as the feedforward case except we tie the weights across time, and have inputs for each time step.

\begin{align*}
    1 &: \gvar^1 := h^0 &   \text{hidden state at $t=0$}\\
    2 &: \gvar^2 := b &   \text{RNN bias}\\
    3 &: \Avar^3 := W   &   \text{RNN weights}\\
    4 &: \gvar^4 := U x^1 + a   &   \text{affine transform of input at $t=1$}\\
    5 &: \gvar^5 := \Avar^3 \gvar^1\\
    6 &: \gvar^6 := \gvar^5 + \gvar^2 + \gvar^4\\
    7 &: \hvar^7 := \phi(\gvar^6)   &   \text{hidden state at $t=1$}\\
    8 &: \gvar^8 := U x^2 + a   &   \text{affine transform of input at $t=2$, with same $U$ and $a$}\\
    9 &: \gvar^9 := \Avar^3 \hvar^7\\
    10 &: \gvar^{10} := \gvar^9 + \gvar^2 + \gvar^8\\
    11 &: \hvar^{11} := \phi(\gvar^{10})    &   \text{hidden state at $t=2$}
\end{align*}

More advanced RNNs like GRU or LSTM can be similarly expressed.

\subsection{Batchnorm, fully-connected}
\label{subsec:exampleBatchnorm}
Let $\tilde\phi(h) = \phi((h - \bar h) / \mathrm{std}(h))$ be batchnorm followed by coordinatewise nonlinearity $\phi$, where $h \in \R^B$ should be interpreted as a single neuron across a batch, and $\bar h = \f 1 B \sum_{i=1}^B h_i, \mathrm{std}(h) = \sqrt{\f 1 B \sum_{i=1}^B (h_i - \bar h)^2}.$
For example, let $x_1, \ldots, x_4$ be the batch of inputs.
\begin{align*}
    1 &: \gvar^1 := W^1 x_1\\
    2 &: \gvar^2 := W^1 x_2\\
    3 &: \gvar^3 := W^1 x_3\\
    4 &: \gvar^4 := W^1 x_4\\
    5 &: \gvar^5 := b^1 &   \text{layer 1 bias}\\
    6 &: \gvar^6 := \gvar^1 + \gvar^5\\
    7 &: \gvar^7 := \gvar^2 + \gvar^5\\
    8 &: \gvar^8 := \gvar^3 + \gvar^5\\
    9 &: \gvar^9 := \gvar^4 + \gvar^5\\
    10 &: \hvar^{10} := \tilde\phi(\gvar_6, \gvar_7, \gvar_8, \gvar_9)_1
        &   \text{layer 1 input 1 activations}\\
    11 &: \hvar^{11} := \tilde\phi(\gvar_6, \gvar_7, \gvar_8, \gvar_9)_2
        &   \text{layer 1 input 2 activations}\\
    12 &: \hvar^{12} := \tilde\phi(\gvar_6, \gvar_7, \gvar_8, \gvar_9)_3
        &   \text{layer 1 input 3 activations}\\
    13 &: \hvar^{13} := \tilde\phi(\gvar_6, \gvar_7, \gvar_8, \gvar_9)_4
        &   \text{layer 1 input 4 activations}
\end{align*}

The transformer without layernorm (in particular, the softmax and self-attention mechanism) can be expressed in a similar way.

\subsection{Convolution}
\label{subsec:exampleConv}
For any convolution weights $\{W^l_{\beta ij}\}_{\beta \in ker, i \in c', j \in [c]}$, $W^l_\beta$ is a dense matrix.
Suppose $x$ is an image input to the network with $s$ pixels and $c$ channels, $\{x_{\alpha j}\}_{\alpha \in pos, j \in [c]}$, so that $x_\alpha$ is vector of dimension $c$.
Then the $\alpha$th pixel, across all channels, of the convolution $W^l$ applied to $x$ can be described as
\begin{align*}
    \sum_{\beta\in ker} W^l_\beta x_{\alpha+\beta}
\end{align*}

Define
\begin{align*}
    \tilde x_{\alpha' j'}
        &=
            \sum W^1_{\beta ij} x_{\alpha+\beta, j}
\end{align*}

For demonstration, assume $ker = \{0, 1\}$ and $pos = [3]$ and the convolution is circular, and for simplicity assume we don't have biases.
\begin{align*}
    1 &: \gvar^1 := W^1_{0} x_{1}\\
    2 &: \gvar^2 := W^1_{0} x_{2}\\
    3 &: \gvar^3 := W^1_{0} x_{3}\\
    4 &: \gvar^4 := W^1_{1} x_{1}\\
    5 &: \gvar^5 := W^1_{1} x_{2}\\
    6 &: \gvar^6 := W^1_{1} x_{3}\\
    7 &: \gvar^7 := \gvar^1 + \gvar^5   &   \text{layer 1 pixel 1 preactivations}\\
    8 &: \gvar^8 := \gvar^2 + \gvar^6   &   \text{layer 1 pixel 2 preactivations}\\
    9 &: \gvar^9 := \gvar^3 + \gvar^4   &   \text{layer 1 pixel 3 preactivations}\\
    10 &: \hvar^{10} := \phi(\gvar^7)   &   \text{layer 1 pixel 1 activations}\\
    11 &: \hvar^{11} := \phi(\gvar^8)   &   \text{layer 1 pixel 2 activations}\\
    12 &: \hvar^{12} := \phi(\gvar^9)   &   \text{layer 1 pixel 3 activations}\\
    13 &: \Avar^{13} := W^2_0\\
    14 &: \Avar^{14} := W^2_1   &   \text{layer 2 weights}\\
    15 &: \gvar^{15} := W^2_0 \hvar^{10}\\
    16 &: \gvar^{16} := W^2_0 \hvar^{11}\\
    17 &: \gvar^{17} := W^2_0 \hvar^{12}\\
    18 &: \gvar^{18} := W^2_1 \hvar^{10}\\
    19 &: \gvar^{19} := W^2_1 \hvar^{11}\\
    20 &: \gvar^{20} := W^2_1 \hvar^{12}\\
    21 &: \gvar^{21} := \gvar^{15} + \gvar^{19}   &   \text{layer 2 pixel 1 preactivations}\\
    22 &: \gvar^{22} := \gvar^{16} + \gvar^{20}   &   \text{layer 2 pixel 2 preactivations}\\
    23 &: \gvar^{23} := \gvar^{17} + \gvar^{18}   &   \text{layer 2 pixel 3 preactivations}\\
    24 &: \hvar^{24} := \phi(\gvar^{21})   &   \text{layer 2 pixel 1 activations}\\
    25 &: \hvar^{25} := \phi(\gvar^{22})   &   \text{layer 2 pixel 2 activations}\\
    26 &: \hvar^{26} := \phi(\gvar^{23})   &   \text{layer 2 pixel 3 activations}\\
\end{align*}

\section{Additional Notations}

We will use teletype font $\gvar, \hvar, \Avar$, etc, to denote variables or nodes as defined in the straightline program.
The superscripts in this case, like $\gvar^l$, will mean that line number associated to such a variable.
In contrast, we will use normal font $g, h, A$, etc, to denote arbitrary variables of the correct type in a program (though sometimes we use $h$ to also denote var of type H \textit{or} G), but the superscripts, such as in $g^i$, are not attached to the semantics of the program.
In either case, $\gvar^l_k$ or $g^i_k$ will denote the scalar value at the $k$th position of $\gvar^l$ or $g^i$.
We write $\lno(g)$ for the line number of the G-node $g$ so that $g = \gvar^{\lno(g)}$ (same for H-nodes and A-nodes).
We write $\nvar(g)$ for the dimension of a G-node so that $\nvar^{\lno(g)} = \nvar(g)$ (similar for H-node).
Similarly, $\nvar_1(A)$ and $\nvar_2(A)$ gives the first and second dimensions of an A-node $A$.
Let $\Gnodes_\pi$ be the collection of all G-nodes of a skeleton $\pi$, $\Hnodes_\pi$ be the collection of all H-nodes, and let $\Gnodesin_\pi$ be the collection of all input G-nodes.
Sometimes we need to talk about all G-nodes on or before line $L$.
We will use $\Gnodes_\pi^{\le L}$ to denote such a set.
When $\pi$ is clear from context, we suppress the subscript $\pi$ for brevity.

\newcommand{\cdcle}[1]{\cdc_{\le #1}}
\newcommand{\cdclt}[1]{\cdc_{< #1}}
\begin{defn}
If $\cdc$ is a CDC, then let $\cdcle m \defeq \{\gvar^l \in \cdc: l \le m\}$ and $\cdclt m \defeq \{\gvar^l \in \cdc: l < m\}.$
\end{defn}

Given a kernel $K: \R^m \times \R^m \to \R$ and subsets $\mathcal X, \mathcal Y \sbe \R^m$, write $K(\mathcal X, \mathcal Y)$ for the corresponding $|\mathcal X| \times |\mathcal Y|$ submatrix of $K$, and write
$K|_{\mathcal X} = K(\mathcal X, \mathcal X)$ to be the restriction of $K$ to $\mathcal X.$

Given two random variables $X, Y$, and a $\sigma$-algebra $\Aa$, the notation $X|_\Aa \disteq Y$ or $X \disteq_\Aa Y$ means that for any integrable function $\phi$ and for any random varible $Z$ measurable on $\Aa$, $\EV \phi(X) Z = \EV \phi(Y)Z$.
We say that $X$ is distributed as (or is equal in distribution to) $Y$ conditional on $\Aa$.
In case $\Aa$ is the trivial $\sigma$-algebra, we just write $X \disteq Y$.
The expression $X \distto Y$ means $X$ converges to $Y$ in distribution.
If random variables $X^t \asto X^\infty$, then we write $X^\infty = \limas_{t\to \infty} X^t$.

\begin{defn}
For a function $\Phi: \R^n \to \R^n$, we define
$$\Vt \Phi(\Sigma) \defeq \EV[\Phi(z) \Phi(z)^\trsp: z \sim \Gaus(0, \Sigma)]$$
for a PSD matrix $\Sigma$.
When $\phi: \R \to \R$, we write $\Vt\phi$ to mean $\Vt$ applied to the function that acts coordinatewise by $\phi$.
\end{defn}

\section{Consequences}
\label{sec:Conseq}

\subsection{Warmup: MLP}
\label{subsec:warmupConseq}
We warm up by considering the GP correspondence, gradient dynamics, and NTK convergence of MLPs first.
We define a fully-connected, feedforward neural network $f(x; \theta), x \in \R^{n^0}$ as follows
\begin{align*}
    x^0(x) &\defeq x\\
    h^l(x) &\defeq \f 1 {\sqrt{n^{l-1}}} W^l x^{l-1}(x) + b^l\\
    x^l(x) &\defeq \phi(h^l(x))\\
    f(x; \theta) &\defeq \f 1 {\sqrt{n^{L}}} v^\trsp x^{L}(x)
\end{align*}
with $b^l \in \R^{n^l}$, $v \in \R^{n^L}$, and $W^l \in \R^{n^l \times n^{l-1}}$ for $l = 1, \ldots, L$.
These form the parameters $\theta$ of $f$.
(Note here that we follow prior notation and use $h$ for ``preactivation,'' but it is in fact equivalent to a G-var).
We sample $W^l_{ij} \sim \Gaus(0, (\sigma_w^l)^2), v_i \sim \Gaus(0, (\sigma_w^{L+1})^2)$ and $b^l_i \sim \Gaus(0, (\sigma_b^l)^2)$.
We can think of this parametrization as ``pulling out the $1/\sqrt{n^l}$ from Glorot initialization.''
This parametrization doesn't change the forward kernel $\Sigma^l$ (defined below), but it does change the scaling of gradients.

Define kernels $\Sigma^{L+1}: (\R^{n^0})^2 \to \R$ by
\begin{align*}
    \Sigma^1(x, x')
        &\defeq
            (\sigma_w^1)^2\f 1 {n^0} \sum_{i=1}^{n^0} x_i x'_i + (\sigma_b^1)^2
            \\
    \Sigma^l
        &\defeq
            (\sigma_w^l)^2\Vt \phi(\Sigma^{l-1}) + (\sigma_b^l)^2
            \\
    \Sigma^{L+1}
        &\defeq
            (\sigma_w^{L+1})^2\Vt \phi(\Sigma^{L})
            .
\end{align*}

For any parametrized function $f(x; \theta)$, the Neural Tangent Kernel can be in general defined as $K_\theta(x, x') = \la \nabla_\theta f(x; \theta), \nabla_\theta f(x'; \theta) \ra.$
In the case when $f(x; \theta)$ is defined as above, there is a scaling limit of $K_\theta$ when $\theta$ is randomized \citep{jacot_neural_2018}.
The ``proof'' given by \citet{jacot_neural_2018} was a sketch and most importantly was silent about its application of the gradient independence assumption (used when applying induction).
Below we give a formal proof of NTK convergence, but first we introduce a ``gradient kernel.''
Suppose we take $n^1, \ldots, n^L \to \infty$ in such a way that $n^{l} / n^{m} \to \alpha_{l, m}$ for constants $\alpha_{l, m} \in (0, \infty)$.
Then define $\Pi^l: (\R^{n^0})^2 \to \R$ by
\begin{align*}
    \Pi^{L}(x, x')
        &\defeq
            (\sigma_w^{L+1})^2
            \\
    \Pi^{l}
        &\defeq
            \alpha_{l+1, l} \Vt\phi'(\Sigma^{l+1}) \odot \Pi^{l+1}
            .
\end{align*}

\begin{thm}\label{thm:GPNTKconv}
Fix a finite set of inputs $\mathcal X \sbe \R^{n^0}.$
As $n^1, \ldots, n^L \to \infty$ with $n^{l} / n^{m} \to \alpha_{l, m}$, with parameter sampled as above,
\begin{align*}
    f(\mathcal X; \theta) &\distto \Gaus(0, \Sigma^{L+1}|_{\mathcal X})
\end{align*}
if the nonlinearity $\phi$ is $\alpha$-controlled with $\alpha <2$; 
and
\begin{align*}
    \f {n^L} {n^l} \sum_{i=1}^{n^l} \psi\lp\{\pdf f {x^l_i}(x)\}_{x \in \mathcal X}\rp
        &\asto
            \EV[\psi(\zeta): \zeta \sim \Gaus(0, \Pi^l|_{\mathcal X})],
            \text{ for all polynomially bounded $\psi$, $l \ge 1$}
            \\
    K_\theta|_{\mathcal X} 
        &\asto
            \sum_{l=1}^L \alpha_{l, L} \Pi^l \odot \Vt\phi'(\Sigma^l) \odot \f{\Sigma^l + (\sigma_w^l)^2 - (\sigma_b^l)^2}{(\sigma_w^l)^2}
            + \Sigma^{L+1}/(\sigma_w^{L+1})^2
            \\
        &=
            \sum_{l=1}^L
            \bigodot_{m=l}^{L} \Vt\phi'(\Sigma^{m})
            \odot \f{\Sigma^l + (\sigma_w^l)^2 - (\sigma_b^l)^2}{(\sigma_w^l)^2}
            + \Sigma^{L+1}/(\sigma_w^{L+1})^2
            .
\end{align*}
if the nonlinearity $\phi$ has a polynomially bounded weak derivative.
\end{thm}
This theorem formally justifies the computations made in \citet{poole_exponential_2016,schoenholz_deep_2017,jacot_neural_2018}
\begin{remk}
If we set $\sigma_w^l = \sigma_b^l = 1$ for all $l$, then we can recover the NTK recurrence relation given in \citet{jacot_neural_2018}.
The $\beta$ factor in \citet{jacot_neural_2018} on the bias can also be easily accounted for, but we will not consider such a parametrization here.
\end{remk}
\begin{proof}
Since $\mathcal X$ is finite, it suffices to just consider two inputs $x^{01} = x, x^{02} = x'$.
We can form a tensor program to model a fully-connected, feedforward network, with a batch of inputs.
We construct it implicitly as follows, where we use superscripts in $\p {\bullet} \bullet$ to denote semantically relevant quantities.

Define the input vars $\p {\tilde \gvar} {i} := W^1_{:i}$, $\p \Avar l := W^l$ for $l = 2, \ldots L$, and $\p {\bar \gvar} l := b^l$ for $l = 1, \ldots, L$.
Define $\p \gvar {1a} := \sum_{i=1}^{n^0} \p {\tilde \gvar} i x^{0a}_i$ for $a = 1, 2$,
$\p {\hat\gvar} {la} := \p \gvar {la} + \p {\bar \gvar} l$ (this represents $h^l$), $\p \hvar {la} := \phi(\p {\hat\gvar} {la})$ (this represents $x^l$), $\p \gvar {l+1,a} := \p \Avar {l+1} \p \hvar {la}$.
For simplicity, we assume that $\nvar_1(\p \Avar l) \ne \nvar_2(\p \Avar l)$ for all $l$, so that each ``layer'' belongs to a different CDC.
Below, we write $K$ for $K^\cdc$ where $\cdc$ is automatically understood based on the arguments; similarly we write $\mu$ for $\mu^\cdc$ with $\cdc$ implied.

We have the corresponding sampling hyperparameters $\sigma^{\lno(\p \Avar l)t} = \sigma_w^l$, $\sigma^{\lno(\p {\bar \gvar} l) t} = \sigma_b^l$, $\mu^{\cdcin t} = 0$ for all input G-vars, and $K^{\cdcin t}$ is such that $\p {\tilde \gvar} i _j \sim \Gaus(0, (\sigma_w^1/n^0)^2)$ and $\p {\bar \gvar} l _j \sim \Gaus(0, (\sigma_b^l)^2)$, for all $j$ in appropriate ranges.

Then we can compute $\mu = 0$ and
\begin{align*}
    K(\p {\hat \gvar} {la}, \p {\hat \gvar} {lb})
        &=
            \Sigma^l(x, x')
            \\
    K(\p {\gvar} {la}, \p { \gvar} {lb})
        &=
            K^\cdc(\p {\hat \gvar} {la}, \p {\gvar} {lb})
            \\
        &=
            \Sigma^l(x, x') - (\sigma_b^l)^2
            \\
    K(\p {\hat \gvar} {la}, \p {\bar \gvar} {l})
        &=
            (\sigma_b^l)^2
\end{align*}
and $K(g, g') = 0$ for all other pairs of G-vars $g, g'$.

Thus, by \cref{thm:notransposeLimit}, for any $(<2)$-controlled $\psi$,
\begin{align*}
    &\phantomeq\f 1 {n^l} \sum_{i=1}^{n^l} \psi(h^l(x^{0a}))\psi(h^l(x^{0b}))\\
    &=\f 1 {n^l} \sum_{i=1}^{n^l} \psi(\p {\hat \gvar} {la})\psi(\p {\hat \gvar} {lb})\\
    &\asto \EV_{(z, z') \sim \Gaus(0, \Sigma^l|_{x, x'})} \psi(z)\psi(z')
\end{align*}
where $\Sigma^l|_{x, x'} = 
\begin{pmatrix}
\Sigma^l(x, x) & \Sigma^l(x, x')\\
\Sigma^l(x', x) & \Sigma^l(x', x')\\
\end{pmatrix}
$.
Obviously, given $h^{L}(x)$ for all $x$, $f(x; \theta)$ is a Gaussian process with kernel $K(x, x') = (\sigma_w^{L+1})^2 \f 1 {n^L} \sum_{i=1}^{n^L} \phi(h^L(x)) \phi(h^L(x'))$ and mean 0.
Since $K \asto \Sigma^{L+1}$ over the randomization of parameters in layer $< L$, we have that $f$ itself is a Gaussian process with this kernel, in the limit.

Now we think about backprop.

We have $\pdf f {x^L} = \f 1 {\sqrt{n^{L}}} v.$
We can thus extend the tensor program by
$\p {\underline \gvar} {La} := v, \p {\underline \hvar} {la} := \p {\underline \gvar} {la} \odot \phi'(\p {\hat\gvar} {la}),  \p {\underline \gvar} {l-1, a} := \p \Avar l{}^\trsp \p {\underline \hvar} {la}$, for $a = 1, 2$ and for $l = L-1, \ldots, 2$.
Here $\p {\underline \gvar} l$ represents $\pdf f {x^l} \sqrt{n^{L}}$ and $\p {\underline \hvar} l$ represents $\pdf f {h^l} \sqrt{n^{L}}.$
For brevity we just wrote $\p \Avar l {}^\trsp$ for an implicitly defined transposed var of $\p \Avar l$.
We can compute $K(\p {\underline \gvar} {la}, \p {\underline \gvar} {lb}) = \Pi^l(x, x')$,
and $K(g, g') = 0$ for all other G-var pairs $g, g'$ not appearing in the original tensor program.

Because $v \sim \Gaus(0, (\sigma_w^{L+1})^2)$, and all $\p {\underline \gvar} {la}, \p {\underline \hvar} {la}$ are odd in $v$ (being linear in $v$), \cref{thm:gradIndep} can be applied.
Then, for $l = 1, \ldots, L$,
\begin{align*}
    \f 1 {n^{l}} \sum_{i,j=1}^{n^l, n^{l-1}} \pdf f {W_{ij}^l}(x) \pdf f {W_{ij}^l}(x')
    &=
        \f 1 {n^{l}} \sum_{i=j=1}^{n^l, n^{l-1}} 
            \lp \f 1 {\sqrt{n^{l-1}}} \pdf f {h^l_i}(x)  {x^{l-1}_j}(x)\rp 
            \lp \f 1 {\sqrt{n^{l-1}}} \pdf f {h^l_i}(x') {x^{l-1}_j}(x')\rp
        \\
    &=  
        \f 1 {n^{l-1} n^l}
            \lp \sum_{i=1}^{n^l} \pdf f {h^l_i}(x) \pdf f {h^l_i}(x')\rp
            \lp \sum_{j=1}^{n^{l-1}} {x^{l-1}_j}(x) {x^{l-1}_j}(x') \rp
        \\
    &=
        \f 1 {n^{l-1} n^l}
            \lp \sum_{i=1}^{n^l} \p {\underline \hvar} {l a} \p {\underline \hvar} {l b}\rp
            \lp \sum_{j=1}^{n^{l-1}} \p {\hvar} {l-1, a} \p {\hvar} {l-1, b} \rp
        \\
    &\asto
        \lp \EV \underline z^l_a \phi'(z^l_a) \underline z^l_b \phi'(z^l_b) \rp
        \lp \EV \phi(z^{l-1}_a) \phi(z^{l-1}_b) \rp
        \\
    &\pushright{\text{by \cref{thm:gradIndep}}}
        \\
    &=
        \EV \underline z^l_a  \underline z^l_b \EV \phi'(z^l_a) \phi'(z^l_b)
        \EV \phi(z^{l-1}_a) \phi(z^{l-1}_b)
\end{align*}
where $(\underline z^{l}_a, \underline z^{l}_b) \sim \Gaus(0, K|_{\p {\underline \gvar} {la}, \p {\underline \gvar} {lb}}), $ and independently $(z^l_a, z^l_b) \sim \Gaus(0, K|_{\p {\hat\gvar} {la}, \p {\hat\gvar} {lb}})$.
Thus the above is just $[\Pi^l \odot \Vt\phi'(\Sigma^l) \odot (\Sigma^{l} - (\sigma_b^l)^2) / (\sigma_w^l)^2] (x, x').$

On the other hand,
\begin{align*}
    \f 1 {n^l} \sum_{i} \pdf f {b_{i}^l}(x) \pdf f {b_{i}^l}(x')
    &=
        \f 1 {n^l} \sum_{i} \pdf f {h_{i}^l}(x) \pdf f {h_{i}^l}(x')
        \\
    &=
        \f 1 {n^l} \lp \sum_i \p {\underline \hvar} {l a} \p {\underline \hvar} {l b}\rp
        \\
    &\asto
        \EV \underline z^l_a \phi'(z^l_a) \underline z^l_b \phi'(z^l_b)
        \\
    &=
        \EV \underline z^l_a \underline z^l_b \EV \phi'(z^l_a) \phi'(z^l_b)
        \\
    &=
        [\Pi^l \odot \Vt\phi'(\Sigma^l)](x, x')
\end{align*}
where $\underline z^{l}_a, \underline z^{l}_b, z^{l}_a, z^{l}_b$ are sampled as above.
Also, $\f 1 {n^L} \sum_{i} \pdf f {v_i}(x) \pdf f {v_i}(x') = \f 1 {n^L} \sum_{i} h_i^L(x) h_i^L(x') \asto \Sigma^{L+1}(x, x') / (\sigma_w^{L+1})^2$ as deduced before.

Thus
\begin{align*}
    \Theta^L
        &\asto
            \sum_{l=1}^L \alpha_{l, L} \Pi^l \odot \Vt\phi'(\Sigma^l) \odot \f{\Sigma^l + (\sigma_w^l)^2 - (\sigma_b^l)^2}{(\sigma_w^l)^2}
            + \Sigma^{L+1}/(\sigma_w^{L+1})^2
            \\
        &=
            \sum_{l=1}^L \alpha_{l, L} \prod_{m=l}^{L-1} \alpha_{m+1, m} \bigodot_{m=l}^{L-1} \Vt\phi'(\Sigma^{m+1}) \odot \Vt\phi'(\Sigma^l) \odot \f{\Sigma^l + (\sigma_w^l)^2 - (\sigma_b^l)^2}{(\sigma_w^l)^2}
            + \Sigma^{L+1}/(\sigma_w^{L+1})^2
            \\
        &=
            \sum_{l=1}^L
            \bigodot_{m=l}^{L} \Vt\phi'(\Sigma^{m})
            \odot \f{\Sigma^l + (\sigma_w^l)^2 - (\sigma_b^l)^2}{(\sigma_w^l)^2}
            + \Sigma^{L+1}/(\sigma_w^{L+1})^2
\end{align*}
\end{proof}

\paragraph{Global mean pooling as readout layer.}
Now suppose that $f(x; \theta) = \f 1 {n^L} \onev^\trsp x^L(x)$.
As in the proof above, we can construct a program $\pi$ for computing $f$ and its gradients over two inputs $x^{01}, x^{02}$:

\begin{description}
    \item[Forward]
        Define the input vars $\p {\tilde \gvar} {i} := W^1_{:i}$, $\p \Avar l := W^l$ for $l = 2, \ldots L$, and $\p {\bar \gvar} l := b^l$ for $l = 1, \ldots, L$.
        Define $\p \gvar {1a} := \sum_{i=1}^{n^0} \p {\tilde \gvar} i x^{0a}_i$ for $a = 1,2$,
        $\p {\hat\gvar} {la} := \p \gvar {la} + \p {\bar \gvar} l$ (this represents $h^l$), $\p \hvar {la} := \phi(\p {\hat\gvar} {la})$ (this represents $x^l$), $\p \gvar {l+1,a} := \p \Avar {l+1} \p \hvar {la}$.
    \item[Backward]
        We set
        $\p {\underline \gvar} {La} := \onev, \p {\underline \hvar} {la} := \p {\underline \gvar} {la} \odot \phi'(\p {\hat\gvar} {la}),  \p {\underline \gvar} {l-1, a} := \p \Avar l{}^\trsp \p {\underline \hvar} {la}$, for $a = 1, 2$ and for $l = L-1, \ldots, 2$.
        Here $\p {\underline \gvar} l$ represents $n^{L}\pdf f {x^l}$ and $\p {\underline \hvar} l$ represents $n^{L} \pdf f {h^l}.$
\end{description}

Now we construct the detransposition $\check \pi$ of $\pi$ (see \cref{subsec:exampleDetransposition} for a concrete example of detransposition).

The forward part of $\check \pi$ is almost identical to that of $\pi$:
\begin{description}
    \item[Forward]
        Define the input vars $\p {\check{\tilde \gvar}} {i} := W^1_{:i}$, $\p {\check \Avar} l := W^l$ for $l = 2, \ldots L$, and $\p {\check {\bar \gvar}} l := b^l$ for $l = 1, \ldots, L$.
        Define $\p {\check \gvar} {1a} := \sum_{i=1}^{n^0} \p {\check {\tilde \gvar}} i x^{0a}_i$ for $a = 1,2$,
        $\p {\check {\hat\gvar}} {la} := \p {\check \gvar} {la} + \p {\check {\bar \gvar}} l$ (this represents $h^l$), $\p {\check \hvar} {la} := \phi(\p {\check {\hat\gvar}} {la})$ (this represents $x^l$), $\p {\check \gvar} {l+1,a} := \p {\check \Avar} {l+1} \p {\check \hvar} {la}$.
        Finally, we set $\varphi(\bullet) = \check \bullet$ for all vars defined here.
\end{description}
Here we have automatically simplified the detransposition of $\p \gvar {la}$ produced from \cref{defn:detransposition}, by identifying $\varphi(\gvar^l)$ and $\varphig(\gvar^l)$, which in this case are the same.

Now the backward part
\begin{description}
    \item[Backward]
        Let $\p {\check \Avar} l {}'$ be an input A-var sampled iid as $\p \Avar l {}^\trsp$.
        We set
        $\p {\check{\underline \hvar}} {La} := \p {\check{\underline \gvar}} {La} := \onev$ (representing $n^L$ times the gradient at $x^L$), $\varphi(\p{\underline \hvar}{la}) =\p {\check{\underline \hvar}} {la} := \p {\check{\underline \hvar}} {la} \odot \phi'(\p {\check{\hat\gvar}} {la}),  \varphig(\p{\underline \gvar}{l-1,a}) = \p {\check{\tilde{\underline \gvar}}} {l-1, a} := \p {\check{\Avar}} l{}' \p {\check{\underline \hvar}} {la}$, $\varphi(\p{\underline \gvar}{l-1,a}) = \p {\check{\underline \hvar}} {l-1, a} := \p {\check{\tilde{\underline \gvar}}} {l-1, a} + \p \avar {l-1, a} \p \hvar {l-1, a}$ for $a = 1, 2$ and for $l = L-1, \ldots, 2$, and for $\p \avar {la}$ computed via the derivative rule of \cref{thm:generalTensorP}.
        Specifically, we have, by a simple induction,
        \begin{align*}
        \p \avar {L-1, a} &= \alpha_{L, L-1}(\sigma_w^L)^2\EV[\phi''(y): y \sim \Gaus(0, \Sigma^L_{aa})]\\
        \p \avar {l a} &= \p \avar {l+1, a} \alpha_{l+1,l}(\sigma_w^{l+1})^2\EV \pd_y(\phi(y) \phi'(y))
            = \p \avar {l+1, a} \alpha_{l+1,l} (\sigma_w^{l+1})^2 (\EV \phi'(y)^2 + \phi(y) \phi''(y)),\\
            &\qquad \text{with } y \sim \Gaus(0, \Sigma^{l+1}_{aa})
        \end{align*}
         The derivatives here should be interpreted as tempered distributions in general, testing against the Gaussian density of $y$.
        Note that if $\phi$ is odd, then $\phi''$ is odd, so that $\avar^{L-1, a}  = 0 = \avar^{la}$ for all $l < L-1$.
        If $\phi$ is ReLU, then $\phi''$ is the Dirac Delta tempered distribution at 0, so that $\avar^{L-1, a} = 1/\sqrt{2\pi\Sigma^L_{aa}}$.
\end{description}

In the forward pass, $(\p {\hat \gvar} {la}, \p {\hat \gvar} {lb}) ``\disteq" \Gaus(0, \Sigma^l|_{a,b})$ as before.
In the backward pass, $(\p {\check{ \underline {\tilde \gvar}}} {la}, \p {\check {\underline {\tilde \gvar}}} {lb}) ``\disteq" \Gaus(0, \Pi^l|_{a,b})$, where
\begin{align*}
    \Pi^{L-1}|_{a,b} &= \alpha_{L, L-1} (\sigma^L_w)^2 \Vt \phi' (\Sigma^L|_{a,b})\\
    \Pi^l|_{a,b} &= \alpha_{l+1, l}(\sigma_w^{l+1})^2 (\Pi^{l+1}|_{a,b} \odot \Vt\phi'(\Sigma^{l+1}|_{a,b}) + (\p \avar {l+1,a}, \p \avar {l+1,b})^{\otimes 2} \odot \Vt \phi(\Sigma^{l+1}|_{a,b}))
\end{align*}
and
\begin{align*}
    \f 1 {n^l} \sum_{i=1}^{n^l} \p {\underline \hvar} {la} _i \p {\underline \hvar} {lb} _i
        &\asto
            \Pi^l_{ab} \Vt\phi'(\Sigma^l)_{ab} + \p \avar {la} \p \avar {lb} \Vt(\phi\phi')(\Sigma^l)_{ab}
\end{align*}

Thus
\begin{cor}
The NTK of the MLP above with global mean pooling readout layer converges a.s. to
\begin{align*}
    &\mathrm{NTK}(x_a, x_b) \asto\\
    &
        \Vt\phi'(\Sigma^L)_{ab}
        \f{\Sigma^{L}_{ab}+(\sigma^L_w)^2 - (\sigma^L_b)^2}{\sigma^L_w}
    +
        \sum_{l=1}^{L-1}
        \alpha_{l, L}
            \big(\Pi^l_{ab} \Vt\phi'(\Sigma^l)_{ab} + \p \avar {la} \p \avar {lb} \Vt(\phi\phi')(\Sigma^l)_{ab}
            \big)
        \f{\Sigma^{l}_{ab}+(\sigma^l_w)^2 - (\sigma^l_b)^2}{\sigma^l_w}
\end{align*}
\end{cor}

\begin{cor}
If the readout layer is global mean pooling in an MLP and the last layer nonlinearity is odd, then the \emph{Gradient Independence Assumption} can be applied to give the correct computation of the gradient covariance and the NTK.
\end{cor}

Now that we have warmed up a little, we will work with tensor programs and apply \cref{thm:notransposeLimit,thm:gradIndep,thm:generalTensorP} more informally.

\subsection{Warmup: Semicircle Law}
\label{warmup:semicirclelaw}
\cref{thm:generalTensorP} has enough power to rederive the semicircle law for the Gaussian Orthogonal Ensemble \cite{tao_topics_2012}.
\begin{defn}
The Gaussian Orthogonal Ensemble (GOE) is the sequence of matrices $(W_n)_{n \ge 0}$ defined as follows:
Let $X_{ij} \sim \Gaus(0, 1)$, iid, for all $i, j \in \N$.
Then set $W_{n} \defeq \f 1 {\sqrt{2 n}} (X_{ij} + X_{ji})_{i,j \in [n]}$.
\end{defn}
\newcommand{\musc}{{\mu_{\mathrm{sc}}}}

\begin{defn}
The empirical spectral distribution (ESD) $\mu_{W_n}$ of $W_n$ is given by $\mu_{W_n} \defeq \f 1 n \sum_{i=1}^n \delta_{\lambda_i(W_n)}$, where $\delta_x$ is the Dirac Delta measure at $x$, and $\lambda_i$ are the eigenvalues in decreasing (in $i$) order.
\end{defn}

\begin{defn}
The semicircle law $\musc$ is defined to be the distribution with density $\propto \sqrt{4 - x^2}$.
\end{defn}

\begin{defn}
A random distribution $\mu$ on $\R$, i.e. a random variable taking values in the space of probability distributions on $\R$,
converges to a deterministic distribution $\mu^*$ \textit{almost surely}, if for all compactly supported continuous function $f$, $ \EV_{z \sim \mu} f(z) \asto \EV_{z \sim \mu^*} f(z).$
\end{defn}

To prove that $\mu_{W_n}$ converges almost surely to the semicircle law $\musc$, it suffices to compute the (polynomial) moments of $\mu_{W_n}$ and show that it converges almost surely to the moments of $\musc$ as $n \to \infty$ \cite{tao_topics_2012}.
It's well known that the odd moments of $\musc$ are 0 and for even $k$,  $\EV_{X \sim \musc} X^{k} = C_{k/2}$, where $C_j$ is the $j$th Catalan number defined by
\begin{align*}
    C_0
        &= 
            1,
            \qquad
    C_j
        =
            \sum_{i=0}^{j-1} C_i C_{j-1-i}.
\end{align*}
Now,
$$\EV_{X \sim  \mu_{W_n}} X^k = \EV_{W_n} \f 1 n \tr\lp W_n^k\rp = \EV_{W_n} \EV_{v \sim \Gaus(0, I_n)} \f 1 n v^\trsp W_n^k v.$$

The latter expectation can be expressed as a tensor program:
We couple the time index $t$ to $t=n$.
Define the A-var $A^{1t}$ to be an input var, with $\nvar_1(A^{1t}) = \nvar_2(A^{1t}) = n$, sampled $A^{1t}_{ij} \sim \Gaus(0, 1/\sqrt{2\nvar^2(A^{1t})}) = \Gaus(0, 1/\sqrt{2n})$, and define $A^{2t} = A^{1t}{}^\trsp.$
Thus $A^{1t} + A^{2t}$ represents $W_n$.
Set $g^0 := v$, an input var, sampling $g^0_i \sim \Gaus(0, 1)$.
Inductively, set $g'{}^j := A^1 g^{i-1}, g''{}^j := A^2 g^{i-1},$ and $g^j := g'{}^j + g''{}^j$.
Thus $g^j$ represents $W_n^j v$ (where the superscript $j$ represents an index in $g^j$ while it represents an exponent in $W_n^j$).
We aim to compute the limit of
$\f 1 n \sum_{i=1}^n v_i (W_n^k v)_i = \f 1 n \sum_{i=1}^n g^0_i g^k_i$
as $n \to \infty$.

This limit is prescribed by \cref{thm:generalTensorP}.
The detransposition $\check \pi$ of the above program can be described by the following:
Define $\check A^{1} := \varphi(A^{1}), \check A^{2} := \varphi(A^{2})$ (so that they are independently sampled in $\check \pi$), and 
\begin{align*}
\check h^0 := \check g^0 &:= v = \varphi(g^0),\\
\check{g}'{}^{j} &:= \check A^{1} \check h^{j-1} = \varphig(g'{}^j), \\
\check{g}''{}^{j} &:= \check A^{2} \check h^{j-1} = \varphig(g''{}^j), \\
\check h'{}^j &:= \check g'{}^j + \sum_{i=0}^{j-2} \avar'{}^j_i \check h''{}^i = \varphi(g'{}^j),\\
\check h''{}^j &:= \check g''{}^j + \sum_{i=0}^{j-2} \avar''{}^j_i \check h'{}^i = \varphi(g''{}^j),\\
\check h{}^j &:= \check h'{}^j + \check h''{}^j = \varphi(g^j),
\end{align*}
where $\avar'{}^j_i$ (resp. $\avar''{}^j_i$) is computed by differentiating $\fvar^{\check h^j}$ via \cref{thm:generalTensorP}, because it can be easily seen that $\fvar^{\check h^j}$ is always a linear function of $\{\check g'{}^r, \check g''{}^r\}_{r=1}^{j-1}$.
In fact, a symmetry argument shows that
$\fvar^{\check h^j}(Z) = \sum_{r=1}^{j-1} b_{r}^j (Z^{\check g'{}^r} + Z^{\check g''{}^r}) + b_0^j Z^{\check g^0}$
for some coefficients $\{b_r^j\}_{r=0}^j$.
An easy inductive argument shows that $b_r^j$ satisfies the recurrence
\begin{align*}
    b_0^0 &= 1,\\
    \forall r \not\in [0, j], b^j_r &= 0,\\
    \forall r<j, b_r^j &= \sum_{i=r}^{j-2} b_r^i b_{i+1}^{j-1},\\
    b_j^j &= 1
\end{align*}
These equations have the unique solution 
\begin{align*}
    b^j_r =
    \begin{cases}
    C_{(j-r)/2} &   \text{if $j-r$ is even}\\
    0   &   \text{else.}
    \end{cases}
\end{align*}

Simultaneously, another easy inductive argument shows that $\mu{\check \cdc} = 0$ and $K^{\check \cdc}(\check g^0, \check g^j) = 0$ for all $j > 0$.
Thus \cref{thm:generalTensorP} yields, for $Z \sim \Gaus(\mu^{\check \cdc}, K^{\check \cdc})$,
\begin{align*}
    \f 1 n \tr(W^k_n) = \f 1 n \sum_{i=1}^n g^0_i g^k_i
    &\asto
    \EV Z^{\check g^0} \lp \sum_{r=1}^{k-1} b_{r}^k (Z^{\check g'{}^r} + Z^{\check g''{}^r}) + b_0^k Z^{\check g^0} \rp\\
    &=
        \EV b_0^k (Z^{\check g^0})^2
        \\
    &=
        b_0^k = 
    \begin{cases}
    C_{k/2} &   \text{if $k$ is even}\\
    0   &   \text{else.}
    \end{cases}
\end{align*}
as desired.

\subsection{Warmup: Marchenko-Pastur Law}
\label{warmup:marchenkopasturlaw}
\newcommand{\mump}[1]{\mu_{\mathrm{mp}(#1)}}
\renewcommand{\p}[2]{#1^{(#2)}}

Suppose $Y_{ij} \sim \Gaus(0, 1)$ for all $i, j \in \N$, and $\p Y n = (Y_{ij})_{i\in[m_n], j\in[n]}$ for a sequence of $\{m_n\}_n$ satisfying $\lim_{n\to\infty} \f {m_n} n \to \alpha \in (0, \infty)$.
The Marchenko-Pastur Law says that the spectral distribution of $\f 1 n \p Y n \p Y n{}^\trsp$ converges almost surely to $\mump \alpha,$ defined as
$$(1 - \inv\alpha)_+ \delta_0 + \f 1 {\alpha 2 \pi x} \sqrt{(b-x)(x-a)} \ind_{[a, b]}(x) \dd x$$
where $(r)_+ = r$ if $r > 0$ and 0 else, and
$a = (1 - \sqrt \alpha)^2$ and $b = (1 + \sqrt \alpha)^2.$
We can again show this via the moment method.

We define a tensor program $\pi$ as follows.
Couple $n = t$.
Let $A := \p Y n /\sqrt n$ be an input A-var and let $A' := A^\trsp$.
Let $g^0 := v \sim \Gaus(0, I_{m_n}) = \Gaus(0, I_{\nvar(g^0)})$ be an input G-var.
Define recursively
\begin{align*}
    \underline g^i &:= A' g^{i-1}\qquad
    g^{i} := A \underline g^i
\end{align*}
We seek to compute $\limas_n \f 1 {m_n} \tr (AA^\trsp)^k = \limas_n \f 1 {m_n} \EV_{v \in \Gaus(0, I_{m_n})} v^\trsp (AA^\trsp)^k v = \limas_{t\to\infty} \f 1 {\nvar(g^{0t})} \sum_{i=1}^{\nvar(g^{0t})} g^{0t}_i g^{kt}_i.$

The detransposition $\check \pi$ of the above, is as follows.
Set $\check A, \check A'$ to be input A-vars, corresponding to $\varphi(A), \varphi(A')$, and sampled iid as such, so that $\sigma^{\infty}(\check A) = 1$ and $\sigma^\infty(\check A') = \sqrt \alpha.$
Then define
\begin{align*}
    \check h^0 &:= \check g^0 = v \sim \Gaus(0, I_{m_n})\\
    \check {\tilde {\underline g}}^i &:= \check A' \check {h}^{i} = \varphig(\underline g^i)\\
    \check {\tilde g}^i &:= \check A \check {\underline h}^{i-1} = \varphig(g^i)\\
    \check {\underline h}^i &:= \check {\tilde {\underline g}}^i + \sum_{j=0}^{i-1} \underline\avar_j^i \check {\underline h}^j = \varphi(\underline g^i)\\
    \check h^i &:= \check {\tilde g}^i + \sum_{j=0}^{i-1} \avar_j^i \check h^j = \varphi(g^i)
\end{align*}
where $\avar^i_j$ (resp. $\underline \avar^i_j$) is computed through the derivative rule of \cref{thm:generalTensorP}.
By a simple inductive argument, we see that we can express
\begin{align*}
    \check h^i
        &=
            \sum_{j=0}^{i-1} b_j^i \check g^j,
            \qquad
    \check {\underline h}^i
        =
            \sum_{j=0}^{i-1} {\underline b}_j^i \check {\underline g}^j
\end{align*}
for some set of coefficients $\{b^i_j, {\underline b}^i_j\}_{i\ge j \ge 0}$.
Then it's easy to see that they satisfy the recurrence
\begin{align*}
    b^0_0 &= \underline{b}^0_0 = 1,
    \\
    \forall j < i, b^i_j &= \sum_{k=j}^{i-1} \underline{b}^{i-1}_k b^k_j,
    \quad
    b^i_i = 1,
    \\
    \forall j < i, \underline{b}^i_j &= \alpha\sum_{k=j+1}^{i} b^i_k \underline{b}^{k-1}_j,
    \quad
    \underline{b}^i_i = 1.
\end{align*}

We claim that the solution to these equations is given by
\begin{align*}
    b^i_j &= M_{i-j}, \\
    \forall j < i, \underline{b}^i_j &= \alpha M_{i-j},\qquad \underline b^i_i = 1,
\end{align*}
where $M_r = \EV[x^r: x \sim \mump\alpha]$.
It suffices to verify the following Catalan-like identity
\begin{align}
    M_s = \alpha \sum_{r=1}^{s-2} M_r M_{s-1-r} + (1+\alpha) M_{s-1}. \label{eqn:catalanIdentityMP}
\end{align}
This can be done by the change of variable in the integral of $\EV[x^r: x \sim \mump\alpha]$ to get
\begin{align*}
    \EV[x^r: x \sim \mump\alpha] &= \EV[(\sqrt \alpha y + 1 + \alpha)^{r-1}: y \sim \musc]\\
        &=
            \sum_{k=0}^{\lfloor (r-1)/2 \rfloor} \alpha^k (1 + \alpha)^{r-1-2k} \binom{r-1}{2k} C_k
\end{align*}
where $C_k$ is the $k$th Catalan number.
Then one can verify \cref{eqn:catalanIdentityMP} by expanding and applying the Catalan identity $C_k = \sum_{i=0}^{k-1} C_i C_{k-1-i}$ repeatedly.

Finally, an easy inductive argument shows that $\mu^{\check \cdc} = 0$ and $K^{\check \cdc}(\check g^0, \check g^j) = 0$ for all $j > 0$.
Thus, we have 
$$
\limas_n \f 1 {m_n} \tr (AA^\trsp)^k
= \limas_t \f 1 {\nvar(g^{0t})} \sum_{i=1}^{\nvar(g^{0t})} g^{0t}_i g^{kt}_i
= \limas_t \EV \f 1 {\nvar(\check g^{0t})} \sum_{i=1}^{\nvar(\check g^{0t})} \check g^{0t}_i \check h^{kt}_i 
= \EV_{Z \sim \Gaus(\mu^{\check \cdc}, K^{\check \cdc})} b^k_0 (Z^{\check g^0})^2 = b^k_0 = M_{k}$$
as desired.

\subsection{DNN-GP correspondence}
\label{subsec:DNNGP}
\newcommand{\nout}{{n_{\mathrm{out}}}}
\newcommand{\nin}{{n_{\mathrm{in}}}}
\newcommand{\nhid}{{n_{\mathrm{hid}}}}

Suppose $\rho = F(z; \theta)$ is the part of a neural network that takes an input embedding $z = (z^1, \ldots, z^B)$ and produces a representation $\rho = (\rho^1, \ldots, \rho^m)$ of it.
For example, $z$ can be $Ax$ for an input $x$ and an embedding matrix $A$, or $(Ax^1, \ldots, Ax^B)$ for a sequence/batch of inputs $(x^i)_{i=1}^B$ (say when $x^1, \ldots, x^B$ is a sequence of tokens to be processed by an RNN, or when they form a batch, perhaps to be processed by batchnorm), or $(A^1 x^1, \ldots, A^B x^B)$ when they form the pixel vectors across the channels of an input image in the case of CNN, perhaps in combination with RNNs/batchnorm.
Similarly, $\rho$ can be a vector representation in a MLP or a sequence/batch of vector representations in the case of RNN/batchnorm/CNN.
The neural network then converts $\rho$ to an output via some linear transformation, say $\rho \mapsto (v^1 {}^\trsp \rho^1, \ldots, v^m {}^\trsp \rho^m)$, where each $v^i$ is a vector of appropriate size, and $v^i$ is allowed to equal to $v^j$ whenever they have the same shape.
Note that this scenario is general enough to cover simultaneous computation of a neural network on a batch of input, where $\rho$ can be partitioned into the corresponding representations of each parallel output.

Suppose $F(z; \theta)$ can be represented by a tensor program $\pi$ where $z$ and $\theta$ appear as input G- and A-vars; let the output $(\rho^1, \ldots, \rho^m)$ be represented by H- or G-vars $h^1, \ldots, h^m$ of $\pi$.
When the input embedding is linear, $z = (A^1 x^1, \ldots, A^B x^B)$, and its matrices $A^1, \ldots, A^B$ are sampled from zero mean Gaussian distributions, $(z^1, \ldots, z^B)$ is jointly Gaussian with a covariance depending on pairwise products between $x^1, \ldots, x^B$.
Furthermore, if $\theta$ is randomized by Gaussians according to \cref{subsec:setup} (with some set of compatible sampling hyperparameters $\sigma^l$, $\mu^\cdcin$, etc), then
by \cref{thm:generalTensorP} we get
\begin{cor}[DNN-GP correspondence]\label{cor:DNNGPappendix}
If all $\fvar^l$ of $\pi$ are polynomially bounded and almost sure rank convergence holds for $\check \pi$, then $\f 1 {\nvar(h^i)} h^i{}^\trsp h^j \asto C_{ij}$ for some PSD matrix $C$, whenever $\nvar(h^i) = \nvar(h^j)$, as the dimensions $\{\nvar^l\}_l$ of $\pi$ go to infinity.
The kernel $C$ can be computed via \cref{thm:generalTensorP}. 
\emph{A fortiori}, if each $v^i$ of the readout layer is sampled from $\Gaus(0, 1/\nvar(v^i))$, where for each $i \ne j$, either $v^i = v^j$ or $v^i$ is independent from $v^j$,
then the neural network output $(v^1 {}^\trsp \rho^1, \ldots, v^m {}^\trsp \rho^m) \distto \Gaus(0, C')$ in this limit, for 
\begin{align*}
    C'_{ij}
        &=
            \begin{cases}
            C_{ij}   &   \text{if $v^i = v^j$}\\
            0   &   \text{otherwise.}
            \end{cases}
\end{align*}
\end{cor}

For example, 
\begin{enumerate}
    \item if $z = (A x^1, \ldots, A x^B)$ is just a batch of inputs where each $A x^i$ is processed by the same neural network $f$ in parallel, and the network outputs $(v {}^\trsp \rho^1, \ldots, v {}^\trsp \rho^m)$ for readout weights $v$, then \cref{cor:DNNGPappendix} says $f$ converges to a Gaussian Process in distribution in the infinite width limit.
    \item if $z = (A x^{ij})_{i=j=1}^{BS}$ represents the embedding of a batch of $B$ sequences of length $S$, and the network is an RNN that processes each sequence in parallel, in a seq2seq fashion, then \cref{cor:DNNGPappendix} says $f$ converges to a (multivariate) Gaussian Process in distribution in the infinite width limit.
    \item we obtain similar GP convergence results for any standard architecture.
\end{enumerate}

\subsection{Gradient Independence Assumption}
\label{subsec:gradIndAssm}

\cref{thm:GPNTKconv} already shows that gradient independence assumption leads to the correct computation for MLPs.

In general, if, as before, $F(z; \theta)$ is the body of the network that takes an input embedding to a representation, and it can be represented by a tensor program $\pi$ having no line of type \ref{linetype:trsp}, then backprop can be represented by an extended program as in \cref{sec:zeroMeanGrad} with $v^i$ being readout layer weights.
Thus, if $v^i$ are sampled with zero mean and all nonlinearities have polynomially bounded weak derivatives, then \cref{thm:gradIndep} applies and we can compute the gradient dynamics by computing $K^\cdc$ and $\mu^\cdc$ according to \cref{sec:zeroMeanGrad}, which allows us to pretend that the G-vars in $\pi$ are independent from the weights used in the backward pass.
This is in particular true if $F(z; \theta)$ has a standard architecture without batchnorm (with no transposed weight sharing in the forward pass).
Batchnorm is not covered by our theorems because its gradient has singularities, for example at the origin.
However, based on the simulations of \citet{yang_mean_2018}, \cref{thm:gradIndep} seems to hold even when batchnorm is involved.

\paragraph{Singular value distribution.}
Let $F(z; \theta)$ be as above.
Denote by $J$ its Jacobian in $z$.
\citet{pennington_resurrecting_2017} applied free probability theory to compute the eigenvalues of $JJ^\trsp$ and hence of the singular value of $J$, when $F(z; \theta)$ represents a MLP.
Thus $J$ can be expressed as $D^L W^L \cdots D^2W^2D^1W^1$ for weight matrices $W^l$ for each layer $l$ and diagonal matrices $D^l = \Diag(\{\phi'(h^{l+1})\})$.
Specifically, the authors compute the Stieljes transform of $JJ^\trsp$ and then its S-transform by leveraging the latter's compatibility with matrix multiplication.
Crucial in this computation is the assumption that $D^l$ and $W^l$ are asymptotically free, allowing the application of S-transform.
We now justify this assumption.

The Stieljes and S-transform methods can be thought of a more nicely packaged way of applying the moment method \citep{tao_topics_2012}, i.e. computing $\tr((JJ^\trsp)^k)$ for each $k$.
Thus it suffices to show that, in the computation of $\tr((JJ^\trsp)^k)$, $\{D^l\}_l$ can be thought of as independent of $\{W^l\}_l.$

Now $\tr((JJ^\trsp)^k) = \EV_{a \sim \Gaus(0, I)} a^\trsp (JJ^\trsp)^k a$.
The computation $(JJ^\trsp)^k a$ can be expressed with a tensor program:
If $\pi$ represents the computation of $F$ (forward pass), $\pi'$ represents $J^\trsp$, i.e. backprop from gradient vector $a$ (so that $\tilde \pi = \pi | \pi'$ is an extended program of the form described in \cref{sec:zeroMeanGrad}), and $\pi''$ represents $J$, then $(JJ^\trsp)^k a$ is given by the output of $\hat \pi = \pi | (\pi' \| \pi'' \| \cdots \| \pi' \| \pi'') $.
Here, $|$ denotes concatenation and $\|$ denotes ``piping'', so that the output of $\rho$ is inserted as the input of $\tau$ in $\rho \| \tau$.
Then $\lim \EV_{a \sim \Gaus(0, I)} a^\trsp (JJ^\trsp)^k a$ can be computed via \cref{thm:generalTensorP}.
Finally, it only remains to notice that $K(g, h') = 0$ for any G-var of $\pi$ and H-var (of G-var) of $(\pi' \| \pi'' \| \cdots \| \pi' \| \pi'') $ other than $a$ because $h'$ is always odd in $a$ (apply the same reasoning from proof of \cref{thm:gradIndep}).
Thus $\limas \EV_{a \sim \Gaus(0, I)} a^\trsp (JJ^\trsp)^k a$ has the same limit as if the A-vars of $\pi$ are independent from the rest of $\hat \pi$.

In fact, this reasoning, applied to mixed moments, establishes that $\{D^i\}_i \cup \{W^j\}_j$ are {\it almost surely asymptotically free}.

\begin{cor}
In the MLP above, let its hidden layer widths $\{n^l\}_l$ go to infinity such that $n^l/n^{l'} \to \alpha_{l,l'} \in (0, \infty)$ for some constants $\alpha_{l,l'}$.
Then, for $X_1,\ldots, X_k$ chosen from $\{D^l, W^l, W^l{}^\trsp\}_l$ such that the sizes match and $X_1 \cdots X_k$ is a square matrix, 
\begin{align*}
    \f 1 {n^L} \tr\lp X_1 \cdots X_k \rp
    -\f 1 {n^L} \tr\lp \varphi(X_1) \cdots \varphi(X_k) \rp \asto 0,
\end{align*}
where $\varphi(W^l) = W^l$ and $\varphi(D^l) =$ an iid copy of $D^l$, independent from all other values of $\varphi$.
\end{cor}

This corollary is sufficient to justify the Stieljes transformation calculations of \cite{pennington_resurrecting_2017}, and show that the singular value distributions converge to their limits, almost surely.

More generally, even with weight tying and arbitrary architecture, we can compute the singular value distribution of the neural network Jacobian, by expressing the moment computations as tensor programs, just like the above, and crank the machinery of \cref{thm:generalTensorP}.
\cref{warmup:marchenkopasturlaw} can be thought of the most basic such case of linear regression.

\subsection{Signal Propagation}
\label{subsec:signalProp}

We begin by examining the simple RNN and the weight-tied autoencoder, before reviewing some mean field equations that appeared in prior literature, which can be justified rigorously.
Finally, we close by looking at the weight-tied residual network, which is perhaps the simpliest ``RNN'' where the weight-tying leads to a different behavior than not tying the weights (in contrast to the simple RNN vs MLP).

\paragraph{Simple RNN}
\newcommand{\RNN}{{\mathrm{RNN}}}
\newcommand{\MLP}{{\mathrm{MLP}}}

A simple RNN that takes in input only at time 1 and outputs only at time $L$ can be thought of as an MLP with parameters tied across layers:
\begin{align*}
    x^0(x) &\defeq x\\
    h^l(x) &\defeq \f 1 {\sqrt{n}} W x^{l-1}(x) + b\\
    x^l(x) &\defeq \phi(h^l(x))\\
    \RNN(x; \theta) &\defeq \f 1 {\sqrt{n}} v^\trsp x^{L}(x)
\end{align*}
for $W \in \R^{n \times n}$ and $x, v, b \in \R^n$.
We will sample $W_{ij} \sim \Gaus(0, \sigma_W^2/n), b_i \sim \Gaus(0, \sigma_b^2).$
The computation of $\{\RNN(x^i; \theta)\}_{i=1}^B$ over a batch of inputs can be expressed by a tensor program with no line of type \ref{linetype:trsp} (essentially the program in \cref{subsec:exampleFCFF} but with weights and biases tied).
By \cref{thm:notransposeLimit}, we can compute $K^\cdc(h^l(x^i), h^{l'}(x^{i'})) = 0$ whenever $l \ne l'$, and that $K^\cdc|_{\{h^l(x^i)\}_i} = \sigma_W^2 \Vt \phi(K^\cdc|_{\{h^{l-1}(x^i)\}_i}) + \sigma_b^2.$
This is, of course, exactly the same as the $K^\cdc$ computed if $W$ and $b$ are not tied across layers (see \cref{sec:Conseq}).
This therefore mathematically proves what the experiments of \citet{chen_dynamical_2018} suggested.
\begin{cor}
Suppose $\phi$ is $\alpha$-controlled for $\alpha <2$.
Assume almost sure rank convergence.
Then for any $\alpha$-controlled $\psi: \R^L \to \R$, as $n \to \infty$,
\begin{align*}
    \f 1 n \sum_{i=1}^n \psi(h^1_\RNN(x^j)_i, \ldots, h^L_\RNN(x^j)_i) - \f 1 n \sum_{i=1}^n \psi(h^1_\MLP(x^j)_i, \ldots, h^L_\MLP(x^j)_i)
    \asto 0,
\end{align*} 
where $h^l_\RNN$ ($h^l_\MLP$) denotes the hidden state of the RNN (MLP), with the weights and biases in each model identically sampled.
\end{cor}

\paragraph{Autoencoder}
A weight-tied autoencoder is described by the following equations (we follow \citet{li_random_2018} here)
\begin{align*}
    x^0(x)
        &\defeq
            x
            \\
    x^l(x)
        &\defeq
            W^l \sigma^{l-1}(x^{l-1}) + b^l, \forall l \in \{1,\ldots, L\}
            \\
    \hat x^L
        &\defeq
            W^L{}^\trsp \phi^L(x^L) + v^L
            \\
    \hat x^l(x)
        &\defeq
            W^l{}^\trsp \phi^l(\hat x^{l+1}) + v^l
\end{align*}
for input $x \in \R^{n^0}$, a set of weights $W^l \in \R^{n^l \times n^{l-1}}$, encoder biases $b^l \in \R^{n^l}$, decoder biases $v^l \in \R^{n^{l-1}}$, for $l =1, \ldots, L.$
We also have the decoder and encoder activation functions $\phi^l: \R \to \R, \sigma^l: \R \to \R$.
The parameters are sampled iid according to
\begin{align*}
    W^l_{ij} \sim \Gaus(0, \sigma_{W^l}^2/n^{l-1}),&&
    b^l_i \sim \Gaus(0, \sigma^2_{b^l}),&&
    v^l_i \sim \Gaus(0, \sigma^2_{v^l}).
\end{align*}
We consider taking the limit where $n^l \to \infty, \forall l$, with $n^l/n^{l-1} \to \alpha^l \in (0, \infty)$.

\citet{li_random_2018} proved a (forward) signal propagation theorem of the above weight-tied autoencoder that uses the following quantities.
Define $\{\tau_l\}_{l=1}^L$ and $\{\bar \tau_l\}_{l = 0}^L$ inductively:
\begin{align*}
    \bar \tau_0^2 &= \f 1 {n^0} \|\sigma^0(x)\|^2,
    &
    \bar \tau_l^2 &= \tau_l^2 + \sigma_{b^l}^2,
    &
    \forall l \in \{1, \ldots, L\},
    \\
    \tau_1^2 &= \sigma_{W^1}^2 \bar \tau_0^2,&
    \tau_l^2 &= \sigma_{W^l}^2 \EV_z \sigma^{l-1}(\bar \tau_{l-1} z)^2,&
    \forall l \in \{2, \ldots, L\}
\end{align*}
where $z \sim \Gaus(0, 1)$.
Next define $\{\gamma_l, \rho_l\}_{l=2}^{L+1}$ inductively:
\begin{align*}
    \gamma_{L+1} &= \f 1 {\bar \tau_L^2} \EV_{z_1} \bar \tau_L z_1 \phi^L (\bar \tau_L z_1),\qquad
    \rho_{L+1} = \EV_{z_1} \phi^L(\bar \tau_L z_1)^2,\\
    \gamma_l
        &=
            \f 1 {\bar \tau^2_{l-1}}
            \EV_{z_1, z_2}
                \bar \tau_{l-1} z_1 \phi^{l-1}
                    \lp
                    \alpha^l \sigma^2_{W^l} \gamma_{l+1}
                    \sigma^{l-1}(\bar \tau_{l-1} z_1)
                    +
                    \sqrt{\alpha^l \sigma^2_{W^l} \rho_{l+1}
                        + \sigma^2_{v^l}}
                    z_2
                    \rp,\\
    \rho_l
        &=
            \EV_{z_1, z_2}
                \phi^{l-1}
                \lp
                    \alpha^l \sigma^2_{W^l} \gamma_{l+1}
                    \sigma^{l-1}(\bar \tau_{l-1} z_1)
                    +
                    \sqrt{\alpha^l \sigma^2_{W^l} \rho_{l+1}
                        + \sigma^2_{v^l}}
                    z_2
                \rp^2,\qquad
    \forall l \in \{L-2, \ldots, 2\}
\end{align*}
where $z_1, z_2 \sim \Gaus(0, 1)$.

By expressing the autoencoder computation on a single input $x$ as a tensor program and applying \cref{thm:generalTensorP}, we obtain a version of the main theorem of \citet{li_random_2018} that assumes no smoothness of the nonlinearities and of test functions.
If $X^t, Y^t \in \R^{n(t)}$ are two sequences of random vectors in $t$, then write $X^t \cong Y^t$ to mean that for any polynomially bounded $\psi: \R \to \R$, $\f 1 {n(t)} \sum_{i=1}^{n(t)} \psi(X^t_i)$ and $\f 1 {n(t)} \sum_{i=1}^{n(t)} \psi(Y^t_i)$ converge a.s. to the same limit, as $n(t) \to \infty$.
\begin{cor}
Let the activation functions $\{\sigma^l, \phi^l\}_l$ be polynomially bounded.
Then in the limit $\{n^l\}_l \to \infty$ as described above,
\begin{enumerate}
    \item $x^l \cong \Gaus(0, \bar \tau_l I_{n^l}), \forall l \in \{1, \ldots, L\}$.
    \item
        $\hat x^l \cong
            \alpha^l \sigma^2_{W^l} \gamma_{l+1} \sigma^{l-1}(\bar \tau_{l-1} \vec z_1)
            +
            \sqrt{\alpha^l \sigma^2_{W^l} \rho_{l+1}
                + \sigma_{v^l}^2}
            z_2, \forall l \in \{2, \ldots, L\}$
        where $\vec z_1, \vec z_2 \sim \Gaus(0, I_{n^{l-1}})$ independently.
    \item
        the autoencoder output $\hat x$ satisfies
        $$\hat x \cong \phi^0(\alpha^1 \sigma_{W^1}^2 \gamma_2 \sigma^0(x)
        + \sqrt{\alpha^1 \sigma^2_{W^1} \rho_2 + \sigma_{v^1}^2}
            \vec z_2),$$
            where $\vec z_2 \sim \Gaus(0, I_{n^0})$ independent of $x$.
\end{enumerate}
\end{cor}
\citet{li_random_2018}'s main theorem is almost the same as this, except that
\begin{enumerate}
    \item
        \citet{li_random_2018} requires $\sigma^l$ to be nontrivial in the sense that for any $\tau > 0$, $\EV_{z \sim \Gaus(0, 1)} \sigma^l(\tau z)^2 > 0$.
        But this is equivalent to saying that $\sigma^l$ is not a.e. 0.
        Indeed, if $\sigma^l(x) = \sum_i a_i h_i(x) $ is its Hermite expansion in orthonormal Hermite basis $h_i$, then
        $\EV_{z \sim \Gaus(0, 1)} \sigma^l(\tau z)^2 = \sum_i a_i^2 \tau^{2i}$, which can be 0 for positive $\tau$ iff all $a_i$s vanish.
    \item 
        All nonlinearities $\sigma^l, \phi^l$ are required by \citet{li_random_2018} to be globally Lipschitz (and hence linearly bounded).
        Here we only need them to be polynomially bounded.
    \item
        The equivalence relation $\cong$ is defined differently in \citet{li_random_2018}.

        There, $X^t \cong Y^t$ if $\phi_t(X^t) - \EV \phi_t(Y^t) \probto 0$ for any sequence of uniformly pseudo-Lipschitz functions.
        A sequence of functions $\phi_t: \R^{n(t)} \to \R$ is said to be \emph{uniformly pseudo-Lipschitz} if there exists a constant $C$, independent of $n$, such that for any $x, y \in \R^{n(t)}$,
                $$|\phi_n(x) - \phi_n(y)| \le C 
                    \lp 1 + \f{\|x\|}{\sqrt n} + \f{\|y\|}{\sqrt n} \rp
                    \f{\|x - y\|}{\sqrt n}.$$
                   
        In contrast, the test functions $\phi_t$ we allow are coordinatewise functions --- a stronger assumption than the above --- but does not need to be smooth, just polynomially bounded --- a weaker assumption than the above.
        We also guarantee almost sure convergence, a stronger result than their convergence in probability.
        It would be interesting in future work to study whether one can remove the smoothness assumption even for noncoordinatewise test functions.
\end{enumerate}

\subsubsection{Justifying semirigorous equations}
\label{subsubsec:justifySemirigor}

Below, we give several examples of signal propagation equations derived heuristically in prior works, which can now be justified rigorously using the tensor program framework.
\newcommand{\Xx}{{\mathcal X}}

\paragraph{MLP \cite{schoenholz_deep_2017}}
See \cref{subsec:warmupConseq}.

\paragraph{Residual Network \cite{yang_mean_2017}}
We define a residual network $f(x; \theta), x \in \R^{n^0}$ as follows
\begin{align*}
    x^0(x) &\defeq x\\
    h^l(x) &\defeq \f 1 {\sqrt{n^{l-1}}} W^l x^{l-1}(x) + b^l\\
    x^l(x) &\defeq \f 1 {\sqrt{n^l}} V^l \phi(h^l(x)) + x^{l-1}(x) + a^l\\
    f(x; \theta) &\defeq \f 1 {\sqrt{n^{L}}} w^\trsp x^{L}(x)
\end{align*}
with $a^l, b^l \in \R^{n^l}$, $w \in \R^{n^L}$, $V^l \in \R^{n^{l} \times n^{l}}$, and $W^l \in \R^{n^l \times n^{l-1}}$ for $l = 1, \ldots, L$.
These form the parameters $\theta$ of $f$.
We sample $W^l_{ij} \sim \Gaus(0, (\sigma_w^l)^2), V^l_{ij} \sim \Gaus(0, (\sigma_V^l)^2), w_i \sim \Gaus(0, (\sigma_w^{L+1})^2)$,
$a_i^l \sim \Gaus(0, (\sigma_a^l)^2)$, and $b^l_i \sim \Gaus(0, (\sigma_b^l)^2)$.
Define kernels $\Sigma^{L+1}: (\R^{n^0})^2 \to \R$ by
\begin{align*}
    \tilde \Sigma^0(x, x')
        &\defeq
            \f 1 {n^0}\sum_{i=1}^{n^0} x_i x'_i
            \\
    \Sigma^l
        &\defeq
            (\sigma_w^l)^2 \tilde\Sigma^{l-1} + (\sigma_b^l)^2
            \\
    \tilde \Sigma^l
        &\defeq
            \tilde \Sigma^{l-1} + (\sigma_V^l)^2\Vt\phi(\Sigma^l) + \sigma_a^2
            \\
    \Sigma^{L+1}
        &\defeq
            (\sigma_w^{L+1})^2
            \Sigma^{L}
            .
\end{align*}
Then for any finite subset $\Xx \sbe \R^{n^0},$ for $\alpha$-controlled $\phi$,
\begin{align*}
    f(\Xx; \theta) \distto \Gaus(0, \Sigma^{L+1} |_\Xx).
\end{align*}

\paragraph{Convolutional Network \cite{xiao_dynamical_2018}}

Consider a convolutional network
\begin{align*}
    x^0_{\alpha i}(x) &\defeq x_{\alpha i}\\
    h^l_{\alpha i}(x)
        &\defeq
            \f 1 {\sqrt{n^l}} \sum_{\substack{j \in [n^{l-1}]\\\beta \in [s^{l-1}]}} W^l_{\beta ij} x^{l-1}_{\alpha + \beta,j}(x) + b^l_i\\
    x^l_{\alpha i}(x) &\defeq \phi(h^l_{\alpha i}(x))
\end{align*}
where $h^l_{\alpha i}$ denotes the preactivation at the $l$th layer, the $i$th channel, each with $s^l$ neurons, and the $\alpha$th neuron, and likewise for $x^l_{\alpha i}$.
$n^l$ is the number of channels in layer $l$.

Suppose we have a non-negative vector $(v^l_\beta)_{\beta \in ker}$ that sums up to 1 and $W^l_{\beta i j} \sim \Gaus(0, (\sigma_W^l)^2 v_\beta^l), b^l_{\beta i} \sim \Gaus(0, (\sigma_b^l)^2) $.
By \cref{thm:notransposeLimit}, $\{h^l_{\alpha\bullet}(x_a)\}_{\alpha, a}$ are ``jointly Gaussian'' in the limit.
Define $\Sigma^l_{\alpha a, \beta b} \defeq \lim^{\mathrm{a.s.}} \f 1 {n^l} \sum_{i=1}^{n^l} h^l_{\alpha i}(x_a) h^l_{\beta i}(x_b)$, for any $i$.
Then with $\star$ denoting 2D circular cross correlation, \citet{xiao_dynamical_2018} calculated, semirigorously,
\begin{align*}
    \Sigma^{l+1}
        &=
            (\sigma_W^l)^2\Diag(v^l) \star \Vt \phi(\Sigma^{l}) + (\sigma_b^l)^2\\
    \Sigma^{l+1}_{\alpha a, \beta b}
        &=
            (\sigma_W^l)^2 \sum_{\gamma \in [s^{l}]} v^l_\gamma \Vt \phi(\Sigma^{l})_{\alpha +\gamma;a,\beta+\gamma;b} + (\sigma_b^l)^2.
\end{align*}
These equations can now be recovered rigorously using \cref{thm:notransposeLimit}.

Now suppose the last layer (layer $L$) is linear with output,
\begin{align*}
    f(x; \theta)
        &\defeq
            \f 1 {\sqrt{n^{L-1} s^{L-1}}}
            \sum_{\substack{\alpha \in [s^{L-1}]\\i \in [n^{L-1}]}} 
            W^L_{\alpha i} x^{L-1}_{\alpha i}
            \in \R
\end{align*}
and the weights are sampled according to $W^L_{\alpha i} \in \Gaus(0, 1)$.
Then, we can compute via \cref{thm:notransposeLimit},
\begin{align*}
(f(x_a; \theta), f(x_b; \theta)) \distto \Gaus\lp 0, \f 1 {s^{L-1}}\begin{pmatrix} \tr \Vt\phi(\Sigma^{L-1})_{\bullet a, \bullet a} & \tr \Vt\phi(\Sigma^{L-1})_{\bullet a, \bullet b}\\ \tr \Vt\phi(\Sigma^{L-1})_{\bullet a, \bullet b} & \tr \Vt\phi(\Sigma^{L-1})_{\bullet b, \bullet b} \end{pmatrix}\rp.
\end{align*}

Define, via \cref{thm:gradIndep}, $\Pi^l_{\alpha a, \beta b} \defeq \limas \sum_{i=1}^{n^l} \pdf{f}{x^l_{\alpha i}}(x_a) \pdf f {x^l_{\beta i}}(x_b)$, for any $i$.
Then \citet{xiao_dynamical_2018} essentially calculated, semirigorously,
\begin{align*}
    \Pi^{L-1}_{\alpha a, \beta b}
        &=
            \f 1 {s^{L-1}}\ind(\alpha = \beta).
\end{align*}
and in all previous layers, the recurrence
\begin{align*}
    \Pi^{l-1}
        &=
            (\sigma_w^l)^2 \Diag(v^l{}^\#) \star (\Vt \phi'(\Sigma^l) \odot \Pi^{l})
\end{align*}
where $v^l{}^\#$ is the reverse of $v^l$.
These equations can now be justified rigorously using \cref{thm:gradIndep}.

\paragraph{Batchnorm \cite{yang_mean_2018}}
Given $\phi: \R \to \R$, let $\batchnorm_\phi: \R^B \to \R^B, z \mapsto \phi\lp \f{z - \bar z}{\|z - \bar z\|/\sqrt B}\rp$.
This is an application of batchnorm followed by coordinatewise action by $\phi$, where $z$ should be thought of as a fixed unit across a batch of size $B.$

If $\vec x = (x_i, \ldots, x_B)$ is a batch of inputs $x_i \in \R^{n_0}$, then define a deep batchnorm network $f(\vec x; \theta): \R^{B \times n_0} \to \R^{B \times 1}$ by
\begin{align*}
    \vec x^0(\vec x) &\defeq \vec x\\
    \vec h^l(\vec x) &\defeq  (\f 1 {\sqrt{n^{l-1}}} W^l x^{l-1}(\vec x)_i + b^l)_{i=1}^B\\
    \vec x^l(\vec x) &\defeq \batchnorm_\phi(\vec h^l(\vec x))\\
    f(\vec x; \theta) &\defeq \lp \f 1 {\sqrt{n^{L}}} w^\trsp x^{L}(\vec x)_i\rp_{i=1}^B.
\end{align*}
Here $B$ and $n_0$ will be fixed and $n^l \to \infty$ for $l > 0.$
We sample $W^l_{ij} \sim \Gaus(0, (\sigma_w^l)^2), w_i \sim \Gaus(0, (\sigma_w^{L+1})^2)$ and $b^l_i \sim \Gaus(0, (\sigma_b^l)^2)$.
Define multivariate kernels $\Sigma^l: (\R^{B \times n_0})^2 \to \R^{B \times B}$ by
\begin{align*}
    \Sigma^1(\vec x, \vec x')_{ij}
        &\defeq
            (\sigma_w^1)^2\f 1 {n^0} x_i{}^\trsp x'_i + (\sigma_b^1)^2
            \\
    \Sigma^l|_{\{\vec x, \vec x'\}}
        &\defeq
            (\sigma_w^l)^2 \Vt\batchnorm_\phi(\Sigma^{l-1}|_{\{\vec x, \vec x'\}}) + (\sigma_b^l)^2
            \\
    \Sigma^{L+1}|_{\{\vec x, \vec x'\}}
        &\defeq
            (\sigma_w^{L+1})^2 \Vt\batchnorm_\phi(\Sigma^{L}|_{\{\vec x, \vec x'\}}).
            \\
\end{align*}
Then for any finite set of batches $\Xx \sbe \R^{B \times n^0},$ for $\alpha$-controlled $\phi$,
\begin{align*}
    f(\Xx; \theta) \distto \Gaus(0, \Sigma^{L+1} |_\Xx).
\end{align*}

\citet{yang_mean_2018} also calculated the gradient dynamics of such a deep batchnorm network, but our theorems cannot rigorously justify them due to the singularity of the Jacobian of batchnorm.

\subsubsection{A taste of weight-tying}

\paragraph{Weight-tied Residual Network}
\renewcommand{\p}[2]{#1^{#2}}
The simpliest ``recurrent neural network'' for understanding when weight-tying can have a different behavior than not is perhaps in a residual network with weights tied across layers.

In this section, fix a matrix $W \in \R^{N \times N}$ and a function $\phi: \R \to \R$.
Consider the dynamics
\begin{align*}
    \p h t = W \phi(\p h {t-1}) + \p h {t-1}, \p h t \in \R^N.
\end{align*}
What is the ``average behavior'' of this dynamics as $N \to \infty$, if we were to sample $W_{ij} \sim \Gaus(0, \sigma_w^2/N)$?
\cref{thm:notransposeLimit} applies here when $\phi$ is $\alpha$-controlled, and it tells us that ``$(\p h t _i)_i$ are i.i.d. samples of a zero-mean Gaussian distribution, in the limit $N \to \infty$,'' as far as $\alpha$-controlled test functions are concerned.

\newcommand{\KK}{\mathsf{K}}
\newcommand{\CC}{\mathsf{C}}
By \cref{thm:notransposeLimit}, we can make the following
\begin{defn}
Define $\KK(l, m) \defeq \limas  \f 1 N \sum_{i=1}^N \p h l _i \p h m _i$ and $\CC(l, m) \defeq \limas \f 1 N \sum_{i=1}^N \p h l _i \sum_{j} W_{ij} \phi(\p h m _j)$.
\end{defn}

\begin{thm}
$\KK$ and $\CC$ satisfy the following equations in the limit $N \to \infty$.
\begin{align*}
    \KK(l, m)
        &=
            \CC(l, m-1) + \KK(l, m-1)
            \numberthis \label{eqn:KC1}\\
        &=
            \CC(m, l-1) + \KK(m, l-1)
            \numberthis \label{eqn:KC2}\\
        &=
            \sigma_w^2 \Vt \phi(K^{l-1, m-1})_{12}
            + \KK(l-1, m-1)
            + \CC(l-1, m-1) + \CC(m-1, l-1)
            \numberthis \label{eqn:KC3}\\
    \CC(l, m)
        &=
            \CC(l-1, m) + \sigma_w^2 \Vt \phi(K^{l-1, m})_{12}
            \numberthis \label{eqn:CCrec}
\end{align*}
where $K^{a, b}$ is the matrix $\begin{pmatrix} \KK(a, a) & \KK(a, b)\\ \KK(a, b) & \KK(b, b) \end{pmatrix}$.

In addition, for all $m, l \ge 0$, $\KK(l, m) = \KK(m, l), \KK(0, m) = \KK(m, 0) = \KK(0, 0), \CC(0, m) = 0$.
\end{thm}

\begin{proof}
The identities at the end are obvious.
We will focus on proving \cref{eqn:KC1,eqn:KC2,eqn:KC3,eqn:CCrec}.
We have
\begin{align*}
    \KK(l, m)
        &=
            \limas  \f 1 N \sum_{i=1}^N \p h l _i \p h m _i \\
        &=
            \limas  \f 1 N \sum_{i=1}^N \p h l _i \lp\sum_j W_{ij} \phi(\p h {m-1} _j) + \p h {m-1} _i\rp\\
        &=
            \limas  \f 1 N \sum_{i=1}^N \p h l _i \sum_j W_{ij} \phi(\p h {m-1} _j) 
            + \limas  \f 1 N \sum_{i=1}^N \p h l _i \p h {m-1} _i\\
        &=
            \CC(l, m-1) + \KK(l, m-1)
\end{align*}
which gives \cref{eqn:KC1} and also \cref{eqn:KC2} by symmetry.

Now
\begin{align*}
    \CC(l, m)
        &=
            \limas  \f 1 N \sum_{i=1}^N \p h l _i \sum_{j} W_{ij} \phi(\p h m _j)\\
        &=
            \limas  \f 1 N \sum_{i=1}^N \sum_{j, k} W_{ik}\phi(\p h {l-1} _k) W_{ij} \phi(\p h m _j)
            + \limas  \f 1 N \sum_{i=1}^N \p h {l-1} _i \sum_{j} W_{ij} \phi(\p h m _j)\\
        &=
            \limas  \f 1 N \sum_{i=1}^N \sum_{j, k} W_{ik}\phi(\p h {l-1}_k) W_{ij} \phi(\p h m _j)
            + \CC(l-1, m)
            \\
        &=
            \limas  \f 1 N \sum_{i=1}^N 
                \sum_j W_{ij}^2 \phi(\p h {l-1}_j)\phi(\p h {m}_j)
            + \CC(l-1, m)
            \\
        &=
            \sigma_w^2 \Vt \phi(K^{l-1, m})_{12}
            + \CC(l-1, m)
\end{align*}
yielding \cref{eqn:CCrec}.

Finally, \cref{eqn:KC3} is given by expanding $\CC(l, m-1)$ by \cref{eqn:KC1} and expanding $\KK(l, m-1)$ by \cref{eqn:KC2}.
\end{proof}

One can see immediately that the growth of $h^l$ norm is much faster here than for untied-weights residual network.

We now study the simultaneous evolution of two vectors $\p h l$ and $\p \hbar l$.

\begin{defn}
Define $\KK_{h\hbar}(l, m) \defeq
    \limas  \f 1 N \sum_{i=1}^N \p h l _i \p \hbar m _i \defeq \KK_{\hbar h}(m, l)$
    and $\CC_{h\hbar}(l, m) \defeq 
    \limas  \f 1 N \sum_{i=1}^N \p h l _i \sum_{j} W_{ij} \phi(\p \hbar m _j),
    \CC_{\hbar h}(l, m) \defeq
    \limas  \f 1 N \sum_{i=1}^N \p \hbar l _i \sum_{j} W_{ij} \phi(\p h m _j)$.
\end{defn}

\begin{thm}
\begin{align*}
    \KK_{h\hbar}(l, m)
        &=
            \CC_{h\hbar}(l, m-1) + \KK_{h\hbar}(l, m-1)
            \numberthis \label{eqn:KC1hh}\\
        &=
            \CC_{\hbar h}(m, l-1) + \KK_{\hbar h}(m, l-1)
            \numberthis \label{eqn:KC2hh}\\
        &=
            \sigma_w^2 \Vt \phi(K_{h\hbar}^{l-1, m-1})_{12}
            + \KK_{h\hbar}(l-1, m-1)
            + \CC_{h\hbar}(l-1, m-1) + \CC_{\hbar h}(m-1, l-1)
            \numberthis \label{eqn:KC3hh}\\
    \CC_{h\hbar}(l, m)
        &=
            \CC_{h\hbar}(l-1, m) + \sigma_w^2 \Vt \phi(K_{h\hbar}^{l-1, m})_{12}
            \numberthis \label{eqn:CCrechh}
\end{align*}
where $K_{h\hbar}^{l-1, m}$ is the matrix
$\begin{pmatrix}
\KK_{hh}(l-1, l-1) & \KK_{h\hbar}(l-1, m)\\
\KK_{h\hbar}(l-1, m) & \KK_{\hbar\hbar}(m, m)
\end{pmatrix}.$
\end{thm}
\begin{proof}
\begin{align*}
    \KK_{h\hbar}(l, m)
        &=
            \limas  \f 1 N \sum_{i=1}^N \p h l _i \p \hbar m _i \\
        &=
            \limas  \f 1 N \sum_{i=1}^N \p h l _i \lp\sum_j W_{ij} \phi(\p \hbar {m-1} _j) + \p \hbar {m-1} _i\rp\\
        &=
            \limas  \f 1 N \sum_{i=1}^N \p h l _i \sum_j W_{ij} \phi(\p \hbar {m-1} _j)
            + \limas  \f 1 N \sum_{i=1}^N \p h l _i \p \hbar {m-1} _i\\
        &=
            \CC_{h\hbar}(l, m-1) + \KK_{h\hbar}(l, m-1)
\end{align*}

\begin{align*}
    \CC_{h\hbar}(l, m)
        &=
            \limas  \f 1 N \sum_{i=1}^N \p h l _i \sum_{j} W_{ij} \phi(\p \hbar m _j)\\
        &=
            \limas  \f 1 N \sum_{i=1}^N \sum_{j, k} W_{ik}\phi(\p h {l-1} _k) W_{ij} \phi(\p \hbar m _j)
            + \limas  \f 1 N \sum_{i=1}^N \p h {l-1} _i \sum_{j} W_{ij} \phi(\p \hbar m _j)\\
        &=
            \sigma_w^2 \Vt \phi(K_{h\hbar}^{l-1, m})_{12}
            + \CC_{h\hbar}(l-1, m)
\end{align*}
\end{proof}

\newcommand{\hgrad}{\mathsf{h}}
\newcommand{\hgradt}{\mathsf{g}}
Now for the backward pass.
Define $\p \hgrad t := \nabla_{\p h t} E, \p \hgradt t := W^T \p \hgrad t$ for a loss function $E$.
Then
\begin{align*}
    \p \hgrad {t-1} &= \p \hgrad {t} + (W^T \p \hgrad t) \circ \phi'(\p h {t-1})
\end{align*}
So
\begin{align*}
    \limas  \f 1 N \sum_{i=1}^N \p \hgrad s _i \p \hgrad t _i
        &=
            \limas  \f 1 N \sum_{i=1}^N \p \hgrad {s+1} _i \p \hgrad t _i + \p \hgradt {s+1} _i \phi'(\p h s _i) \p \hgrad t _i \\
    \limas  \f 1 N \sum_{i=1}^N \p \hgradt t _i \p \hgradt s _i        
        &=
            \limas  \f 1 N \sum_{i=1}^N \sum_{j ,j'} W_{ji} \p \hgrad t _j W_{j'i} \p \hgrad t _{j'}\\
        &=
            \sigma_w^2 \limas  \f 1 N \sum_{i=1}^N \p \hgrad t _k \p \hgrad t _k\\
    \limas  \f 1 N \sum_{i=1}^N \p \hgrad t _i \p \hgradt s _i \phi'(\p h r _i)
        &=
            \limas  \f 1 N \sum_{i=1}^N \p \hgrad {t+1} _i \p \hgradt s _i \phi'(\p h r _i)
                + \p \hgradt {t+1} _i \p \hgradt s _i \phi'(\p h {t} _i)\phi'(\p h r _i)\\
        &=
            \limas  \f 1 N \sum_{i=1}^N
                \p \hgrad {t+1} _i \p \hgradt s _i \phi'(\p h r _i)
                + \sigma_w^2 
                    \lp \limas  \f 1 N \sum_{i=1}^N
                        \p \hgrad {t+1} _i \p \hgrad s _i \rp
                    \lp \limas  \f 1 N \sum_{i=1}^N
                        \phi'(\p h {t} _i) \phi'(\p h r _i)\rp
\end{align*}

Suppose we backprop a zero mean Gaussian vector with normalized norm 1.
For a weight-tied residual network that runs $S$ steps, we have boundary conditions
\begin{align*}
    \limas  \f 1 N \sum_{i=1}^N \p \hgrad S _i \p \hgrad S _i &= 1\\
    \limas  \f 1 N \sum_{i=1}^N \p \hgrad {S+1} _i \p \hgrad t _i &= 0, \forall t\\
    \limas  \f 1 N \sum_{i=1}^N \p \hgrad {S} _i \p \hgradt s _i \phi'(\p h r _i) &= 0, \forall s, r
\end{align*}

These equations then yield the dynamics of gradients in a weight-tied residual network.

\renewcommand{\hbar}{{\bar h}}
\newcommand{\xbar}{{\bar x}}
\newcommand{\ubar}{{\bar u}}

\subsection{Neural Tangent Kernel}
\label{subsec:NTKAppendix}
\newcommand{\NTK}{\mathrm{NTK}}
Again, let $F(z; \theta)$ be the body of a neural network as above, and suppose it's represented by a tensor program $\pi$.
For every input A-var $A$ of $\pi$, $\pdf F {A} = \sum_{g := A h} \pdf F {g} \otimes h + \sum_{g:=A^\trsp h} h \otimes \pdf F{g}$, where the sums are over vars satisfying the subscripts in the sums.
If the network has scalar output is given by $f(x) = v^\trsp F(E(x); \theta)$ where $E$ is an embedding function, then the contribution of $A$ to the NTK of $f$ is
\begin{align*}
    \NTK_A(x, y)
        &\defeq
            \sum_{g := A h} 
                \sum_i (v^\trsp \pdf F {g}(E(x)))_i (v^\trsp \pdf F {g}(E(y)))_i
                \sum_j h_j(E(x)) h_j(E(y))\\
        &\phantomeq
            +
            \sum_{g:=A^\trsp h}
                \sum_i h_i(E(x)) h_i(E(y))
                \sum_j (v^\trsp \pdf F{g}(E(x)))_j
                    (v^\trsp \pdf F{g}(E(y)))_j.
        \numberthis
        \label{eq:NTKA}
\end{align*}
Each of the four subsums in \cref{eq:NTKA} can be computed via \cref{thm:generalTensorP} (or \cref{thm:gradIndep} if $\pi$ doesn't use \ref{linetype:trsp} lines) when we expand the computation of $f$ over $x$ and $y$ as well as its gradients into a single tensor program.
Note that \cref{eq:NTKA} scales like $(\nvar^l)^2$; dividing by this factor roughly corresponds to using the parametrization of \citet{jacot_neural_2018}.
The contribution to NTK from input G-vars is similar and even simpler to compute.

The above computation would hold as long as we can apply \cref{thm:generalTensorP}, which requires that we have almost sure rank convergence of the relevant programs and that the nonlinearities of $f$ have polynomially bounded weak derivatives.

We give an example by computing the NTK of a CNN (which has not appeared in prior literature).

\paragraph{CNN}
Assume the notation of the CNN section of \cref{subsubsec:justifySemirigor}.

The contribution to the NTK of weights $W^l_{\beta ij}$ for $l < L$ is
\begin{align*}
    &\phantomeq
         \limas \sum_{\beta \in ker}\sum_{i,j=1}^{n^l, n^{l-1}} \pdf f {W^l_{\beta ij}}(x_a)\pdf f {W^l_{\beta ij}}(x_b)\\
    &=
        \limas \f 1 {n^{l-1}} \sum_{\beta \in ker}\sum_{i,j=1}^{n^l, n^{l-1}} 
            (\sum_{\alpha}\pdf f {h^l_{\alpha i}}(x_a) x^{l-1}_{\alpha+\beta,j}(x_a))
            (\sum_{\alpha'}\pdf f {h^l_{\alpha' i}}(x_b) x^{l-1}_{\alpha'+\beta,j}(x_b))\\
    &=
        \f 1 {n^l n^{l-1}}\sum_{\beta \in ker} \sum_{i,j=1}^{n^l, n^{l-1}} 
        \sum_{\alpha,\alpha'=1}^{2k+1}
        \Pi^l_{\alpha a, \alpha' b} \Vt\phi(\Sigma^l)_{\alpha a, \alpha' b} \Vt\phi'(\Sigma^{l-1})_{\alpha+\beta; a, \alpha'+\beta; b}\\
    &=
        \sum_{\alpha,\alpha'=1}^{2k+1}
        \Pi^l_{\alpha a, \alpha' b} \Vt\phi(\Sigma^l)_{\alpha a, \alpha' b}
        \sum_{\beta \in ker}\Vt\phi'(\Sigma^{l-1})_{\alpha+\beta; a, \alpha'+\beta; b}\\
    &=
        \la \Pi^l_{\bullet a,\bullet b} \odot \Vt \phi(\Sigma^l_{\bullet a, \bullet b}), I \star \Vt\phi'(\Sigma^{l-1}_{\bullet a, \bullet b})\ra
\end{align*}
Note that if we sample $W^l_{\beta ij} \sim \Gaus(0, (\sigma_w^l)^2)$ and replace $W^l_{\beta ij}$ with $\sqrt{v^l_\beta} W^l_{\beta ij}$, then in the above expression we replace $I$ with $\Diag(v^l)$.

Similarly, the contribution of $b^l_i$ for $l < L$ is
\begin{align*}
    &\phantomeq
        \limas  \sum_{i=1}^{n^l} \pdf f {b^l_i}(x_a)\pdf f {b^l_i}(x_b)\\
    &=
        \limas  \sum_{i=1}^{n^l} (\sum_\alpha \pdf f {h^l_{\alpha i}}(x_a))(\sum_{\alpha'} \pdf f {h^l_{\alpha' i}}(x_b))\\
    &=
        \sum_{\alpha,\alpha'} \Pi^l_{\alpha a, \alpha' b} \Vt\phi'(\Sigma^l)_{\alpha a, \alpha' b}\\
    &=
        \la \onem, \Pi^l_{\bullet a, \bullet b} \odot \Vt\phi'(\Sigma^l)_{\bullet a, \bullet b}\ra\\
    &=
        \la \Pi^l_{\bullet a, \bullet b}, \Vt\phi'(\Sigma^l)_{\bullet a, \bullet b}\ra
\end{align*}

The last layer weights (in the linear layer setting) contribute
\begin{align*}
    \limas  \f 1 {n^{L-1} s^{L-1}} \sum_{\alpha, i} x^{L-1}_{\alpha i}(x_a)x^{L-1}_{\alpha i}(x_b)
        &=
            \f 1 {s^{L-1}} \tr \Vt\phi(\Sigma^{L-1})_{\bullet a, \bullet b}.
\end{align*}

Therefore
\begin{cor}
The NTK of the CNN defined above converges almost surely to
\begin{align*}
    \mathrm{NTK}(x_a, x_b) \asto
    \f 1 {s^{L-1}} \tr \Vt\phi(\Sigma^{L-1})_{\bullet a, \bullet b} 
    +
    \sum_{l < L} 
        \la \Pi^l_{\bullet a,\bullet b} \odot \Vt \phi(\Sigma^l_{\bullet a, \bullet b}), I \star \Vt\phi'(\Sigma^{l-1}_{\bullet a, \bullet b})\ra
    +
    \la \Pi^l_{\bullet a, \bullet b}, \Vt\phi'(\Sigma^l)_{\bullet a, \bullet b}\ra
\end{align*}
as long as $\phi$ has a polynomially bounded weak derivative.
\end{cor}

\subsection{Approximate Message Passing}
\label{subsec:AMP}

We follow \citet{bayati_dynamics_2011,berthier_state_2017} for a brief introduction to Approximate Message Passing.

Given an $n \times N$ matrix $A$, the compressed sensing problem asks for a way to reconstruct a (sparse) vector $x_0 \in \R^N$ from a (small) vector of linear observations $y = A x_0 + w \in \R^{n}$.
Here $w$ is a noise vector and $A$ is assumed to be known.
The Approximate Message Passing algorithm \citep{donoho_message_2009} starts with an initial guess $x^0 = 0$ and proceed by
\begin{align*}
    x^{t+1}
        &=
            \eta_t(A^\trsp z^t + x^t),
            \\
    z^t
        &=
            y - A x^t + \alpha^t z^{t-1}
\end{align*}
for an appropriate sequence of nonlinearities $\{\eta_t: \R \to \R\}_{t\ge0}$ and $\alpha^t = \f 1 n \sum_{i=1}^N \eta'_{t-1}((A^\trsp z^{t-1} + x^{t-1})_i) \in \R$.
The algorithm succeeds if $x^t$ converges to a good approximation of $x_0$.
Similar algorithms have been applied to robust regression \cite{donoho_high_2016}, Bayesian estimation \cite{kamilov_approximate_2012}, low rank matrix recovery \cite{kabashima_phase_2016}, phase retrieval \cite{schniter_compressive_2015}, and community detection in graphs \cite{deshpande_asymptotic_2017}.

The behavior of the AMP algorithm is accurately described by a formalism called ``stated evolution'' (SE), as $n, N \to \infty$ with constant ratio $n/N \to \delta \in (0, \infty)$, that bears some resemblance to the evolution of kernels in the GP correspondence of deep neural networks (see \cref{subsec:GPNN}) and to the gradient dynamical equations in the signal propagation analysis of DNNs (see \cref{subsec:signalprop}).
SE was introduced in \citet{donoho_message_2009} and later suitably formalized and rigorously proved for random Gaussian $A$ and suitably smooth $\eta_t$ in \citet{bayati_dynamics_2011}.
A more general version of the algorithm where $\eta_t: \R^N \to \R^N$ (instead of acting coordinatewise) was analyzed and a similar SE equations proved in \citet{berthier_state_2017}.

As a corollary to one of our main theorems \cref{thm:generalTensorP}, we show that, in the main theorem of \citet{bayati_dynamics_2011}, we can forgo smoothness assumptions on $\eta_t$ when each component of $x_0$ is sampled iid from a Gaussian.
We'll work with the following more general version of AMP from \citet{bayati_dynamics_2011}.
The algorithm is defined by two sequences of functions $\{f_t: \R^2 \to \R\}_{t \ge 0}, \{g_t: \R^2 \to \R\}_{t \ge 0}$.
Given $w \in \R^n, x_0 \in \R^N$, define the sequence of vectors $h^t, q^t \in \R^N$ and $z^t, m^t \in \R^n$, by fixing the initial condition $q^0$, and obtaining $\{b^t\}_{t \ge0}, \{m^t\}_{t \ge 0}, \{h^t\}_{t \ge 1}, $ and $\{q^t\}_{t \ge 1}$ through
\begin{align*}
    h^{t+1} &= A^\trsp m^t - \xi_t q^t,&
    m^t &= g_t(b^t, w),\\
    b^t &= A q^t - \lambda_t m^{t-1},&
    q^t &= f_t(h^t, x_0),
\end{align*}
where $\xi_t = \f 1 {N\sigma^2_t} \la b^t, g_t(b^t, w) \ra$ and $\lambda_t = \f 1 {n \tau^2_{t-1}} \la h^t, f_t(h^t, x_0) \ra$,
\footnote{Note that here we are using $\la, \ra$ to denote (unscaled) inner product, which is different from the usage of this notation in \citet{bayati_dynamics_2011}.}
and $\sigma_t$ and $\tau_t$ are defined via
\begin{align*}
    \tau_t^2 &\defeq \EV g_t(\sigma_t(Z, W))^2,&
    \sigma^2_t &\defeq \f N n \EV f_t(\tau_{t-1} Z, X_0)^2,&
    \text{where }
    Z \sim \Gaus(0, 1),
    W \sim \Gaus(0, \sigma_w^2),
    X_0 \sim \Gaus(0, \sigma_{x_0}^2)
\end{align*}
for sampling hyperparameters $\sigma_w^2, \sigma_{x_0}^2$.

By translating the above computation into a tensor program and applying \cref{thm:generalTensorP}, we obtain
\begin{cor}
Let $\{q_0(N)\}_{N \ge 0}$ and $\{A(N)\}_{N \ge 0}$ be resp. a sequence of initial conditions and a sequence of matrices $A \in \R^{n \times N}$ indexed by $N$ with iid entries $A_{ij} \sim \Gaus(0, 1/n)$.
Assume $n/N \to \delta \in (0, \infty)$.
Consider the sequence of vectors $\{x_0(N), w(N)\}_{N \ge 0}$ whose empirical distributions converge weakly to $\Gaus(0, \sigma_{x_0}^2)$ and $\Gaus(0, \sigma_{w}^2)$.
Suppose that the functions $f_t$ and $g_t$ are polynomially bounded for all $t$.
Then for any polynomially bounded function $\psi: \R^2 \to \R$ and all $t \ge 0$,
\begin{align*}
    \f 1 N \sum_{i=1}^N \psi(h_i^{t+1}, x_{0,i})
        &\asto
            \EV \psi(\tau_t Z, X),
            \\
    \f 1 n \sum_{i=1}^n \psi(b^t_i, w_i)
        &\asto
            \EV \psi(\sigma_t Z, W),
\end{align*}
as $N \to \infty$,
where $X_0 \sim \Gaus(0, \sigma_{x_0}^2)$ and $W \sim \Gaus(0, \sigma_{w}^2)$ independent of $Z \sim \Gaus(0, 1)$.
\end{cor}
This version differs from theorem 2 of \citet{bayati_dynamics_2011} in the following ways
\begin{enumerate}
    \item
        \citet{bayati_dynamics_2011} defined $\xi_t = \f 1 N \sum_{i=1}^N g'_t(b^t_i, w_i)$ and $\lambda_t = \f 1 n \sum_{i=1}^n f'_t(h^t_i, x_{0,i})$, where the derivatives are taken against the first argument.
        This is asymptotically equivalent to our formulation here by Stein's lemma \cref{lemma:stein}.
        Our formulation has the benefit of being defined for $g_t$ and $f_t$ without weak derivatives.
    \item 
        We are requiring that the $x_0$ and $w$ have empirical distributions that converge to Gaussians; with a bit more effort, we can also prove a result that allow them to converge to any distribution with all moments.
        This is a much stronger assumption than \citet{bayati_dynamics_2011}, who only assume that the limit distributions have some finite number of bounded moments.
    \item
        We don't have any smoothness assumptions on the nonlinearities $f_t$ and $g_t$, whereas \citet{bayati_dynamics_2011} requires them to be Lipschitz.
    \item
        We don't have any smoothness assumptions on the test function $\psi$, whereas \citet{bayati_dynamics_2011} requires them to be pseudo-Lipschitz of some order \footnote{A function $f: \R^s \to \R$ is pseudo-Lipschitz of order $k$ if there is a universal constant $C$ s.t. $|f(x) - f(y)| \le C (1 + |f(x)|^{k-1} + |f(y)|^{k-1}) \|x - y\|$.
        Note that this implies $f$ is bounded by a polynomial of degree $k$}.
\end{enumerate}

This concludes our discussion of various corollaries of our main theorems.
We now turn to their proofs.
First let us present the necessary lemmas.

\section{Lemmas}

\subsection{The Conditioning Trick}
\label{sec:conditioningTrick}
We first recall Moore-Penrose pseudoinverse and some properties of it.
\begin{defn}\label{defn:pseuodoinverse}
For $A \in \R^{n \times m}$, a pseudoinverse of $A$ is defined as a matrix $A^+ \in \R^{m \times n}$ that satisfies all of the following criteria
\begin{itemize}
	\item $A A^+ A = A$
	\item $A^+ A A^+ = A^+$
	\item $(AA^+)^\trsp = AA^+$
	\item $(A^+ A)^\trsp = A^+ A$
\end{itemize}
\end{defn}

The following facts are standard
\begin{itemize}
    \item if $A$ has real entries, then so does $A^+$.
    \item The pseudoinverse always exists and is unique.
    \item When $A$ is invertible, $A^+ = \inv A$.
    \item $(A^\trsp)^+ = (A^+)^\trsp$, which we denote as $A^{+\trsp}$.
    \item $A^+ = (A^\trsp A)^+ A^\trsp = A^\trsp (A A^\trsp)^+$.
    \item $AA^+$ is the orthogonal projector to the column space of $A$;
        $I - A^+ A$ is the orthogonal project to the null space of $A$.
    \item if $A$ has singular value decomposition $A = U\Lambda V$ where $U$ and $V$ are orthogonal and $\Lambda$ has the singular values on its diagonal, then $A^+ = V^\trsp \Lambda^+ U^\trsp$ where $\Lambda^+$ inverts all nonzero entries of $\Lambda$.
    \item For any collection of vectors $\{v_i\}_{i=1}^n$ in a Hilbert space, $w \mapsto \sum_{i,j=1}^n v_i (\Sigma^+)_{ij} \la v_j, w \ra $, where $\Sigma_{ij} = \la v_i, v_j \ra$, is the projection operator to the linear span of $\{v_i\}_{i=1}^n$.
\end{itemize}

We present a slightly more general versions of lemmas from \citet{bayati_dynamics_2011} that deal with singular matrices.

\begin{lemma}\label{lemma:condTrickVec}
Let $z \in \R^n$ be a random vector with i.i.d. $\Gaus(0, v^2)$ entries and let $D \in \R^{m\times n}$ be a linear operator.
Then for any constant vector $b \in \R^n$ the distribution of $z$ conditioned on $Dz = b$ satisfies:
\begin{align*}
    z
        &\disteq_{Dz = b}
            D^+ b + \Pi \tilde z
\end{align*}
where $D^+$ is the (Moore-Penrose) pseudoinverse, $\Pi$ is the orthogonal projection onto subspace $\{z: Dz = 0\}$, and $\tilde z$ is a random vector of i.i.d. $\Gaus(0, v^2)$.
\end{lemma}
\begin{proof}
When $D = [I_{m \times m} | 0_{m \times {n-m}}]$, this claim is immediate.
By rotational symmetry, this shows that, for any vector space $\mathcal V$ and $v$ orthogonal to it, conditioning $z$ on $\mathcal V + v$ yields a Gaussian centered on $v$ with covariance determined by $\Pi_{\mathcal V} z$.
Then the lemma in the general case is implied by noting that $\{z: Dz = b\}$ can be decomposed as $\{z: Dz = 0 \} + D^+ b$.
\end{proof}

\begin{lemma}\label{lemma:condTrick}
Let $A \in \R^{n \times m}$ be a matrix with random Gaussian entries, $A_{ij} \sim \Gaus(0, \sigma^2)$.
Consider fixed matrices $Q \in \R^{m \times q}, Y \in \R^{n \times q}, P \in \R^{n \times p}, X \in \R^{m \times p}$.
Suppose there exists a solution in $A$ to the equations $Y = AQ$ and $X = A^\trsp P$.
Then the distribution of $A$ conditioned on $Y = AQ$ and $X = A^\trsp P$ is
\begin{align*}
    A &\disteq_{Y=AQ, X=A^\trsp P} E + \Pi_P^\perp \tilde A \Pi_Q^\perp
\end{align*}
where
\begin{align*}
    E
        &=
            Y Q^+
            + P^{+\trsp} X^\trsp
            - P^{+\trsp} P^\trsp
                YQ^+,
\end{align*}
$\tilde A$ is an iid copy of $A$,
and $\Pi_P^\perp = I - \Pi_P = PP^+$ and $\Pi_Q^\perp = I - \Pi_Q=QQ^+$ in which $\Pi_P = I - PP^+$ and $\Pi_Q = I - QQ^+$ are the orthogonal projection to the space spanned by the column spaces of $P$ and $Q$ respectively.
\end{lemma}
\begin{proof}
We apply \cref{lemma:condTrickVec} to $D: A \mapsto (AQ, P^\trsp A)$.
The pseudoinverse of $D$ applied to $(Y, X^\trsp)$ can be formulated as the unique solution of
\begin{align*}
    \argmin_A \left\{ \|A\|^2_F : AQ = Y, P^\trsp A = X^\trsp \right\}
\end{align*}
where $\|-\|_F$ denotes Frobenius norm.
We check that $E$ is a 1) a solution to $AQ = Y, P^\trsp A = X^\trsp$ and 2) the minimal norm solution.

We have $EQ = 
        Y Q^+Q
            + P^{+\trsp} X^\trsp Q
            - P^{+\trsp} P^\trsp
                YQ^+Q$.
Note that $YQ^+Q = Y$ because $Y=AQ \implies YQ^+ Q = AQQ^+Q = AQ = Y$.
So $EQ = Y + P^{+T} (X^\trsp Q - P^\trsp Y)$.
But $X^\trsp Q = P^\trsp A Q = P^\trsp Y$, so $EQ = Y$ as desired.
A similar, but easier reasoning, gives $P^\trsp E = X^\trsp$.
This verifies that $E$ is a solution.

To check that $E$ is minimal norm, we show that it satisfies the stationarity of the Lagrangian
\begin{align*}
    L(A, \Theta, \Gamma)
        &=
            \|A\|^2_F + \la \Theta, Y - AQ \ra + \la \Gamma, X - A^\trsp P\ra.
\end{align*}
So $\pdf{L}{A} = 0 \implies 2A = \Theta Q^\trsp + P \Gamma^\trsp$ for some choices of $\Theta \in \R^{n \times q}$ and $\Gamma \in \R^{m \times p}$.
For $\Theta = 2 Y (Q^\trsp Q)^+$ and $\Gamma^\trsp = 2(P^\trsp P)^+ [ X^\trsp - P^\trsp Y Q^\trsp]$, we can check that
\begin{align*}
    \Theta Q^\trsp + P \Gamma^\trsp
        &=
            2 Y (Q^\trsp Q)^+ Q^\trsp + 2P(P^\trsp P)^+ [ X^\trsp - P^\trsp Y Q^+] \\
        &=
            2 Y Q^+ + 2 P^{+\trsp} X^\trsp - 2P^{+\trsp} P^\trsp Y Q^+\\
        &=
            2E
\end{align*}
as desired.
\end{proof}

\subsection{Probability Facts}

\begin{thm}[Strong Law of Large Numbers for triangular arrays \cite{hu_strong_1997}]\label{thm:SLLN}
Let $\{X_{n,i}: 1 \le i \le n, n \ge 1\}$ be a triangular array of random variables with $(X_{n, 1}, \ldots, X_{n, n})$ mutually independent with mean equal to zero for each $n$ and $\inv n \sum_{i=1}^n \EV |X_{n,i}|^{2+\rho} \le c n^{\rho/2}$ for some $0 < \rho < 1, c < \infty$.
Then $\f 1 n \sum_{i=1}^n X_{i, n} \to 0$ almost surely as $n \to \infty$.
\end{thm}

\begin{lemma}\label{lemma:momentBoundASConvergence}
Let $\{X_n\}_{n \ge 1}$ be a sequence of random variables with zero mean.
If for some $p \in \N$ and for all $n$, $\EV X_n^{2p} \le c n^{-1-\rho}$, for some $\rho > 0$, then $X_n \to 0$ almost surely.
\end{lemma}
\begin{proof}
By Markov's inequality, for any $\epsilon > 0$,
\begin{align*}
    \Pr(|X_n| > \epsilon)
        &=
            \Pr(X_n^{2p} > \epsilon^{2p})
        \le
            \EV X_n^{2p}/\epsilon^{2p}
        \le c n^{-1-\rho}/\epsilon^{2p}
        \\
    \sum_n \Pr(|X_n| > \epsilon)
        &\le
            \sum_n c n^{-1-\rho}/\epsilon^{2p}
        <    
            \infty.
\end{align*}
By Borel-Cantelli Lemma, almost surely, $|X_n| \le \epsilon$ for all large $n$.
Then, if we pick a sequence $\{\epsilon_k > 0\}_k$ converging to 0, we have that, almost surely, for each $k$, $|X_n| \le \epsilon_k$ for large enough $n$ --- i.e. almost surely, $X_n \to 0$.
\end{proof}

The following is a standard fact about multivariate Gaussian conditioning
\begin{prop}\label{prop:GaussianCondition}
Suppose $\R^{n_1 + n_2} \ni x \sim \Gaus(\mu, K)$, where we partition $x = (x_1, x_2) \in \R^{n_1} \times \R^{n_2}, \mu = (\mu_1, \mu_2) \in \R^{n_1} \times \R^{n_2}$, and $K = \begin{pmatrix} K_{11} & K_{12}\\ K_{21} & K_{22}\end{pmatrix}$.
Then
$x_1 \disteq_{x_2} \Gaus(\mu|_{x_2}, K|_{x_2})$
where
\begin{align*}
    \mu|_{x_2}
        &=
            \mu_1 - K_{12} K_{22}^+ (x_2 - \mu_2)\\
    K|_{x_2}
        &=
            K_{11} - K_{12} K_{22}^+ K_{21}.
\end{align*}

\end{prop}

\begin{lemma}\label{lemma:gaussianDer}
Let $\Phi: \R^n \to \R$ be measurable.
Then for $z \sim \Gaus(\zeta, \Sigma)$,
\begin{align*}
    \Jac{^2}{\zeta^2} \EV \Phi(z)
        &=
            2\Jac{}{\Sigma} \EV \Phi(z)
\end{align*}
whenever both sides exist.
\end{lemma}
\begin{proof}
First assume $\Sigma$ is invertible.
We check
\begin{align*}
    \Jac{}{\zeta} e^{-\f 1 2 (\zeta-z)\inv \Sigma (\zeta-z)}
        &=
            - \inv \Sigma (\zeta - z)
            e^{-\f 1 2 (\zeta-z)\inv \Sigma (\zeta-z)}
            \\
    \Jac{^2}{\zeta^2} e^{-\f 1 2 (\zeta-z)\inv \Sigma (\zeta-z)}
        &=
            \left[
                - \inv \Sigma
                + \inv \Sigma (\zeta-z) (\zeta-z)^\trsp \inv \Sigma
            \right]
            e^{-\f 1 2 (\zeta-z)\inv \Sigma (\zeta-z)}
            \\
    \Jac{}{\Sigma} \f{e^{-\f 1 2 (\zeta-z)\inv \Sigma (\zeta-z)}}{
                    \det(2\pi \Sigma)^{1/2}}
        &=
            \f 1 2 \left[
                - \inv \Sigma
                + \inv \Sigma (\zeta-z) (\zeta-z)^\trsp \inv \Sigma
            \right]
            \f{e^{-\f 1 2 (\zeta-z)\inv \Sigma (\zeta-z)}}{
                    \det(2\pi \Sigma)^{1/2}}
            \\
        &=
            \f 1 2
            \Jac{^2}{\zeta^2} e^{-\f 1 2 (\zeta-z)\inv \Sigma (\zeta-z)}
            .
\end{align*}
Integrating against $\Phi$ gives the result.
For general $\Sigma$, apply a continuity argument, since the set of invertible $\Sigma$s is dense inside the set of all PSD $\Sigma$.
\end{proof}

\begin{lemma}[Stein's lemma]\label{lemma:stein}
For jointly Gaussian random variables $Z_1, Z_2$ with zero mean, and any function $\phi: \R \to \R$ where $\EV \phi'(Z_1)$ and $\EV Z_1 \phi(Z_2)$ exists, we have
$$\EV Z_1 \phi(Z_2) = \Cov(Z_1, Z_2) \EV \phi'(Z_2).$$
\end{lemma}

\subsection{$\alpha$-controlled functions}
The next lemma is easy to show using the equivalence of norms in finite dimensional Euclidean space.
\begin{lemma}
Let $\phi: \R^k \to \R$.
The following are equivalent
\begin{enumerate}
    \item $\phi$ is $\alpha$-controlled
    \item For some $p \ge 1$ and some $g(x) = o_{\|x\|_p \to \infty}(\|x\|_p^\alpha)$, $C, c > 0$, $|\phi(x)| \le e^{C \|x\|^{\alpha}_p + g(x)}$
    \item For all $p \ge 1$, there is some $C, c > 0$, $|\phi(x)| \le e^{C \|x\|^{\alpha}_p + c}$
\end{enumerate}
\end{lemma}

\newcommand{\alpExp}{\mathsf{C}}
\begin{lemma}\label{lemma:alpExp}
Let $\alpExp_\alpha^k: \R^{\ge 0} \to \R, c \mapsto \EV_{z \sim \Gaus(0, I_k)}e^{c \|z\|^\alpha_2} $.
Then
\begin{enumerate}
    \item $\alpExp_\alpha^k < \infty$ iff $\alpha < 2$
    \item for $\alpha \ge 1$,
        \begin{align*}
            \EV_{z\sim \Gaus(\mu, \Sigma)}e^{C \|z\|^\alpha_2} \le
                e^{C\|\mu\|^\alpha_2 } \alpExp_\alpha^k(C  \alpha \|\Sigma\|_2^{\alpha/2})
        \end{align*}
    where $\|\Sigma\|_2$ denotes the spectral norm of $\Sigma$.
    \item for any $\alpha$-controlled $\phi: \R^k \to \R$ with $\alpha \ge 1$, there is $C > 0$ such that for all $\mu \in \R^k, \Sigma \in \PSD^k$,
        \begin{align*}
            \EV_{z \sim \Gaus(\mu, \Sigma)} |\phi(z)|
                &\le
                    C e^{C\|\mu\|^\alpha_2 } \alpExp_\alpha^k(C \alpha \|\Sigma\|_2^{\alpha/2})
        \end{align*}
        where $\|\Sigma\|_2$ denotes the spectral norm of $\Sigma$.
\end{enumerate}
Note that the RHS is a montonic function in $\|\mu\|_2$ and $\|\Sigma\|_2$, in the sense that if $\|\mu\|_2$ and $\|\Sigma\|_2$ don't decrease, then the RHS will not decrease either.
\end{lemma}
\begin{proof}
The first claim is obvious and the third follows from the second easily.
For the second,
\begin{align*}
    \EV_{z\sim \Gaus(\mu, \Sigma)}e^{C \|z\|^\alpha_2}
        &\le
            \EV_{z\sim \Gaus(0, I)}e^{C \|\sqrt \Sigma z + \mu\|^\alpha_2}\\
        &\le
            \EV_{z\sim \Gaus(0, I)}e^{C \alpha \lp
                \|\sqrt \Sigma z\|^\alpha_2 + \|\mu\|^\alpha_2 
            \rp}\\
        &\le
            e^{C\|\mu\|^\alpha_2 }
            \EV_{z\sim \Gaus(0, I)}e^{C  \alpha \|\Sigma\|_2^{\alpha/2}
                \|z\|^\alpha_2}\\
        &=
            e^{C\|\mu\|^\alpha_2 } \alpExp_\alpha^k(C  \alpha \|\Sigma\|_2^{\alpha/2}).
\end{align*}
\end{proof}

\subsection{Hermite Polynomials}
We follow a presentation roughly given by \citet{odonnell_analysis_2014}.

\newcommand{\pHerm}{\mathrm{He}}
\newcommand{\Herm}{H}

\begin{defn}
Let $\pHerm_n(x)$ be the {\it probabilist's Hermite polynomial}, given by the generating function $e^{xt - \f 1 2 t^2} = \sum_{n=0}^\infty \pHerm_n(x) \f{t^n}{n!}.$
Let $L^2(\R; \Gaus(0, 1))$ be the space of square-integrable functions against the standard Gaussian measure, equipped with inner product $\la \phi, \psi \ra_G = \EV_{x \sim \Gaus(0, 1)} \phi(x) \psi(x)$ and norm $\|\phi\|_G^2 = \la \phi, \phi \ra_G$.
Let $\Herm_n(x) = \pHerm_n(x) / \|\pHerm_n\|_G$ be the normalized versions.
\end{defn}

\begin{fact}
$\{\pHerm_n(x)\}_{n \ge 0}$ form an orthogonal basis for $L^2(\R; \Gaus(0, 1))$ and $\{\Herm_n(x)\}_{n \ge 0}$ form an orthonormal basis for $L^2(\R; \Gaus(0, 1))$.
\end{fact}
\begin{fact}
$\|\pHerm_n\|_G^2 = n!$ so that $\Herm_n(x) = \pHerm_n(x)/\sqrt{n!}$.
\end{fact}

Suppose $u^1, \ldots, u^k$ are unit vectors in $\R^k$, and let $\rho_{ij} := \la u^i, u^j\ra$.
Construct a zero mean Gaussian vector $z = (z_1, \ldots, z_k)$ such that $\EV z_i z_j = \rho_{ij}$.
Note that $z \disteq U g$ where $g = (g_1, \ldots, g_k)$ is a standard Gaussian vector and $U = (u^i_j)_{i,j=1}^k$ is the matrix with $u^i$ as rows.
Then for any $s = (s_1, \ldots, s_k)$ we can compute
\begin{align*}
    \EV \exp(\la s, z \ra)
        &=
            \EV \exp(s^\trsp U g)
        = 
            \EV \prod_i \exp(g_i (U^\trsp s)_i)
            \\
        &=
            \prod_i \EV \exp(g_i (U^\trsp s)_i)
            \\
        &
            \pushright{\text{by independence of $\{g_i\}_i$}}
            \\
        &=
            \prod_i \exp\left(\f 1 2 (U^\trsp s)_i^2\right)
            \\
        &=
            \exp\left(\f 1 2 \sum_i (U^\trsp s)_i^2\right)
            \\
        &=
            \exp\left(\f 1 2 \|U^\trsp s\|^2\right)
            \\
        &=
            \exp \lp \f 1 2 \sum_{i,j} \la u^i, u^j \ra s_i s_j \rp
            \\
        &=
            \exp \lp \f 1 2 \sum_{i,j} \rho_{ij} s_i s_j \rp
            .
\end{align*}
Dividing by $\exp\lp \f 1 2 \sum_i s_i^2 \rp$, we obtain
\begin{align*}
    \EV \exp\left(\sum_i s_i z_i - s_i^2\right)
        &=
            \exp \lp \sum_{i< j} \rho_{ij} s_i s_j \rp
            \\
    \EV \prod_i \sum_m \pHerm_m(z_i) (m!)^{-1} s_i^m
        &=
            \prod_{i< j} \sum_n (n!)^{-1}\lp \rho_{ij} s_i s_j\rp^n
            \\
    \sum_{(m_i)_{i=1}^k}
        \prod_i \f{s_i^{m_i}}{m_i!} \EV \prod_i \pHerm_{m_i}(z_i)
        &=
            \sum_{(n_{(ij)})_{i<j}}
                \prod_i s_i^{\sum_{j \ne i} n_{(ij)}}
                \prod_{i<j} \f{\rho_{ij}^{n_{(ij)}}}
                                {n_{(ij)}!}
\end{align*}
where $m_i \ge 0$ for all $i$, and $n_{(ij)} = n_{(ji)} \ge 0$ are indexed by unordered sets $\{i,j\}$.
Matching coefficients of $s$, we get
\begin{thm}
For any sequence $(m_i \ge 0)_{i=1}^k$,
\begin{align*}
    \EV \prod_i \pHerm_{m_i}(z_i)
        &=
            \lp \prod_r m_r! \rp 
            \lp \prod_{i<j} \f{\rho_{ij}^{n_{(ij)}}}
                                {n_{(ij)}!}
                \rp
            \\
    \EV \prod_i \Herm_{m_i}(z_i)
        &=
            \lp \prod_r \sqrt{m_r!} \rp 
            \lp \prod_{i<j} \f{\rho_{ij}^{n_{(ij)}}}
                                {n_{(ij)}!}
                \rp
\end{align*}
whenever there are $(n_{(ij)} \ge 0)_{i<j}$ such that, for all $i$, $m_i = \sum_{j\ne i} n_{(ij)}$.
$\EV \prod_i \pHerm_{m_i}(z_i) = 0$ otherwise.
\end{thm}

In particular, 
\begin{thm}\label{thm:multivariateHermiteExpectation}
If $\phi_i: \R \to \R$ has Hermite expansion $\phi_i(z) = \sum_{u=0}^\infty a_{iu} \Herm_u(z) = \sum_{u=0}^\infty b_{iu} \pHerm_u(z)$ where $b_{iu} = a_{iu}/\sqrt{u!}$, then
\begin{align*}
    \EV \prod_i \phi_i(z_i)
        &=
            \sum_{(n_{(ij)})_{i<j}}
                \lp \prod_r b_{r m_r} m_r! \rp 
                \lp \prod_{i<j} \f{\rho_{ij}^{n_{(ij)}}}
                                    {n_{(ij)}!}
                    \rp
            \\
        &=
            \sum_{(n_{(ij)})_{i<j}}
                \lp \prod_r a_{r m_r} \sqrt{m_r!} \rp 
                \lp \prod_{i<j} \f{\rho_{ij}^{n_{(ij)}}}
                                    {n_{(ij)}!}
                    \rp
            \\
        &=
            \sum_{(n_{(ij)})_{i<j}}
                \lp \prod_r a_{r m_r} \sqrt{\binom{m_i}{\{n_{(ij)}\}_{j \ne i}}} \rp 
                \lp \prod_{i<j} \rho_{ij}^{n_{(ij)}}
                    \rp
\end{align*}
where $m_i = \sum_{j\ne i} n_{(ij)}$, whenever the RHS is absolutely convergent.
\end{thm}

\begin{lemma}\label{lemma:smallRhoMomentBound}
Suppose $\phi_i, i\in[k]$ are as in \cref{thm:multivariateHermiteExpectation}, with additionally the constraint that we have an index set $I \sbe [k]$ such that $b_{i0} = a_{i0} = 0$ (i.e. $\EV \phi_i(z_i) = 0$) for all $i \in I$.
Assume that, for some $\lambda < 1/2 $, $|\rho_{ij}| \le \lambda/(k-1)$ for all $i \ne j$.
Then
\begin{align*}
    \left|\EV \prod_{i=1}^k \phi_i(z_i)\right| \le
    C_{k,|I|} \lp \prod_{r=1}^k \|\phi_r\|_G\rp 
        \lambda^{\lceil |I|/2\rceil}
\end{align*}
for some constant $C_{k, |I|}$ depending on $k$ and $|I|$ but independent of $\{\phi_i\}_i$ and $\lambda$.
\end{lemma}
\begin{proof}
In the notation of \cref{thm:multivariateHermiteExpectation}, $\binom{m_i}{\{n_{(ij)}\}_{j \ne i}} \le (k-1)^{m_i}$ by the multinomial theorem.
Thus
\begin{align*}
    \left|\EV \prod_{i=1}^k \phi_i(z_i)\right|
    &\le
        \sum_{\substack{(n_{(ij)})_{i<j}:\\\forall r\in I, m_r \ge 1}}
        \left|
        \lp \prod_{r=1}^k a_{r m_r} \sqrt{\binom{m_r}{\{n_{(rj)}\}_{j \ne r}}} \rp 
        \lp \prod_{i<j} \rho_{ij}^{n_{(ij)}}
            \rp
        \right|
        \\
    &\le
        \sum_{\substack{(n_{(ij)})_{i<j}:\\\forall r \in I, m_r \ge 1}}
        \lp \prod_{r=1}^k \|\phi_r\|_G \sqrt{(k-1)^{m_r}} \rp 
        \lp \prod_{i<j} \lp \f \lambda{k-1} \rp^{n_{(ij)}}
            \rp
        \\
    &\le
        \sum_{\substack{(n_{(ij)})_{i<j}:\\\forall r \in I, m_r \ge 1}}
        \lp \prod_{r=1}^k \|\phi_r\|_G  \rp 
        \lambda^{\sum_{i<j} n_{(ij)}}
        \\
    &=
        \lp \prod_{r=1}^k \|\phi_r\|_G\rp 
        \lp B_{|I|} \lambda^{\lceil |I|/2\rceil} (1 + o(1)) \rp
        .
\end{align*}
where $B_V$ is the number of ways to cover $V$ vertices with $\lceil V/2 \rceil$ edges, and
$o(1)$ is a term that goes to 0 as $\lambda \to 0$ and is bounded above by a function of $k$ whenever $\lambda < 1/2$.
Then an appropriate $C_{k, |I|}$ can be chosen to obtain the desired result.
\end{proof}

\begin{lemma}\label{lemma:largeRhoMomentBound}
Suppose $\phi_i, i\in[k]$ are as in \cref{thm:multivariateHermiteExpectation}, with additionally  the constraint that, we have some index set $I \sbe [3, k]$ such that for all $i \in I$, $b_{i0} = a_{i0} = 0$ (i.e. $\EV \phi_i(z_i) = 0$).
Assume that $|\rho_{12}| \le 1/2$, for some $\lambda < 1/\sqrt 8 $, $|\rho_{ij}| \le \lambda/(k-1)$ for all $i \ne j$ and $\{i,j\} \ne \{1, 2\}$.
Then
\begin{align*}
    \left|\EV \prod_{i=1}^k \phi_i(z_i)\right| \le
    C'_{k, |I|} \lp \prod_{r=1}^k \|\phi_r\|_G\rp 
        \lambda^{\lceil |I|/2\rceil}
\end{align*}
for some constant $C'_k$ depending on $k$ and $I$ but independent of $\{\phi_i\}_i$ and $\lambda$.
\end{lemma}
\begin{proof}
\newcommand{\suchthat}{\mathrm{s.t.}}
Define $\mathcal P = \{(i,j): 1\ne i<j \ne 2\}$ and $\mathcal Q = \{(i,j): i<j \And (\text{$i=1$ XOR $j=2$})\}$.
Also write $R = \prod_{r=1}^k \|\phi_r\|_G.$
As in the above proof,
\begin{align*}
    |\EV \prod_{i=1}^k \phi_i(z_i)|
    &\le
        \sum_{\substack{(n_{(ij)})_{i<j}:\\\forall r \in I, m_r \ge 1}}
        \left|
        \lp \prod_{r=1}^k a_{r m_r} \sqrt{\binom{m_r}{\{n_{(rj)}\}_{j \ne r}}} \rp 
        \lp \prod_{(i,j) \in \mathcal P} \rho_{ij}^{n_{(ij)}}
            \rp
        2^{-n_{(12)}}
        \right|
        \\
    &\le
        \sum_{\substack{(n_{(ij)})_{i<j}:\\\forall r \in I, m_r \ge 1}}
        R
        \sqrt{\prod_{r=1}^2 \binom{m_r}{n_{(12)}}
                    \binom{m_r - n_{(12)}}{\{n_{(rj)}\}_{j \not \in \{1,2\}}}
            }
        \prod_{r=3}^k \sqrt{(k-1)^{m_r}}
        \lp \prod_{(i,j) \in \mathcal P} \lp \f \lambda{k-1} \rp^{n_{(ij)}}
            \rp
        2^{-n_{(12)}}
        \\
    &\le
        R
        \sum_{\substack{(n_{(ij)})_{i<j}:\\\forall r \in I, m_r \ge 1}}
        \sqrt{\prod_{r=1}^2 \binom{m_r}{n_{(12)}}
                    (k-1)^{m_r - n_{(12)}}
            }
        (k-1)^{\f 1 2 \sum_{r=3}^k m_r}
         \lp \f \lambda{k-1} \rp^{\sum_{i<j} n_{(ij)} - n_{(12)}}
        2^{-n_{(12)}}
        \\
    &=
        R
        \sum_{\substack{(n_{(ij)})_{i<j}:\\\forall r \in I, m_r \ge 1}}
        \sqrt{\prod_{r=1}^2 \binom{m_r}{n_{(12)}}
            }
         \lambda^{\sum_{i<j} n_{(ij)} - n_{(12)}}
         2^{-n_{(12)}}
        \\
    &\le
        R
        \sum_{\substack{(n_{(ij)})_{i<j}:\\\forall r \in I, m_r \ge 1}}
        \binom{\f{m_1+m_2}2}{n_{(12)}}
         2^{-n_{(12)}}
         \lambda^{\sum_{i<j} n_{(ij)} - n_{(12)}}
        \\
    &\le
        R
        \sum_{\substack{(n_{(ij)})_{i<j}:\\\forall r \in I, m_r \ge 1}}
        \binom{n_{(12)} + \f 1 2 m_{(12)}}{n_{(12)}}
         2^{-n_{(12)}}
         \lambda^{\sum_{i<j} n_{(ij)} - n_{(12)}}
        \\
    &\qquad {\text{where $m_{(12)} = \sum_{(i,j) \in \mathcal Q} n_{(ij)}$}}
        \\
    &\le
        2R
        \sum_{\substack{(n_{(ij)})_{(i,j) \in \mathcal P}:\\\forall r \in I, m_r \ge 1}}
        \lp \f 1 {1 - 1/2} \rp^{1 + \f 1 2 m_{(12)}}
         \lambda^{\sum_{(i,j) \in \mathcal P} n_{(ij)}}
        \\
    &\le
        2R
        \sum_{\substack{(n_{(ij)})_{(i,j) \in \mathcal P}:\\\forall r \in I, m_r \ge 1}}
         (\sqrt 2 \lambda)^{\sum_{(i,j) \in \mathcal Q} n_{(ij)}}
         \lambda^{\sum_{(i,j)\in \mathcal P \setminus \mathcal Q} n_{(ij)}}
        \\
    &\le
        2R
        \lp 2^{|I|/4} B_{|I|} \lambda^{\lceil |I|/2\rceil} (1 + o(1)) \rp
\end{align*}
where $B_{|I|}$ is the number of ways of covering $|I|$ vertices with $\lceil \f{|I|}2 \rceil$ edges, and
$o(1)$ is a term that goes to 0 as $\lambda \to 0$ and is upper bounded by a function of $k$ for all $\lambda < 1/\sqrt 8$.
Choosing the appropriate constant $C'_{k, |I|}$ then gives the result.
\end{proof}

\subsection{Moment Bounds of Lightly Correlated Gaussians}

\begin{lemma}\label{lemma:projectionDiagonal}
Let $\Pi \in \R^{n \times n}$ be an orthogonal projection matrix.
Then each diagonal entry $\Pi_{ii} \in [0, 1].$
\end{lemma}
\begin{proof}
Because $\Pi = \Pi^2$, we have for each $i$, $\Pi_{ii} = \sum_{j} \Pi_{ij}^2 \implies \Pi_{ii} (1 - \Pi_{ii}) = \sum_{j \ne i} \Pi_{ij}^2 \ge 0 \implies \Pi_{ii} \in [0, 1].$
\end{proof}

\begin{lemma}\label{lemma:projectionCorrelation}

Let $\Pi \in \R^{n \times n}$ be an orthogonal projection matrix of rank $k$.
Consider the correlation matrix $C \defeq D^{-1/2} \Pi D^{-1/2}$ where $D = \Diag(\Pi)$.
Then the off-diagonal entries of $C$ satisfy $\sum_{i < j} C_{ij}^2 \le 2.5 (n - k)^2,$ assuming $k \ge n/2$.
\end{lemma}
\begin{proof}
Because $\Pi = \Pi^2$, we have for each $i$, $\Pi_{ii} = \sum_{j} \Pi_{ij}^2 = \sum_j \Pi_{ii}\Pi_{jj} C_{ij}^2
\implies
1 - \Pi_{ii} = \sum_{j \ne i} \Pi_{jj} C_{ij}^2.$
At the same time, each $C_{ij}^2 \in [0, 1]$.
Thus we seek an upper bound on the following linear program in the $n(n-1)/2$ variables $C_{(ij)}^2$ (which identiy $C_{ij} = C_{ji} = C_{(ij)}$).
\begin{align*}
    \text{Maximize }&
        \sum_{i \ne j} C_{(ij)}^2
        \\
    \text{s.t. }&
        \forall i, 1 - \Pi_{ii} = \sum_{j \ne i} \Pi_{jj} C_{(ij)}^2
        \\
        &\forall i < j, C_{(ij)}^2 \in [0, 1].
\end{align*}

This LP has the dual
\begin{align*}
    \text{Minimize }&
        \sum_{i < j} \tau_{ij} + \sum_i (1 - \Pi_{ii}) \zeta_i\\
    \text{s.t. }&
        \forall i < j,
            \tau_{ij} + \zeta_i \Pi_{jj} + \zeta_j \Pi_{ii} \ge 1\\
        &\forall i < j,
            \tau_{ij} \ge 0
            \\
        &\forall i, \zeta_i \in \R.
\end{align*}

Any feasible value of the dual LP is an upper bound on the original LP.
We now set the dual variables.

WLOG, assume $\Pi_{11} \ge \cdots \ge \Pi_{nn}$.
Then necessarily, $1 \ge \Pi_{ii} \ge \f k n$ for each $i \in [k]$.
First define $\rho \defeq \f k n + \f n k = \f{k^2 + n^2}{nk} \ge 2$.
Note that $\rho - 2 = \f{(n-k)^2}{nk}.$

\paragraph{Dual variables for $1 \le i < j \le k$.}
Now set $\zeta_i \defeq \f 1 {\rho \Pi_{ii}}$ for $i \in [k]$.
Then for $1 \le i< j \le k$, 
\begin{align*}
    \tau_{ij}
        &\defeq
            1 - (\zeta_i \Pi_{jj} + \zeta_j \Pi_{ii})
            \\
        &=
            1  - \f 1 \rho \lp r_{ij} + \inv r_{ij}\rp
\end{align*}
where $r_{ij} = \Pi_{ii}/\Pi_{jj} \ge 1$.
Note that 1) $\tau_{ij}$ is nonnegative:
indeed, since $r + \inv r$ is increasing in $r$ for $r \ge 1$, and $r_{ij} \le r_{1k} \le n/k$, we have $r_{ij} + \inv r_{ij} \le \rho$, so that $\tau_{ij} \ge 0$;
2) $\tau_{ij} \le \f{(n-k)^2}{n^2 + k^2}$: $r_{ij} + \inv r_{ij} \ge 2$ so $\tau_{ij} \le 1 - 2/\rho = \f{(n-k)^2}{n^2 + k^2}.$

\paragraph{$\zeta_j$ for $j > k$.}
Now set $\zeta_j \defeq \f 1 {\Pi_{11}} \lp 1 - \f {\Pi_{jj}}{2\Pi_{11}}\rp$ for $k < j \le n$.
Note that $\zeta_j \ge 1/2.$

\paragraph{$\tau_{ij}$ for $i \le k < j$.}
Then for $i \le k < j$, set
\begin{align*}
    \tau_{ij}
        &\defeq
            1 - \zeta_i \Pi_{jj} - \zeta_j \Pi_{ii}
            \\
        &=
            1 - \f{\Pi_{jj}}{\rho \Pi_{ii}}
            - \zeta_j \Pi_{ii}.
\end{align*}
Note that for $\Pi_{11} = x \ge y \ge k/n$,
\begin{align*}
    &\phantomeq
        \f{\Pi_{jj}}{\rho x} + \zeta_j x 
        - \lp \f{\Pi_{jj}}{\rho y} + \zeta_j y \rp
        \\
    &=
        \f{\Pi_{jj}}\rho \f{y - x}{xy} + \zeta_j (x - y)
        \\
    &=
        (x - y) \lp \zeta_j - (xy)^{-1} \f{\Pi_{jj}}\rho\rp
        \\
    &=
        (x - y)\lp x^{-1} \lp 1 - \f{\Pi_{jj}}{2x}\rp - (xy)^{-1} \f{\Pi_{jj}}{\rho}\rp
        \\
    &=
        (x - y)x^{-1} \lp 1 - \f{\Pi_{jj}}{2x} - \inv y \f{\Pi_{jj}}{\rho}\rp
        \numberthis \label{eqn:zetaDiff}
        \\
    &\ge
        (x - y)x^{-1} \lp 1 - \f{k/n}{2x} - (k/n)^{-1} \f{k/n}{\rho}\rp
        \\
    &\ge
        (x - y)x^{-1} \lp 1 - \f 1 2 - \f 1 2\rp
        \\
    &\ge 0.
\end{align*}
Thus for all $i \le k < j$, $ \tau_{1j} \le \tau_{ij}$.
Simultaneously, Eq.~(\ref{eqn:zetaDiff}) also shows that $\f{\Pi_{jj}}{\rho x} + \zeta_j x 
        - \lp \f{\Pi_{jj}}{\rho y} + \zeta_j y \rp \le (x - y)\inv x$, so that $\tau_{ij} - \tau_{1j} \le \f{\Pi_{11} - \Pi_{ii}}{\Pi_{11}} \le \f{n-k}n.$
We check that $\tau_{1j} \ge 0$:
\begin{align*}
    \tau_{1j}
        &=
            1 - \f{\Pi_{jj}}{\rho \Pi_{11}} - \zeta_j \Pi_{11}
            \\
        &=
            1 - \f{\Pi_{jj}}{\rho \Pi_{11}} - \lp 1 - \f{\Pi_{jj}}{2 \Pi_{11}}\rp
            \\
        &=
            \f{\Pi_{jj}}{\Pi_{11}}\lp \f 1 2 - \f 1 \rho \rp
            \\
        &\in
            \left[0, \f{\Pi_{jj}}{\Pi_{11}} \f{1}{2\rho} \f{(n-k)^2}{nk}\right]
            \\
        &\sbe
            \left[0, \f{(n-k)^2}{nk}\right]
\end{align*}
Combined with our deduction above, we get $\tau_{ij} \le \tau_{1j} + \f{n-k}n \le \f{n(n-k)}{nk} = \f{n-k}{k}.$

\paragraph{$\tau_{ij}$ for $k < i < j$.}
For $k < i < j$, we set
\begin{align*}
    \tau_{ij}
        &\defeq
            1 - \lp \zeta_i \Pi_{jj} + \zeta_j \Pi_{ii}\rp\\
        &=
            1 - \lp 
                \f {\Pi_{jj}} {\Pi_{11}} \lp 1 - \f {\Pi_{ii}}{2\Pi_{11}}\rp
                + \f {\Pi_{ii}} {\Pi_{11}} \lp 1 - \f {\Pi_{jj}}{2\Pi_{11}}\rp
                \rp\\
        &=
            1 - \lp
                \f{\Pi_{jj} + \Pi_{ii}}{\Pi_{11}}
                - \f{\Pi_{jj} \Pi_{ii}}{\Pi_{11}^2}
                \rp
            \\
        &=
            \f{(\Pi_{11} - \Pi_{jj})(\Pi_{11} - \Pi_{ii})}
            {\Pi_{11}^2}
            \\
        &\in
            [0, 1].
\end{align*}

\paragraph{Summary of dual variables.}
In summary, we have
\begin{align*}
    \forall 1 \le i \le k,
    \zeta_i
        &\defeq
            (\rho \Pi_{ii})^{-1}
        \in
            \left[\inv \rho, \inv \rho \f n k\right]
            \\
    \forall k < i \le n,
    \zeta_i
        &\defeq
            \f 1 {\Pi_{11}} \lp 1 - \f {\Pi_{jj}}{2\Pi_{11}}\rp
        \in
            \left[\f 1 2, \f n k\right]
            \\
    \forall 1 \le i < j \le k,
    \tau_{ij}
        &\in
            \left[0, \f{(n-k)^2}{n^2 + k^2}\right]
            \\
    \forall i \le k < j \le n,
    \tau_{ij}
        &\in
            \left[0, \f{n-k}k\right]
            \\
    \forall k < i < j \le n,
    \tau_{ij}
        &\in 
            [0, 1].
\end{align*}

\paragraph{Objective value given by the dual variables.}
Now we compute
\begin{align*}
    &\phantomeq
        \sum_{i < j} \tau_{ij} + \sum_i (1 - \Pi_{ii}) \zeta_i
        \\
    &=
        \lp \sum_{i < j \le k} + \sum_{i \le k < j} + \sum_{k < i < j} \rp \tau_{ij}
        + 
        \lp \sum_{i=1}^k + \sum_{i=k+1}^n\rp (1 - \Pi_{ii}) \zeta_i
        \\
    &\le
        \lp
        \f{k(k-1)}2 \f{(n-k)^2}{n^2 + k^2}
        +
        k(n-k) \f{n-k}{k}
        +
        \f{(n-k)(n-k-1)}2
        \rp
        +
        \lp
        k\f{n-k}n \f{n}{k\rho}
        +
        (n-k) \f{n}k
        \rp
        \\
    &<
        (n-k)^2\lp \f 1 4 + 1 + \f 1 2\rp
        + (n-k)\lp \f 1 \rho + \f n k\rp
        \\
    &\le
        2.5 (n-k)^2
\end{align*}
assuming $k \ge n/2$.

\end{proof}

\begin{lemma}\label{lemma:projectVariance}
Let $z \sim \Gaus(0, \Pi)$ where $\Pi \in \R^{n \times n}$ is an orthogonal projection matrix of rank $k$.
Suppose $\phi_i: \R \to \R$ for each $i$ has finite variance $\Var\lp \phi_i(x) : x \sim \Gaus(0, \Pi_{ii})\rp$.
Then
\begin{align*}
    \Var\lp\sum_{i=1}^n \phi_i(z_i)\rp
        &\le
            \lp\f{5(n-k)^2}{\sqrt 2} + 1 \rp \sum_i \Var\lp \phi_i(x) : x \sim \Gaus(0, \Pi_{ii})\rp.
\end{align*}
In particular, if $n-k = O(1)$, then $\Var\lp\sum_{i=1}^n \phi_i(z_i)\rp = \Theta\lp \sum_i \Var\lp \phi_i(x) : x \sim \Gaus(0, \Pi_{ii})\rp\rp.$
\end{lemma}
\begin{proof}
Let $C = D^{-1/2} \Pi D^{-1/2}, D = \Diag(\Pi),$ be the correlation matrix of $\Pi$.
Let $\psi_i(y) = \phi_i(\sqrt{\Pi_{ii}}y) - \EV_{x \sim \Gaus(0, 1)}[\phi_i(\sqrt{\Pi_{ii}}x)]$.
Then
\begin{align*}
\Var_{z \sim \Gaus(0, \Pi)}\lp \sum_{i=1}^n \phi_i(z_i) \rp
    &=
        \EV_{z \sim \Gaus(0, C)}
            \lp \sum_{i=1}^n \psi_i(z_i) \rp^2\\
    &=
        \EV_{z \sim \Gaus(0, C)}
            \sum_{i,j} \psi_i(z_i) \psi_j(z_j).
\end{align*}
Expand $\psi_i$ in the Hermite orthonormal basis,
\begin{align*}
    \psi_i(x) = a_{i1} H_1(x) + a_{i2} H_2(x) + \cdots
\end{align*}
where $H_j(x)$ is the $j$th Hermite polynomial, normalized so that $\EV_{z \sim \Gaus(0, 1)} H_j(z)^2 = 1$, (note that $H_0(x) = 1$ and does not appear here because $\EV_{x \sim \Gaus(0, 1)}\psi_i(x) = 0$ by construction).
For any locally integrable $\phi: \R \to \R$, let $\|\phi\|_G^2 \defeq \EV_{z \sim \Gaus(0, 1)} \phi(z)^2$, so that $\|\psi_i\|_G^2 = \sum_k a_{ik}^2 = \Var\lp \phi_i(x) : x \sim \Gaus(0, \Pi_{ii})\rp.$
Then,
\begin{align*}
    \sum_{i < j} \EV_{z \sim \Gaus(0, C)} \psi_i(z_i) \psi_j(z_j)
        &=
            \sum_{i<j} \sum_{k=1}^\infty a_{ik} a_{jk} C_{ij}^k
            \\
        &\le
            \sum_{k=1}^\infty \sqrt{\lp \sum_{i<j} a_{ik}^2 a_{jk}^2\rp
                                    \lp \sum_{i<j} C_{ij}^{2k}\rp}
            \\
        &\le
            \sum_{k=1}^\infty \sqrt{\f 1 2 \lp \sum_{i} a_{ik}^2\rp^2
                                    \lp \sum_{i<j} C_{ij}^2\rp}
            \\
        &
            \pushright{\text{since $|C_{ij}| \le 1$}}
            \\
        &\le
            2^{-1/2}\sum_{k=1}^\infty \lp \sum_{i} a_{ik}^2\rp (2.5 (n-k)^2)
            \\
        &
            \pushright{\text{by \cref{lemma:projectionCorrelation}}}
            \\
        &=
            \f{5(n-k)^2}{2^{3/2}} \sum_i \|\psi_i\|_G^2
\end{align*}
On the other hand, $\sum_i \EV_{x \sim \Gaus(0, 1)} \psi_i(x)^2 = \sum_i \|\psi\|_G^2$, so that
\begin{align*}
    \sum_{i,j}\EV_{z \sim \Gaus(0, C)} \psi_i(z_i) \psi_j(z_j)
        &\le
            \lp\f{5(n-k)^2}{\sqrt 2} + 1 \rp \sum_i \|\psi_i\|_G^2
            \\
        &=
            \lp\f{5(n-k)^2}{\sqrt 2} + 1 \rp \sum_i \Var\lp \phi_i(x) : x \sim \Gaus(0, \Pi_{ii})\rp.
\end{align*}

\end{proof}

\begin{thm}\label{thm:controlHighMoments}
\newcommand{\LL}{2p}
Let $z \sim \Gaus(0, \Pi)$ where $\Pi \in \R^{n \times n}$ is an orthogonal projection matrix of rank $n - O(1)$, where $O(1)$ denotes a quantity that stays bounded as $n \to \infty$.
Suppose $\phi_i: \R \to \R$ for each $i \in [n]$ has finite centered moments $\EV_x[(\phi_i(x) - \EV_{x'} \phi_i(x'))^{r}]$, for $x, x' \sim \Gaus(0, \Pi_{ii})$, for all $r \le \LL$, where $p \ge 6$.
Then for $Q \defeq \f 1 n \sum_{i=1}^n \phi_i(z_i),$ as $n \to \infty$,
\begin{align*}
    \EV[(Q - \EV Q)^{2p}]
        &\le
           O\lp
            n^{-1.5} \max_{i \in [n]} \EV_x\left[\left(\phi_{i}(x) - \EV_{x'} \phi_i(x')\right)^{2p}: 
                        x, x' \sim \Gaus(0, \Pi_{ii})\right]
            \rp.
\end{align*}
If in addition, each $\phi_i$ has finite centered moments up to $r \le 2p L$ for some $L > 1$, then
\begin{align*}
    \EV[(Q - \EV Q)^{2p}]
        &\le
           O\lp
            n^{-1.5+1/L} \sqrt[L]{\f 1 n \sum_{i=1}^n 
                \EV_x\left[\left(\phi_{i}(x) - \EV_{x'} \phi_i(x')\right)^{2p L}: 
                        x, x' \sim \Gaus(0, \Pi_{ii})\right]}
            \rp.
\end{align*}
Here $O(-)$ hides constants that do not depend on $n$, any of the functions $\phi_i$, or $\Pi$.
\end{thm}
\begin{proof}
\newcommand{\Cijt}[3]{C_{\p{(#1 #2)}{#3}}}
Let $C = D^{-1/2} \Pi D^{-1/2}, D = \Diag(\Pi),$ be the correlation matrix of $\Pi$.
Let $\psi_i(y) = \phi_i(\sqrt{\Pi_{ii}}y) - \EV_{x \sim \Gaus(0, 1)}[\phi_i(\sqrt{\Pi_{ii}}x)]$.

Order the off-diagonal entries of the correlation matrix in the order of decreasing squared value:
$$C_{(ij)^{(1)}}^2 \ge C_{(ij)^{(2)}}^2 \ge \ldots \ge C_{(ij)^{(N)}}^2,$$
where $N = \binom n 2$, and $(ij)^{(t)} = (\p i t \p j t)$ are unordered pairs of distinct indices $\p i t \ne \p j t$.
Since $\sum_t C_{\p{(ij)}{t}}^2 \le R$ for some constant $R$, by \cref{lemma:projectionCorrelation}, we deduce that $|\Cijt i j t| \le n^{-1/4}$ for all $t > R \sqrt n$.

Consider the $(2p)$th centered moment $\EV \lp \f 1 n \sum_{i=1}^n \psi_i(y_i)\rp^{2p} = \EV n^{-2p} \sum_{\sigma: [2p] \to [n]} \prod_{a=1}^{2p} \psi_{\sigma(a)}(y_{\sigma(a)})$, where $y \sim \Gaus(0, C).$
We shall bound the sum to show that this moment is not too large.

First note the naive bound via AM-GM,
\begin{align*}
    \EV \left|\prod_{a=1}^{2p} \psi_{\sigma(a)}(y_{\sigma(a)})\right|
    &\le
        \EV \f 1 {2p} \sum_{a=1}^{2p} \psi_{\sigma(a)}(y_{\sigma(a)})^{2p}\\
    &\le
        \max_{i \in [n]} \EV_{y \sim \Gaus(0, 1)}[\psi_i(y)^{2p}]
        \\
    &=
        \max_{i \in [n]} \EV[(\phi_{i}(z_i) - \EV \phi_i(z_i))^{2p}: 
                        z_i \sim \Gaus(0, \Pi_{ii})]
        \\
    &\defeq
        B_{2p}.
        \numberthis
        \label{eqn:naiveAMGM}
\end{align*}
Now, for any collection of numbers $\{x_i \in \R \}_{i=1}^m$ and any $L > 0$, we have the trivial bound $\max_i |x_i| \le \lp \sum_{j=1}^m |x_j|^L \rp^{1/L}$, and this bound is tighter the larger $L$ is.
Thus $B_{2p} \le n^{1/L} \sqrt[L]{\f 1 n \sum_{i=1}^n (\EV[\psi_i(y)^{2p}])^L} \le n^{1/L} B_{2p,L}$, where $B_{2p,L} \defeq \sqrt[L]{\f 1 n \sum_{i=1}^n \EV[\psi_i(y)^{2pL}]}$, for any $L$.

\renewcommand{\CC}{\mathsf{C}}

We can categorize the $n^{2p}$ terms of $\EV \sum_{\sigma: [2p] \to [n]} \prod_{a=1}^{2p} \psi_{\sigma(a)}(y_{\sigma(a)})$ 
as follows.
Here we use $O(-)$ to hide any constant not depending on $n$ or the functions $\psi_i$.
\begin{itemize}
    \item Suppose $\sigma$ is injective.
    \begin{itemize}
        \item
            Suppose for each $a\ne b$, $(\sigma(a)\sigma(b)) = \p{(ij)}{t}$ for some $t > R\sqrt n$.
            By \cref{lemma:smallRhoMomentBound}, $\EV \prod_{a=1}^{2p} \psi_{\sigma(a)}(y_{\sigma(a)}) \le \CC \lp \prod_{r=1}^{2p} \|\psi_{\sigma(r)}\|_G\rp \lp n^{-1/4} \rp ^{p}$
            for some constant $\CC$ independent of $\{\psi_r\}_r$ and $n$.
            Thus the contribution of all such $\sigma$ to the sum is at most
            \begin{align*}
                \sum_\sigma \CC \lp \prod_{r=1}^{2p} \|\psi_{\sigma(r)}\|_G\rp n^{-p/4} 
                \le 
                O\lp n^{-p/4} \lp \sum_{i=1}^{n} \|\psi_{\sigma(r)}\|_G \rp^{2p}\rp.
            \end{align*}
        \item
            Suppose for some $a, b \in [2p]$, $(\sigma(a)\sigma(b)) = \p{(ij)}t$ for $t \le R\sqrt n$.
            There are at most $2\binom{2p}2 R \sqrt n \cdot n^{2p-2} = O(n^{2p-1.5})$ such $\sigma$ (indeed, there are $R \sqrt n$ of choosing such a $t$, $2\binom {2p} 2$ ways of choosing their preimages under $\sigma$ out of $2p$, and $\le n^{2p-2}$ ways of choosing the rest of the values of $\sigma$).
            By \cref{eqn:naiveAMGM}, the contribution of all such $\sigma$ to the sum is at most $O(n^{2p-1.5} B_{2p}).$
            
    \end{itemize}
    \item
        Suppose for some $a^* \ne b^*$ in $[2p]$, $\sigma(a^*) = \sigma(b^*)$, but $\sigma|_{[n] \setminus \{a^*, b^*\}}$ is injective and takes range outside $\{\sigma(a^*)\}$.
        There are $\binom {2p} 2 n \binom{n-1}{2p-2} = O(n^{2p-1})$ such $\sigma$.
        \begin{itemize}
            \item 
                Suppose for each $a\ne b$, $(\sigma(a)\sigma(b)) = \p{(ij)}{t}$ for some $t > R\sqrt n$, so that $|C_{\sigma(a)\sigma(b)}| \le n^{-1/4}.$
                We apply \cref{lemma:smallRhoMomentBound} to the functions $\{\psi_{\sigma(a^*)}^2\} \cup \{\psi_{\sigma(a)}\}_{a \not \in \{a^*, b^*\}}$, with $\psi_{\sigma(a^*)}^2$ being the sole function whose expectation is not 0, so that the $I$ of \cref{lemma:smallRhoMomentBound} has size $2p-2$, and the $\lambda$ of \cref{lemma:smallRhoMomentBound} is $(k-1)n^{-1/4}.$
                Then \cref{lemma:smallRhoMomentBound} gives
                \begin{align*}
                    \EV \prod_{a=1}^{2p} \psi_{\sigma(a)}(z_{\sigma(a)})
                        &\le
                            \CC
                            \|\psi^2_{\sigma(a^*)}\|_G
                            \lp \prod_{a\not\in\{a^*, b^*\}} \|\psi_{\sigma(a)}\|_G \rp
                            (n^{-1/4})^{(2p-2)/2}
                            \\
                        &=
                            \CC
                            \|\psi^2_{\sigma(a^*)}\|_G
                            \lp \prod_{a\not\in\{a^*, b^*\}} \|\psi_{\sigma(a)}\|_G \rp
                            n^{-(p-1)/4}
                \end{align*}
                for some constant $\CC$.
                Thus the collective contribution of such $\sigma$ to the sum is at most
                \begin{align*}
                    \sum_{\sigma} \CC
                            \|\psi^2_{\sigma(a^*)}\|_G
                            \lp \prod_{a\not\in\{a^*, b^*\}} \|\psi_{\sigma(a)}\|_G \rp
                            n^{-(p-1)/4}
                        &\le
                            O\lp 
                            n^{-(p-1)/4}
                            \lp \sum_{i=1}^n \|\psi_i^2\|_G\rp
                            \lp \sum_{i=1}^n \|\psi_i\|_G\rp^{2p-2}
                            \rp.
                \end{align*}
            \item
                 Suppose for some $a, b \in [2p]$, $(\sigma(a)\sigma(b)) = \p{(ij)}t$ for $t \le R\sqrt n$.
                 There are at most $\binom{2p}{2} n \cdot R\sqrt n \cdot \binom{n-2}{2p-3} = O(n^{2p-1.5})$ such $\sigma$.
                 Using \cref{eqn:naiveAMGM} again, we can upper bound the contribution of such $\sigma$ by $O(n^{2p-1.5} B_{2p}).$
        \end{itemize}
        \item Otherwise, there are more than one pair of inputs that collide under $\sigma$. 
        There are at most $O(n^{2p-2})$ such $\sigma$.
        Using \cref{eqn:naiveAMGM}, we upper bound their contributions by $O(n^{2p-2} B_{2p})$.
\end{itemize}
To summarize, 
\begin{align*}
    \EV \sum_{\sigma: [2p] \to [n]} \prod_{a=1}^{2p} \psi_{\sigma(a)}(y_{\sigma(a)})
    &\le
        O\lp
        n^{1.75p} B'_{2p}
        + n^{1.75(p-1)} B''_{2p}
        + n^{2p-1.5} B_{2p}
        \rp
        \\
    &\le
        O\lp
        n^{1.75p} B'_{2p}
        + n^{1.75(p-1)} B''_{2p}
        + n^{2p-1.5+1/L} B_{2p, L}
        \rp
        \\
    \EV \lp \f 1 n \sum_{i=1}^n \phi_i(z_i) - \EV \f 1 n \sum_{i=1}^n \phi_i(z_i)\rp^{2p}
        &\le
           O\lp
            n^{-0.25 p} B'_{2p}
            + n^{-0.25 p - 1.75} B''_{2p}
            + n^{-1.5} B_{2p}
            \rp
            \\
        &\le
           O\lp
            n^{-0.25 p} B'_{2p}
            + n^{-0.25 p - 1.75} B''_{2p}
            + n^{-1.5+1/L} B_{2p, L}
            \rp
\end{align*}
where
\begin{align*}
    B'_{2p}
        &=
            \lp \f 1 n \sum_{i=1}^{n} \|\psi_{i}\|_G \rp^{2p}
            \\
    B''_{2p}
        &=
            \lp \f 1 n \sum_{i=1}^n \|\psi_i^2\|_G\rp
            \lp \f 1 n \sum_{i=1}^n \|\psi_i\|_G\rp^{2p-2}.
\end{align*}

By the power mean inequality, we get that $B'_{2p}, B''_{2p} \le B_{2p} \le n^{1/L} B_{2p,L}$.
Substitution then gives the desired result.

\end{proof}

This theorem can be significantly strengthened with more careful case work and applying more involved versions of \cref{lemma:smallRhoMomentBound,lemma:largeRhoMomentBound}.

\begin{lemma}
Suppose $\phi: \R \to \R$ is polynomially bounded: $\forall x, |\phi(x)| \le C(1 + |x|^p)$ for some $C$.
Then for $y \sim \Gaus(0, 1)$,
\begin{align*}
    \left|\pd_\mu \Var \phi(\sigma y + \mu)\right|    
        &\le
            R\inv \sigma (1 + \sigma^{2p} + |\mu|^{2p})
\end{align*}
for a universal constant $R$ depending only on $p$ but not on $\mu$ or $\sigma$.
\end{lemma}
\begin{proof}
By \cref{lemma:gaussianDer},
\begin{align*}
    \pd_\mu \EV \phi(\sigma y + \mu)^2
        &=
            \inv \sigma \EV y \phi(\sigma y + \mu)^2
            \\
    \left|\pd_\mu \EV \phi(\sigma y + \mu)^2\right|
        &\le
            \inv \sigma C^2 \EV |y|(1 + |\sigma y + \mu|^p)^2
            \\
        &\le
            \inv \sigma C^2 \EV |y|(1 + 2^{p-1}(\sigma^p |y|^p + |\mu|^p))^2
            \\
        &\le
            \inv \sigma 3C^2 \EV |y|(1 + 2^{2p-2}(\sigma^{2p} |y|^{2p} + |\mu|^{2p}))
            \\
        &\le
            \inv \sigma C' (1 + \sigma^{2p} + |\mu|^{2p})
\end{align*}
where $C'$ depends only on $p$ but not on $\mu$ or $\sigma$.
Similarly, 
\begin{align*}
    \left|\pd_\mu \EV \phi(\sigma y + \mu)\right| 
        &\le
            \inv \sigma C'' (1 + \sigma^p + |\mu|^p)
            \\
    \left|\EV \phi(\sigma y + \mu)\right|
        &\le
            C''' (1 + \sigma^p + |\mu|^p)
\end{align*}
for constants $C'', C'''$ depending only on $p$ but not on $\mu$ or $\sigma$.
Therefore,
\begin{align*}
    \left|\pd_\mu \Var \phi(\sigma y + \mu)\right|
        &\le
            \left|\pd_\mu \EV \phi(\sigma y + \mu)^2\right|
            + 2 \left|\EV \phi(\sigma y + \mu)\right|
                \left|\pd_\mu \EV \phi(\sigma y + \mu)\right| 
            \\
        &\le
            \inv \sigma C' (1 + \sigma^{2p} + |\mu|^{2p})
            + 2C''' (1 + \sigma^p + |\mu|^p)
                \inv \sigma C'' (1 + \sigma^p + |\mu|^p)
            \\
        &\le
            C'''' \inv \sigma (1 + \sigma^{2p} + |\mu|^{2p})
\end{align*}
for constant $C''''$ depending only on $p$ but not on $\mu$ or $\sigma$.

\end{proof}

\section{Proof of Main Theorems}

\notransposeLimit*
\begin{proof}

WLOG we assume that all input vars appear in the beginning of the program.
We do induction on the length of the program.

\textbf{Base case.}
We show that for each $\cdc$, 
    $\f 1 {\nvar^{\cdc t}} \sum_{i=1}^{\nvar^{\cdc t}}\phi(\gvar^{\cdcin t}_i) \asto \EV \phi(z)$
where $\gvar^{\cdcin t}_i = (\gvar^{lt}_i)_{\gvar^l \in \cdcin} \disteq \Gaus(\mu^{\cdcin t}, K^{\cdcin t})$ and $z \sim \Gaus(\mu^{\cdcin\infty}, K^{\cdcin\infty})$.

For $\alpha$-controlled $\phi$, $\alpha \in [1, 2)$, for every $q> 0$, there is monotonic function $f$ such that $\EV_{x \sim \Gaus(\mu, K)} |\phi(x)|^q \le f(\|\mu\|_2, \|\Sigma\|_2)$ by \cref{lemma:alpExp} uniformly over all $\mu, K$.
If we let $X_{\nvar^{\cdc t}, i} = \phi(\gvar^{\cdcin t}_i) - \EV\phi(\gvar^{\cdcin t}_i)$, then $\f 1 {\nvar^{\cdc t}} \sum_{i=1}^{\nvar^{\cdc t}}\EV |X_{n, i}|^{2+\rho}$ is bounded uniformly over all $t$ as $\mu^{\cdcin t}$ and $K^{\cdcin t}$ are bounded uniformly over all $t$.
So by \cref{thm:SLLN}, $\f 1 {\nvar^{\cdc t}} \sum_{i=1}^{\nvar^{\cdc t}}\phi(\gvar^{\cdcin t}_i) - \EV\phi(\gvar^{\cdcin t}_i) \asto 0$.
Because $\EV_{z \sim \Gaus(\mu^{\cdcin t}, K^{\cdcin t})} \phi(z) \to \EV_{z \sim \Gaus(\mu^{\cdcin\infty}, K^{\cdcin\infty})} \phi(z)$ as $t \to \infty$ (for example by an easy application of dominated convergence), we have our desired result.

\textbf{Inductive Case.}
The inductive case for lines of type~\ref{linetype:lincomb} is obvious, as $\alpha$-controlled functions are closed under composition with linear transforms.

Suppose at time $t$, $g^t = \gvar^{Lt} = A^t h^t$ is a line of type~\ref{linetype:GLinear} and $h^t = f(\gvar^{l_{01}t}, \ldots, \gvar^{l_{0 k_0}t})$ is a line of type~\ref{linetype:nonlin} where $\gvar^{l_{01}t}, \ldots, \gvar^{l_{0 k_0}t}$ (resp. $A^t$) are some previous G-vars (resp. A-var).
(The case for $g^t = A^t g'{}^t$ for a G-var $g'{}^t$ can be reduced to this case by setting $f = \id$ and $h^t = f(g'{}^t)$).
Set $\cdc \defeq \cdc_1 \defeq \cdc(g) = \cdc_1(A)$ and $\cdc_2 \defeq \cdc(h) = \cdc_2(A).$
Suppose $g^i := A h^i$, $i = 1, \ldots, r$, are all previous lines of type \ref{linetype:GLinear} involving $A$.
Here $\{g^i\}_{i=1}^r$ are G-vars and $\{h^i\}_{i=1}^r$ are G- or H-vars, defined by \ref{linetype:nonlin} lines $h^i := f^i(\gvar^{l_{i1}}, \ldots, \gvar^{l_{ik_i}})$ for a collection of functions $\{f^i: \R \to \R\}_{i=1}^r.$
Let $G^t \defeq [g^{1t}|\cdots|g^{rt}] \in \R^{\nvar^{\cdc_1 t} \times r}, H^t \defeq [h^{1t}|\cdots |h^{rt}] \in \R^{\nvar^{\cdc_2 t} \times r}$, so that $G^t = A^t H^t.$
We will abuse notation and sometimes use $G$ to mean the collection of G-vars $\{g^j\}_{j=1}^r$.
Let $\Aa^t$ be the $\sigma$-algebra spanned by the G-vars appearing before $g$.
By the conditioning trick \cref{lemma:condTrick},
\begin{align*}
    g^t
        &\disteq_{\Aa^t}
            (G^t (H^t)^+ + \tilde A^t \Pi_{H^t}^\perp) h^t
\end{align*}
where $\tilde A^t$ is an independent copy of $A^t$ and $\Pi_{H^t} = H^t (H^t)^+$ is projection to the space spanned by the columns of $H^t$.
Each $i$th coordinate of $g^t$ is independent conditioned on $\Aa^t$, with mean $\mu_i^t \defeq G^t_{i:}(H^t)^+ h^t$ and standard deviation $\sigma^t \defeq \sigma^{At}\sqrt{\|\Pi_{H^t}^\perp h^t\|^2/\nvar_2(A^t)}$ where $\sigma^{At}$ is a shorthand for $\sigma^{\lno(A)t}$.
For simplicity, we assume $\sigma^{At} = 1$ for all $t$; the general case follows very easily from the reasoning below and the fact that $\sigma^{At}$ converges to a finite, nonzero value.

\begin{claim}
$(\sigma^t)^2 \asto (\sigma^\infty)^2 \defeq K^\cdc(g, g) - K^\cdc(g, G) K^\cdc(G, G)^+ K^\cdc(G, g)$.
\end{claim}

\begin{claimproof}
Note that $(\sigma^t)^2 = \f 1 {\nvar^{\cdc_2 t}} (h^t{}^\trsp h^t - h^t{}^\trsp \Pi_{H^t} h^t) = \f 1 {\nvar^{\cdc_2 t}} (h^t{}^\trsp h^t - h^t{}^\trsp H^t (H^t{}^\trsp H^t)^+ H^t{}^\trsp h^t)$.
By induction hypothesis, because $f(z)^2$ is $\alpha$-controlled for some $\alpha <2$, 
\begin{align*}
\f 1 {\nvar^{\cdc_2 t}} h^t{}^\trsp h^t 
    = 
        \f 1 {\nvar^{\cdc_2 t}} \sum_{i=1}^{\nvar^{\cdc_2 t}} f(\gvar^{l_{01}t}, \ldots, \gvar^{l_{0 k_0}t})^2 
    \asto
        \EV[ f(z^{l_{01}}, \ldots, z^{l_{0 k_0}})^2: z \sim \Gaus(\mu^\cdc, K^\cdc)]
    =
        K^\cdc(g, g),
\end{align*}
where $z = (z^{l})_{\gvar^l \in \cdc}$.
Likewise, because both $f$ and $\{f^j\}_j$ are $\alpha$-controlled jointly for some $\alpha < 2$, by induction hypothesis,
\begin{align*}
\f 1 {\nvar^{\cdc_2 t}} h^t{}^\trsp h^{jt}
    &=
        \f 1 {\nvar^{\cdc_2 t}} \sum_{i=1}^{\nvar^{\cdc_2 t}} 
            f(\gvar^{l_{01}t}, \ldots, \gvar^{l_{0 k_0}t})
            f^j(\gvar^{l_{j1}t}, \ldots, \gvar^{l_{jk_j} t}) \\
    &\asto
        \EV[f(z^{l_{01}}, \ldots, z^{l_{0 k_0}})
            f^j(z^{l_{j1}}, \ldots, z^{l_{jk_j}})
            : z \sim \Gaus(\mu^\cdc, K^\cdc)]\\
    &=
        K^\cdc(g, g^j)\\
\f 1 {\nvar^{\cdc_2 t}} h^{j't}{}^\trsp h^{jt}
    &=
        \f 1 {\nvar^{\cdc_2 t}} \sum_{i=1}^{\nvar^{\cdc_2 t}} 
            f^{j'}(\gvar^{l_{j'1}t}, \ldots, \gvar^{l_{j k_{j'}}t})
            f^j(\gvar^{l_{j1}t}, \ldots, \gvar^{l_{j k_j} t}) \\
    &\asto
        \EV[f^{j'}(z^{l_{j'1}}, \ldots, z^{l_{j' k_{j'}}})
            f^j(z^{l_{j1}}, \ldots, z^{l_{jk_j}})
            : z \sim \Gaus(\mu^\cdc, K^\cdc)]\\
    &=
        K^\cdc(g^{j'}, g^j).
\end{align*}
Finally, by the rank convergence assumption
we get $(H^t {}^\trsp H^t)^+ \asto K^\cdc(G, G)^+$.
\end{claimproof}

\begin{claim}
$(H^t)^+ h^t \asto v^\infty \defeq K^\cdc(G, G)^+ K^\cdc(G, g)$.
\end{claim}

\begin{claimproof}
This is similar to the above.
$(H^t)^+ h^t = (H^t{}^\trsp H^t)^+ H^t{}^\trsp h^t\asto K^\cdc(G, G)^+ K^\cdc(G, g).$
\end{claimproof}

Let $L = \lno(g).$
Let $\phi: \R^{|\cdclt L|} \to \R$ be $\alpha$-controlled with $\alpha \in [1, 2)$, such that for coefficients $C, c > 0$, $|\phi(x)| \le e^{C\sum_i |x_i|^\alpha + c}$.
Since for every $q > 0$,
\begin{align*}
    \EV \left[\left. |\phi(g^t_i, \gvar_i^{\cdclt L t})|^q \right| \Aa^t\right]
        &=
            \EV\left[ |\phi(\mu^t_i + \sigma^t z, 
                    \gvar_i^{\cdclt L t}
                    )|^q: z \sim \Gaus(0, 1)
            \right]\\
        &\le
            \EV_z e^{Cq
                \left(
                    |\mu_i^t + \sigma^t z|^\alpha + \sum_{\hat g \in \cdclt L} |\hat g^t_i|^\alpha
                \right) + cq}\\
        &\le
            \EV_z e^{Cq \alpha
                \left(
                    |\mu_i^t|^\alpha + |\sigma^t z|^\alpha + \sum_{\hat g \in \cdclt L} |\hat g^t_i|^\alpha
                \right) + cq}\\
        &=
            e^{C q \alpha \left(
                    |\mu_i^t|^\alpha+ \sum_{\hat g \in \cdclt L} |\hat g^t_i|^\alpha
                \right) + cq}
            \EV_z e^{Cq\alpha (\sigma^t)^\alpha |z|^\alpha}\\
        &= 
            e^{C q \alpha \left(
                    |\mu_i^t|^\alpha+ \sum_{\hat g \in \cdclt L} |\hat g^t_i|^\alpha
                \right) + cq}
            \alpExp_\alpha^1(Cq\alpha (\sigma^t)^\alpha),
\end{align*}
we have
\begin{align*}
    \f 1 {\nvar^{\cdc t}} \sum_{i=1}^{\nvar^{\cdc t}}
        \EV \left[\left. |\phi(g^t_i, \gvar^{\cdclt{L}t}_i)|^q \right| \Aa^t\right]
    &\le
        \f 1 {\nvar^{\cdc t}} 
            \alpExp_\alpha^1(Cq\alpha (\sigma^t)^\alpha)
        \sum_{i=1}^{\nvar^{\cdc t}}
        e^{C q \alpha \left(
                    |\mu_i^t|^\alpha+ \sum_{\hat g \in \cdclt L} |\hat g^t_i|^\alpha
                \right) + cq}\\
    &\le
        \f 1 {\nvar^{\cdc t}} 
            \alpExp_\alpha^1(Cq\alpha (\sigma^t)^\alpha)
        \sum_{i=1}^{\nvar^{\cdc t}}
        e^{C' q \alpha \left(
                    \sum_j |v^t_j g^{jt}_i|^\alpha+ \sum_{\hat g \in \cdclt L} |\hat g^t_i|^\alpha
                \right) + cq}\\
    &\le
        \f 1 {\nvar^{\cdc t}} 
            \alpExp_\alpha^1(Cq\alpha (\sigma^t)^\alpha)
        \sum_{i=1}^{\nvar^{\cdc t}}
        e^{C' q \alpha \left(
                    \sum_j (|v^\infty_j|+1) |g^{jt}_i|^\alpha+ \sum_{\hat g \in \cdclt L} |\hat g^t_i|^\alpha
                \right) + cq}\\
    &\pushright{\text{for large enough $t$, almost surely}}
        \\
    &\defeq
        \f 1 {\nvar^{\cdc t}} 
            \alpExp_\alpha^1(Cq\alpha (\sigma^t)^\alpha)
        \sum_{i=1}^{\nvar^{\cdc t}}
            \psi(\gvar_i^{\cdclt L t})
        \numberthis
        \label{eqn:asUniformlyBoundedMoment}
\end{align*}
for some constant $C'$, where $\psi$ is a $\alpha$-controlled function.
Thus by our claims above, this converges as $\nvar^{\cdc t} \to \infty$ to
\begin{align*}
    \alpExp_\alpha^1(Cq\alpha (\sigma^\infty)^\alpha)
    \EV_{z \sim \Gaus(\mu^\cdc, K^\cdc)}
        \psi( z^{\cdclt L}),
\end{align*}
where $z^{\cdclt L} = (z^l)_{l \in \cdc, l < L}$.
Therefore it is also almost surely uniformly bounded in $t$, so that with $q = 2 + \rho$, we can apply \cref{thm:SLLN} to $X_{\nvar^{\cdc t}, i} = \phi(g^t_i, \gvar^{\cdclt L t}_i) - \EV\left[\left.\phi(g^t_i, \gvar^{\cdclt L t}_i) \right| \Aa^t\right]$ to conclude that
\begin{align*}
    \f 1 {\nvar^{\cdc t}} 
    \sum_{i=1}^{\nvar^{\cdc t}}
        \phi(g^t_i, \gvar^{\cdclt L t}_i)
         - \EV\left[\left.\phi(g^t_i, \gvar^{\cdclt L t}_i) \right| \Aa^t\right]
         &\asto 
            0.
\end{align*}

After applying the following claim and the induction hypothesis on 
$$\gvar^{\cdclt L t}_i \mapsto \EV\left[\phi\lp\sum_j g^{jt}_i v_j^\infty + \sigma^\infty z, \gvar^{\cdclt L t}_i\rp : z \sim \Gaus(0, 1)\right],$$
we get
\begin{align*}
    \f 1 {\nvar^{\cdc t}} 
    \sum_{i=1}^{\nvar^{\cdc t}}
        \phi(g^t_i, \gvar^{\cdclt L t}_i)
         &\asto 
            \EV\left[
                \phi\left(\sum_{j=1}^r \zeta^{\lno(g^{j})} v_j^\infty + \sigma^\infty z, \zeta^{\cdclt L}\right)
                :
                z \sim \Gaus(0, 1), \zeta \sim \Gaus(\mu^\cdc, K^\cdc)
            \right].
\end{align*}
This yields the desired theorem by noting that conditioned on $\zeta^{\cdclt L}$, $\sum_j \zeta^{\lno(g^{j})} v_j^\infty + \sigma^\infty z \disteq \zeta^L$ via \cref{prop:GaussianCondition}.

\begin{claim}
 $
    \f 1 {\nvar^{\cdc t}} 
    \sum_{i=1}^{\nvar^{\cdc t}}
    \EV\left[\left.\phi(g^t_i, \gvar^{\cdclt L t}_i) \right| \Aa^t\right] - \EV\left[\phi(\sum_j g^{jt}_i v_j^\infty + \sigma^\infty z, \gvar^{\cdclt L t}_i) : z \sim \Gaus(0, 1)\right] \asto 0$.
\end{claim}
    
\begin{claimproof}
From the claims above we have that almost surely, $g^t_i \disteq_{\Aa^t} \sum_j g^{jt}_i v_j^\infty + o(1)(\sum_j |g^{jt}_i|) + (1 + o(1))z,$ where $o_t(1)$ is a quantity that decreases to 0 with $t$ and doesn't depend on $i$.
Let $\Phi(x^L, x^{\cdclt L}; \sigma) \defeq \EV[\phi(x^L + \sigma z, x^{\cdclt L}): z \sim \Gaus(0, 1)] = \EV[\phi(z, x^{\cdclt L}): z \sim \Gaus(x^L, \sigma^2)].$
Then by \cref{lemma:gaussianDer}, 
$\Phi$ is differentiable in $x^L$ and $\pd_{x^L} \Phi(x^L, x^{\cdclt L}; \sigma)
= \sigma^{-1}\EV[z\phi(x^L + \sigma z, x^{\cdclt L}): z \sim \Gaus(0, 1)].$
Clearly, 
\begin{align*}
|\pd_{x^L} \Phi(x^L, x^{\cdclt L}; \sigma)|
    &\le
        \sigma^{-1}\EV[|z\phi(x^L + \sigma z, x^{\cdclt L})|: z \sim \Gaus(0, 1)]\\
    &\le
        \sigma^{-1}\EV[|z|e^{C\left(
                    |x^L + \sigma z|^\alpha
                    + \|x^{\cdclt L}\|^\alpha_\alpha
                \right) + c}: z \sim \Gaus(0, 1)]\\
    &\le
        \sigma^{-1}
        e^{C\alpha\left(
                    |x^L|^\alpha
                    + \|x^{\cdclt L}\|^\alpha_\alpha
                \right) + c}
        \EV[|z|e^{C \alpha\sigma^\alpha |z|^\alpha}: z \sim \Gaus(0, 1)]\\
    &\defeq
        \sigma^{-1}
        e^{C\alpha\left(
                    |x^L|^\alpha
                    + \|x^{\cdclt L}\|^\alpha_\alpha
                \right) + c}
        R(\sigma)\\
|\Phi(x^L, x^{\cdclt L}; \sigma)
- \Phi(x^L + \epsilon, x^{\cdclt L}; \sigma)|
    &=
        \left|
            \int_{x^L}^{x^L+\epsilon} \dd \xi \ 
                \pd_\xi \Phi(\xi, x^{\cdclt L}; \sigma)
        \right|\\
    &\le
        \sigma^{-1}R(\sigma)
        \int_{x^L}^{x^L+\epsilon} |\dd \xi| \ 
            e^{C\alpha\left(
                        |\xi|^\alpha
                        + \|x^{\cdclt L}\|^\alpha_\alpha
                    \right) + c}
            \\
    &\le
        \sigma^{-1}R(\sigma)
        \int_{0}^{\epsilon} |\dd \xi| \ 
            e^{C\alpha^2\left(
                        |x^L|^\alpha + |\xi|^\alpha
                        + \|x^{\cdclt L}\|^\alpha_\alpha
                    \right) + c}
            \\
    &\le
        \sigma^{-1}R(\sigma)
        |\epsilon| 
        e^{C\alpha^2\left(
            |x^L|^\alpha + |\epsilon|^\alpha
            + \|x^{\cdclt L}\|^\alpha_\alpha
        \right) + c}.
\end{align*}
Hence
\begin{align*}
    &\phantomeq
    \left|\f 1 {\nvar^{\cdc t}} 
    \sum_{i=1}^{\nvar^{\cdc t}}
        \EV\left[\left.\phi(g^t_i, \gvar^{\cdclt L t}_i) \right| \Aa^t\right] - \EV\left[\left.\phi(\sum_j g^{jt}_i v_j^\infty + \sigma^t z, \gvar^{\cdclt L t}_i) \right| z \sim \Gaus(0, 1)\right]
    \right|
    \\
    &\le
        \f 1 {\nvar^{\cdc t}} 
        \sum_{i=1}^{\nvar^{\cdc t}}
            (\sigma^t)^{-1} R(\sigma^t) o_t(1)(\sum_j |g^{jt}_i|)
            e^{C\alpha^2\left(
                \left|\sum_j g^{jt}_i v_j^\infty\right|^\alpha
                + o_t(1)\lp\sum_j |g^{jt}_i|\rp^\alpha
                + \|\gvar_i^{\cdclt L t}\|^\alpha_\alpha
            \right) + c}\\
    &\le
        (\sigma^t)^{-1} R(\sigma^t) o_t(1)
        \f 1 {\nvar^{\cdc t}} 
        \sum_{i=1}^{\nvar^{\cdc t}}
        \Psi(\gvar_i^{\cdclt L t})
\end{align*}
for some $\alpha$-controlled function $\Psi$.
By induction hypothesis, $\f 1 {\nvar^{\cdc t}} 
        \sum_{i=1}^{\nvar^{\cdc t}}
        \Psi(\gvar_i^{\cdclt L t})$ converges almost surely, so that the entire quantity decreases to 0 due to $o_t(1)$.
A similar argument shows that we can replace $\sigma^t$ with $\sigma^\infty$.
\end{claimproof}

\end{proof}

\gradIndep*

Note that we impose a more stringent condition here, compared to \cref{thm:notransposeLimit}, that $\fvar^l$ are polynomial bounded, because, as it will be apparent from the reasoning below, we need to reason about compositions of $\phi$ and $\fvar^l$; if we still allow $\fvar^l$ and $\phi$ to be $\alpha$-controlled in general, then their composition in general is not integrable against Gaussian measures.
\begin{proof}
\cref{thm:notransposeLimit} already show that this is true for all $\gvar^l$ in $\pi$.
Because we assume that $v^{it}$ are sampled independently from $x^{lt}$, this is also true up to line $L+L_A+\Lnabla$.
We induct on the line number starting from $\ell = L +L_A+\Lnabla+1$.

If line $\ell$ does not produce a new G-var, then there is nothing to prove.

If line $\ell$ is of type \ref{linetype:lincomb}, then the induction hypothesis is obviously true.

In the following, suppose line $\ell$ is of type \ref{linetype:GLinear}.

\noindent\textbf{Setup}
This line involves a transposed matrix, $\ell: \gvar^{\ell} := \Avar^{L+a} \hvar^l$, where $\Avar^{L+a}=(\Avar^a)^\trsp$ by line $L+a$ and $l > L$ (the argument for $\gvar^l$ instead of $\hvar^l$ is similar and simpler).
Let $\cdc = \cdc_1 = \cdc_1(\Avar^{L+a}) = \cdc(\gvar^\ell)$ and $\cdc_2 = \cdc_2(\Avar^{L+a}) = \cdc(\hvar^l).$
Conditioning on all G-vars that appeared before, we have constraints of the form $g^{it} = \Avar^{L+a,t} h^{it}$ for $i = 1, \ldots, r$, and $g'{}^{it} = \Avar^{at} h'{}^{it}$ for $i = 1, \ldots, s$.
Here $\{g^i\}_{i=1}^r$ and $\{g'{}^i\}_{i=1}^s$ are previous G-vars and $\{h^i\}_{i=1}^r$ and $\{h'{}^i\}_{i=1}^s$ are previous G- or H-vars.
Letting $G^t = [g^{1t}|\cdots|g^{rt}] \in \R^{\nvar^{\cdc_1 t} \times r}$ (where $g^{it}$ are treated as column vectors), and similarly for $H^t\in \R^{\nvar^{\cdc_2 t} \times r}, G'{}^t\in \R^{\nvar^{\cdc_2 t} \times s}, H'{}^t\in \R^{\nvar^{\cdc_1 t} \times s}$, we get the expressions
$G^t=\Avar^{L+a,t}H^t, G'{}^t=\Avar^{at} H'{}^t$.
We will abuse notation and sometimes use $G$ to also denote the corresponding collection of G-vars; likewise for $G', H, H'$.
By the construction of $\tilde \pi$, $g^{i} = \Avar^{L+a} h^{i}$ are lines that appear after $\pi$, and $g'{}^{i} = \Avar^{a} h'{}^{i}$ are lines that appear in $\pi$.
In addition, for each $i \in [r]$, $h_i$ is an H-var that appears after $\pi$.

Let $\Aa^t$ be the $\sigma$-algebra spanned by the values of all G-vars that appeared before line $\ell$ at time $t$.
By the conditioning trick, we have
\begin{align*}
    \gvar^{\ell t} 
        &\disteq_{\Aa^t}
            (E^t + \Pi_{H'{}^t}^\perp \tilde A^t \Pi_{H^t}^\perp) \hvar^{lt}
        &\\
        &\text{with}
            \\
    E^t
        &=
            G^t H^t{}^+
            + H'{}^t{}^{+\trsp} G'{}^t{}^\trsp
            - H'{}^t{}^{+\trsp} G'{}^t{}^{\trsp} H^t H^t{}^+
            \\
        &=
            G^t(H^t{}^\trsp H^t)^+ H^t{}^\trsp + H'{}^t(H'{}^t{}^\trsp H'{}^t)^+ G'{}^t{}^\trsp
            - H'{}^t(H'{}^t{}^{\trsp} H'{}^t)^+ G'{}^t{}^{\trsp} H^t(H^t{}^\trsp H^t)^+ H^t{}^\trsp
\end{align*}
where $\tilde A^t$ is sampled independently and identically as $\Avar^{L+a,t}$.
Note that
\begin{align*}
    \gvar^{\ell t} &\disteq_{\Aa^t} \mu^t + \sigma^t \Pi^\perp_{H'{}^t} y,\text{ with $y \sim \Gaus(0, I_{\nvar^{\cdc_1 t}})$}\\
    \mu^t
        &\defeq E^t \hvar^{lt}\\
    (\sigma^t)^2 
        &\defeq (\sigma^{at})^2\f{\nvar^{\cdc_2 t}}{\nvar^{\cdc_1 t}} \|\Pi_{H^t}^\perp \hvar^{lt}\|^2/\nvar^{\cdc_2 t}
\end{align*}
where, to recall, $(\sigma^{at})^2/\nvar_2(\Avar^{at}) = (\sigma^{at})^2/\nvar_1(\Avar^{L+a,t})$ is the sampling variance of each entry of $\Avar^{at}$ and $\Avar^{L+a,t}.$
For brevity, we use the following shorthands
\begin{align*}
    \Sigma^t
        &\defeq H^t {}^\trsp H^t / \nvar^{\cdc_2 t}
            \in \R^{r \times r}
            &
    \Sigma'{}^t
        &\defeq H'{}^t{}^\trsp H'{}^t / \nvar^{\cdc_1 t}
            \in \R^{s \times s}
            &
    \Upsilon^t
        &\defeq
            G'{}^t{}^{\trsp} H^t / \nvar^{\cdc_2 t}
            \in \R^{s \times r}
            \\
    \omega^t
        &\defeq
            H^t{}^\trsp \hvar^{lt} / \nvar^{\cdc_2 t}
            \in \R^{r}
            &
    \beta^t
        &\defeq
            G'{}^t{}^\trsp \hvar^{lt}/ \nvar^{\cdc_2 t}
            \in \R^{s}
\end{align*}
so that $$
\mu^t = 
    G^t \Sigma^{t}{}^+ \omega^t 
    + \f{\nvar^{\cdc_2 t}}{\nvar^{\cdc_1 t}} H'{}^t \Sigma'{}^t{}^+ \beta^t 
    - \f{\nvar^{\cdc_2 t}}{\nvar^{\cdc_1 t}} H'{}^t \Sigma'{}^t{}^+ \Upsilon^t \Sigma^t{}^+ \omega^t.$$
By induction hypothesis, $\Sigma^t, \Sigma'{}^t, \Upsilon^t, \omega^t, \beta^t$ all converge almost surely to corresponding limit values:
Let $\alpha = \alpha_{\cdc_2, \cdc_1} = \lim_{t \to \infty} \f{\nvar^{\cdc_2 t}}{\nvar^{\cdc_1 t}}$;
if $\lambda_i = \lno(h^i), \lambda_i' = \lno(h'{}^i)$, then, with $Z \sim \Gaus(\mu^\cdc, K^\cdc)$,
\begin{align*}
    \Sigma^t_{ij} 
        &\asto
            \Sigma^\infty_{ij} \defeq
            \EV \fvar^{\lambda_i}(Z)\fvar^{\lambda_j}(Z)
            = (\sigma^{k\infty})^{-2} \alpha^{-1} K^\cdc(g^i, g^j)
            \\
    \Sigma'{}^t_{ij}
        &\asto
            \Sigma'{}^\infty_{ij} \defeq
            \EV \fvar^{\lambda_i'}(Z)\fvar^{\lambda_j'}(Z)
            = (\sigma^{k\infty})^{-2} K^\cdc(g'{}^i, g'{}^j)
            \\
    \omega^t_i
        &\asto
            \omega^\infty_i \defeq
            \EV \fvar^{\lambda_i}(Z) \fvar^l(Z)
            = (\sigma^{k\infty})^{-2} \alpha^{-1} K^\cdc(g^i, \gvar^\ell)
            &
    \Upsilon^t_i
        &\asto
            0
            &
    \beta^t_i
        &\asto
            0.
\end{align*}
The last two limits go to 0 because $H^t$ and $\hvar^{lt}$ are odd in $v^1, \ldots v^\Lnabla$, which are sampled independently from $G'{}^t$ as remarked above.
Consequently, by our rank assumption, $\Sigma^t{}^+ \asto \Sigma^\infty{}^+, \Sigma'{}^t{}^+ \asto \Sigma'{}^\infty{}^+$.
\begin{claim}
$(\sigma^t)^2
    \asto
        (\sigma^\infty)^2
    \defeq
        K^\cdc(\gvar^\ell, \gvar^\ell) - K^\cdc(\gvar^\ell, G) K^\cdc(G, G)^+ K^\cdc(G, \gvar^\ell)$ with $t$.
\end{claim}
\begin{claimproof}
We have
\begin{align*}
    \f{(\sigma^t)^2}{(\sigma^{at})^2\f{\nvar^{\cdc_2 t}}{\nvar^{\cdc_1 t}}}
        &=
            \hvar^{lt}{}^\trsp \Pi_{H^t}^\perp \hvar^{lt}/\nvar^{\cdc_2 t}
            \\
        &=
            \|\hvar^{lt}\|^2/\nvar^{\cdc_2 t}
            -
            \hvar^{lt}{}^\trsp \Pi_{H^t} \hvar^{lt}/\nvar^{\cdc_2 t}
            \\
        &=
            \|\hvar^{lt}\|^2/\nvar^{\cdc_2 t}
            -
            (\hvar^{lt}{}^\trsp 
                H^t/\nvar^{\cdc_2 t})
                (H^t{}^\trsp H^t/\nvar^{\cdc_2 t})^+ 
                (H^t{}^\trsp \hvar^{lt}/\nvar^{\cdc_2 t})
            \\
        &=
            \|\hvar^{lt}\|^2/\nvar^{\cdc_2 t}
            -
            \omega^t{}^\trsp \Sigma^t{}^+ \omega^t
            .
\end{align*}
Now $
\|\hvar^{lt}\|^2/\nvar^{\cdc_2 t}
    =
        \f 1 {\nvar^{\cdc_2 t}} \sum_{i=1}^{\nvar^{\cdc_2 t}} (\hvar^{lt}_i)^2
    \asto
        \EV[ \fvar^l(Z)^2: Z \sim \Gaus(\mu^\cdc, K^\cdc)]
    =
        (\sigma^{k\infty})^{-2} \alpha^{-1} K^\cdc(\gvar^\ell, \gvar^\ell)
$ by induction hypothesis.
On the other hand, again by induction hypothesis and the rank assumption, the second term converges almost surely to $(\sigma^{k\infty})^{-2} \alpha^{-1} K^\cdc(\gvar^\ell, G) K^\cdc(G, G)^+ K^\cdc(G, \gvar^\ell).$
Combined with the simple fact that $(\sigma^{at})^2\f{\nvar^{\cdc_2 t}}{\nvar^{\cdc_1 t}} \to (\sigma^{\infty t})^2 \alpha$, we get the desired result.

\end{claimproof}

Let $v^t \defeq \Sigma^t{}^+ \omega^t$, for $t \in [1, \infty]$, so that, by our rank condition, $v^t \asto v^\infty = \Sigma^\infty{}^+ \omega^\infty,$ which we can check is equal to $K^{\cdc}(G, G)^+ K^{\cdc}(G, \gvar^\ell).$

\begin{claim}\label{claim:epsilon}
For some vectors $\varepsilon^t \in \R^{r}, \varepsilon'{}^t \in \R^{s}$ that go to 0 almost surely with $t$,
$\mu^t = E^t \hvar^{lt} = G^t (v^\infty + \varepsilon^t) + H'{}^t \varepsilon'{}^t.$
\end{claim}
\begin{claimproof}
Follows immediately from the fact derived above that $\Sigma^t{}^+ \asto \Sigma^\infty{}^+, \omega^t \asto \omega^\infty, \Upsilon^t \asto 0, \beta^t \asto 0$.
\end{claimproof}

\newcommand{\probA}{\mathsf{A}}
\newcommand{\probB}{\mathsf{B}}
\newcommand{\probC}{\mathsf{C}}
\noindent\textbf{Convergence almost surely.}
Let $\phi$ be a function with $|\phi(x)| \le C ( 1 + \|x\|^{2p}), p \in \N$; $\phi$ will be our test function.
With $ w\sim \Gaus(0, 1), Z \sim \Gaus(\mu^\cdc, K^\cdc),$
\begin{align*}
    &\phantomeq
        \left| \f 1 {\nvar^{\cdc t}} \sum_{i=1}^{\nvar^{\cdc t}}
            \phi(\gvar^{\ell t}_i, \gvar_i^{\cdclt \ell t})
            -
            \EV_Z \phi(Z) ]
            \right|
        \le \probA + \probB + \probC
            \\
    \text{with}
        &
            \\
    \probA
        &\defeq
            \left| \f 1 {\nvar^{\cdc t}} \sum_{i=1}^{\nvar^{\cdc t}}
                \EV_{w} 
                    \phi\lp
                        \sum_{j=1}^r v^\infty_j g^{jt}_i + \sigma^\infty w, \gvar_i^{\cdclt \ell t}\rp
                -
                \EV_Z \phi(Z)
                \right|
            \\
    \probB
        &\defeq
            \left| 
                \f 1 {\nvar^{\cdc t}} \sum_{i=1}^{\nvar^{\cdc t}}
                    \EV_{w} \phi\left(\mu^t_i + \sigma^t \sqrt{(\Pi^\perp_{H'{}^t})_{ii}} w,
                        \gvar_i^{\cdclt \ell t}\right)
                    -
                    \EV_{w}
                        \phi\lp
                            \sum_{j=1}^r v^\infty_j g^{jt}_i + \sigma^\infty w, \gvar_i^{\cdclt \ell t}\rp
                \right|
            \\
    \probC
        &\defeq
                \left|
                    \f 1 {\nvar^{\cdc t}} \sum_{i=1}^{\nvar^{\cdc t}}
                    \phi\left(\gvar^{\ell t}_i, \gvar_i^{\cdclt \ell t}\right)
                    - 
                    \EV_{w}
                        \phi\left(\mu^t_i + \sigma^t \sqrt{(\Pi^\perp_{H'{}^t})_{ii}} w,
                            \gvar_i^{\cdclt \ell t}\right)
                \right|
\end{align*}
We shall show that each of $\probA, \probB, \probC$ goes to 0 almost surely as $t \to \infty$.

\begin{claim}\label{claim:Aasto0}
$\probA \asto 0$.
\end{claim}
\begin{claimproof}
Note that if $\R^{\cdc} \ni \zeta \sim \Gaus(\mu^\cdc, K^\cdc),$
\begin{align*}
    &\phantomeq
        \sum_{j=1}^r v^\infty_j \zeta^{g^j} + \sigma^\infty w 
        \\
    &=
        (\zeta^{G})^\trsp K^\cdc(G, G)^+ K^\cdc(G, \gvar^\ell)
        + (K^\cdc(\gvar^\ell, \gvar^\ell) - K^\cdc(\gvar^\ell, G) K^\cdc(G, G)^+ K^\cdc(G, \gvar^\ell)) w 
        \\
    &\disteq_{\zeta^{\cdclt \ell}}
        \zeta^{\gvar^\ell}.
\end{align*}
Since $\EV_w \phi\lp \sum_{j=1}^r v^\infty_j g^{jt}_i + \sigma^\infty w, \gvar_i^{\cdclt \ell t}\rp$ is purely a polynomially-bounded function of $\gvar_i^{\cdclt \ell t}$,
this claim is given by the induction hypothesis.
\end{claimproof}

\begin{claim}\label{claim:Casto0}
$\probC \asto 0$.
\end{claim}
\begin{claimproof}
Fix the values of $\gvar^{\cdclt \ell t}$.
For each $i \in [\nvar^{\cdc t}]$, let $\phi_i^t(x) \defeq \phi(\mu^t_i + \sigma^t x, \gvar^{\cdclt \ell t}_i)$, and $\tilde \phi_i^t(x) \defeq \phi_i^t(x) - \EV_{x' \sim \Gaus(0, (\Pi_{H'{}^t}^\perp)_{ii})} \phi_i^t(x')$.
By \cref{thm:controlHighMoments} applied to $\Pi_{H'{}^t}^\perp$ and $\{\phi_i^t\}_{i=1}^{\nvar^{\cdc t}}$, we get, for any $\rho \ge 6$ and any $q > 1$,
\begin{align*}
    \EV_{z \sim \Gaus(0, \Pi_{H'{}^t}^\perp)}
    \lp \f 1 {\nvar^{\cdc t}} \sum_{i=1}^{\nvar^{\cdc t}} \tilde \phi_i^t(z_i) \rp^{2\rho}
    &\le
        O\lp
            (\nvar^{\cdc t})^{-1.5 + 1/q}
            \sqrt[q]{
                \f 1 {\nvar^{\cdc t}}
                \sum_{i=1}^{\nvar^{\cdc t}}
                    \EV_{z_i\sim \Gaus(0, (\Pi_{H'{}^t}^\perp)_{ii})}\left[\tilde \phi_i^t(z_i)^{2\rho q}\right]
            }
        \rp
\end{align*}
where the constant hidden in $O(-)$ is independent of $q$, $t$, the functions $\phi_i^t$, and $\Pi_{H'{}^t}$.
We first show that the sum $\f 1 {\nvar^{\cdc t}} \sum_{i=1}^{\nvar^{\cdc t}}
                    \EV_{z_i}\left[\tilde \phi_i^t(z_i)^{2\rho q}\right]$
    is uniformly bounded almost surely in $t$ over the probability of $\{\Aa^t\}_{t \ge 1}$, for any $q > 1$.
Indeed, with $z \sim \Gaus(0, \Pi_{H'{}^t}^\perp),$
\begin{align*}
        \f 1 {\nvar^{\cdc t}} \sum_{i=1}^{\nvar^{\cdc t}}
                    \EV_{z_i}\left[\tilde \phi_i^t(z_i)^{2\rho q}\right]
    &
    \le
        \f 1 {\nvar^{\cdc t}} 2^{2\rho q-1} \sum_{i=1}^{\nvar^{\cdc t}}
                    \EV_{z_i}\left[\phi_i^t(z_i)^{2\rho q} + \lp \EV_{z'_i} \phi_i^t(z'_i) \rp^{2\rho q} \right]
    \le
        \f 1 {\nvar^{\cdc t}} 2^{2\rho q} \sum_{i=1}^{\nvar^{\cdc t}}
                    \EV_{z_i}\left[\phi_i^t(z_i)^{2\rho q} \right]
        \\
    &=
        \f 1 {\nvar^{\cdc t}} 2^{2\rho q} \sum_{i=1}^{\nvar^{\cdc t}}
                    \EV_{z_i}\left[\phi(\mu^t_i + \sigma^t z_i, \gvar^{\cdclt \ell t}_i)^{2\rho q} \right]
        \\
    &\le
        \f 1 {\nvar^{\cdc t}} C' \sum_{i=1}^{\nvar^{\cdc t}}
                    \EV_{z_i}\left[
                        |\mu^t_i|^{4 \rho p q} + |\sigma^t z_i|^{4\rho p q} + \|\gvar^{\cdclt \ell t}_i\|^{4\rho p q}
                        \right]
        \\
    &\overset{\substack{t\to\infty\\ \mathrm{a.s.}}}\le
        \f 1 {\nvar^{\cdc t}} C'' \sum_{i=1}^{\nvar^{\cdc t}}
                    \left[
                        |\mu^t_i|^{4 \rho p q} + (2|\sigma^\infty|+1)^{4\rho p q} \EV_{z_i}|z_i|^{4\rho p q} + \|\gvar^{\cdclt \ell t}_i\|^{4\rho p q}
                        \right]
\end{align*}
for some constants $C', C'' > 0$,
where this last inequality holds for large enough $t$, almost surely.
By induction hypothesis and the fact that $(\Pi_{H'{}^t}^\perp)_{ii} \in [0, 1]$ for all $i$,
\begin{align*}
    \f 1 {\nvar^{\cdc t}} \sum_{i=1}^{\nvar^{\cdc t}}
        \left[
            (2|\sigma^\infty|+1)^{4\rho p q} \EV_{z_i}|z_i|^{4\rho p q} + \|\gvar^{\cdclt \ell t}_i\|^{4\rho p q}
            \right]
\end{align*}
is uniformly bounded in $t$, almost surely.
Thus it remains to bound
\begin{align*}
    \f 1 {\nvar^{\cdc t}} \sum_{i=1}^{\nvar^{\cdc t}} |\mu^t_i|^{4 \rho p q}
        &=
            \f 1 {\nvar^{\cdc t}} \sum_{i=1}^{\nvar^{\cdc t}}
                \left(\sum_{j=1}^r g^{jt}_i (v^\infty_j + \varepsilon^t_j) + h'{}^{jt}_i \varepsilon'_j{}^t\right)^{4 \rho p q}
            \\
        &\pushright{\text{where $\varepsilon^t, \varepsilon'{}^t \asto 0$, by \cref{claim:epsilon}}}
            \\
        &\le
            \f 1 {\nvar^{\cdc t}} C''' \sum_{i=1}^{\nvar^{\cdc t}}
                \left(\sum_{j=1}^r g^{jt}_i v^\infty_j\right)^{4 \rho p q}
                + \lp  \sum_{j=1}^r g^{jt}_i \varepsilon^t_j + h'{}^{jt}_i \varepsilon'_j{}^t\rp^{4\rho p q}
            \\
        &\overset{\substack{t\to\infty\\ \mathrm{a.s.}}}\le
            \f 1 {\nvar^{\cdc t}} C''' \sum_{i=1}^{\nvar^{\cdc t}}
                \left(\sum_{j=1}^r g^{jt}_i v^\infty_j\right)^{4 \rho p q}
                + \lp  \sum_{j=1}^r |g^{jt}_i| + |h'{}^{jt}_i|\rp^{4\rho p q}
\end{align*}
for some constant $C'''$,
where the last inequality holds for large enough $t$, almost surely.
Because each $h'{}^{jt}_i$ is a polynomially bounded function of $\gvar^{\cdclt \ell t}_i$, each summand of the RHS is as well \footnote{This is the only place where we need the assumption that all $\fvar^l$ are polynomially bounded. Otherwise, their composition might not ne integrable against the Gaussian measure.}.
So by induction hypothesis, this converges to a finite value, and hence is uniformly bounded in $t$, almost surely, as desired.

Thus, almost surely, $\EV_{z \sim \Gaus(0, \Pi_{H'{}^t}^\perp)}
    \lp \f 1 {\nvar^{\cdc t}} \sum_{i=1}^{\nvar^{\cdc t}} \tilde \phi_i^t(z_i) \rp^{2\rho } \le c (\nvar^{\cdc t})^{-1.25}$ for some $c$, by choosing $q$ large enough.
By \cref{lemma:momentBoundASConvergence} and the fact that $\nvar^{\cdc t}$ strictly increases with $t$, we have
\begin{align*}
    \f 1 {\nvar^{\cdc t}} \sum_{i=1}^{\nvar^{\cdc t}} \tilde \phi_i^t(z_i) =
    \f 1 {\nvar^{\cdc t}} \sum_{i=1}^{\nvar^{\cdc t}}
                    \psi\left(\gvar^{\ell t}_i, \gvar_i^{\cdclt \ell t}\right)
                    - 
                    \EV_{w}
                        \psi\left(\mu^t_i + \sigma^t \sqrt{(\Pi^\perp_{H'{}^t})_{ii}} w,
                            \gvar_i^{\cdclt \ell t}\right)
    &\asto 0,
    \text{ where }w \sim \Gaus(0, 1).
\end{align*}

\end{claimproof}

\begin{claim}
$\probB \asto 0$.
\end{claim}
\begin{claimproof}
We apply a similar argument as the one in the proof of \cref{thm:notransposeLimit} that leverages the smoothness of Gaussian average over $\tilde A$.
The major difference here is that we have to deal with the varying variances $(\Pi_{H'{}^t}^\perp)_{ii}$ for each $t$, but this can be done by using the fact that $\rank H'{}^t = O(1)$.

Define $\Phi(x^\ell, x^{\cdclt \ell}; \sigma) \defeq \EV[\phi(x^\ell + \sigma w, x^{\cdclt \ell}): w \sim \Gaus(0, 1)].$
Then by \cref{lemma:gaussianDer}, 
$\Phi$ is differentiable in $x^\ell$ and $\sigma^2$ with $\pd_{x^\ell} \Phi(x^\ell, x^{\cdclt \ell}; \sigma)
= \sigma^{-1}\EV[w\phi(x^\ell + \sigma w, x^{\cdclt \ell}): w \sim \Gaus(0, 1)]$ and $\pd_{\sigma^2} \Phi(x^\ell, x^{\cdclt \ell}; \sigma) = \f 1 2 \sigma^{-2} \EV_{w \sim \Gaus(0,1 )} \phi(x^\ell + \sigma w, x^{\cdclt \ell})(w^2 - 1).$
Clearly,
\begin{align*}
    |\pd_{x^\ell} \Phi(x^\ell, x^{\cdclt \ell}; \sigma)|
        &\le
            \inv \sigma
            C\EV[|w|(1 + |x^\ell| + |\sigma w| + \|x^{\cdclt \ell}\|)^{2p}: w \sim \Gaus(0, 1)]
            \\
        &\le
            \inv \sigma 4^{2p-1}
            C\EV[|w|(1 + |x^\ell|^{2p} + |\sigma w|^{2p} + \|x^{\cdclt \ell}\|^{2p}): w \sim \Gaus(0, 1)]
            \\
        &\le
            \inv \sigma
            C'
            (1 + |x^\ell|^{2p} + \|x^{\cdclt \ell}\|^{2p} + \sigma^{2p})
            \\
    |\pd_{\sigma^2} \Phi(x^\ell, x^{\cdclt \ell}; \sigma)|
        &\le
            \f 1 2 \sigma^{-2} \EV_{w \sim \Gaus(0,1 )}
                |\phi(x^\ell + \sigma w, x^{\cdclt \ell})| |w^2 - 1|
            \\
        &\le
            D \sigma^{-2} \EV_{w \sim \Gaus(0,1 )}
                (1 + |x^\ell|^{2p} + |\sigma w|^{2p} + \|x^{\cdclt \ell}\|^{2p}) |w^2 - 1|
            \\
        &\le
            D' \sigma^{-2} 
            (1 + |x^\ell|^{2p} + \|x^{\cdclt \ell}\|^{2p} + \sigma^{2p})
            \\
    \implies
    \|\nabla_{x^\ell, \sigma^2}\Phi(x^\ell, x^{\cdclt \ell}; \sigma)\|
        &=
            \sqrt{|\pd_{x^\ell} \Phi(x^\ell, x^{\cdclt \ell}; \sigma)|^2
                + |\pd_{\sigma^2} \Phi(x^\ell, x^{\cdclt \ell}; \sigma)|^2}
            \\
        &\le
            D'' (1 + \sigma^{-2})
            (1 + |x^\ell|^{2p} + \|x^{\cdclt \ell}\|^{2p} + \sigma^{2p})
\end{align*}
for some constant $C', D, D', D''$ depending only on $p$ and $C$.
Thus
\begin{align*}
    &\phantomeq
        |\Phi(x^\ell, x^{\cdclt \ell}; \sigma) - \Phi(x^\ell + \vartheta, x^{\cdclt \ell}; \sqrt{\sigma^2 + \varsigma^2})|
            \\
        &\le
            (|\vartheta| + \varsigma^2)\int_0^{1} \dd t\ \|\nabla_{x^\ell, \sigma^2}\Phi(x^\ell+\vartheta t, x^{\cdclt \ell}; \sqrt{\sigma^2 + \varsigma^2 t})\|
            \\
        &\le
            D''(|\vartheta| + \varsigma^2)
            \int_0^1 \dd t\
                \lp 1 + \f{1}{\sigma^2 + \varsigma^2 t}\rp
                (1 + |x^\ell + \vartheta t|^{2p} + \|x^{\cdclt \ell}\|^{2p} + (\sigma^2 + \varsigma^2 t)^{p})
            \\
        &\le
            \tilde C (|\vartheta| + \varsigma^2)
            \int_0^1 \dd t\
                \lp 1 + \f{1}{\sigma^2}\rp
                (1 + |x^\ell|^{2p} + \vartheta^{2p} t^{2p} + \|x^{\cdclt \ell}\|^{2p} + \sigma^{2p} + \varsigma^{2p} t^{p})
            \\
        &\le
            \tilde C' (|\vartheta| + \varsigma^2)
                \lp 1 + \f{1}{\sigma^2}\rp
                (1 + |x^\ell|^{2p} + \vartheta^{2p} + \|x^{\cdclt \ell}\|^{2p} + \sigma^{2p} + \varsigma^{2p})
\end{align*}
for some constants $\tilde C, \tilde C'$ depending only on $p$ and $C$.
Partition $[\nvar^{\cdc t}] = U \sqcup V$ where $U \defeq \{i: (\Pi_{H'{}^t}^\perp)_{ii} < 1/2\}$ and $V$ is its complement.
Note that $|U| \le 2\rank H'{}^t \le 2s$.
So
\begin{align*}
    &\phantomeq
    \left|
        \f 1 {\nvar^{\cdc t}} \sum_{i=1}^{\nvar^{\cdc t}}
        \EV_{w} \phi\left(\mu^t_i + \sigma^t \sqrt{(\Pi^\perp_{H'{}^t})_{ii}} w,
                        \gvar_i^{\cdclt \ell t}\right)
        -
        \EV_{w}
            \phi\lp
                G^t_{i:}\Sigma^\infty{}^+ \omega^\infty + \sigma^\infty w, \gvar_i^{\cdclt \ell t}\rp
    \right|
        \\
    &=
         \f 1 {\nvar^{\cdc t}} \sum_{i=1}^{\nvar^{\cdc t}}
         \left|
            \Phi\left(\mu^t_i,
                        \gvar_i^{\cdclt \ell t};
                    \sigma^t \sqrt{(\Pi^\perp_{H'{}^t})_{ii}} \right)
            -
            \Phi\lp
                G^t_{i:}\Sigma^\infty{}^+ \omega^\infty, \gvar_i^{\cdclt \ell t}; \sigma^\infty\rp
        \right|
        \\
    &\le
        \f 1 {\nvar^{\cdc t}} 
        \bigg( \sum_{i\in U}
            \left|
            \Phi\left(\mu^t_i,
                        \gvar_i^{\cdclt \ell t};
                    \sigma^t \sqrt{(\Pi^\perp_{H'{}^t})_{ii}} \right)
            \right|
            +
            \left|
            \Phi\lp
                G^t_{i:}\Sigma^\infty{}^+ \omega^\infty, \gvar_i^{\cdclt \ell t}; \sigma^\infty\rp
            \right|
            \numberthis\label{eqn:Usum}
            \\
    &\qquad
        +
            \sum_{i \in V}
            \left|
            \Phi\left(\mu^t_i,
                        \gvar_i^{\cdclt \ell t};
                    \sigma^t \sqrt{(\Pi^\perp_{H'{}^t})_{ii}} \right)
            -
            \Phi\lp
                G^t_{i:}\Sigma^\infty{}^+ \omega^\infty, \gvar_i^{\cdclt \ell t}; \sigma^\infty\rp
            \right|
            \bigg)
            .
            \numberthis\label{eqn:Vsum}
\end{align*}
The first sum \eqref{eqn:Usum} converges almost surely to 0:
\begin{align*}
    &\phantomeq
        \f 1 {\nvar^{\cdc t}} 
        \sum_{i\in U}
            \left|
            \Phi\left(\mu^t_i,
                        \gvar_i^{\cdclt \ell t};
                    \sigma^t \sqrt{(\Pi^\perp_{H'{}^t})_{ii}} \right)
            \right|
            +
            \left|
            \Phi\lp
                G^t_{i:}\Sigma^\infty{}^+ \omega^\infty, \gvar_i^{\cdclt \ell t}; \sigma^\infty\rp
            \right|
            \\
    &\le
        \f {2s} {\nvar^{\cdc t}}
        \max_{i \in [\nvar^{\cdc t}]} 
            \left|
            \Phi\left(\mu^t_i,
                        \gvar_i^{\cdclt \ell t};
                    \sigma^t \sqrt{(\Pi^\perp_{H'{}^t})_{ii}} \right)
            \right|
            +
            \left|
            \Phi\lp
                G^t_{i:}\Sigma^\infty{}^+ \omega^\infty, \gvar_i^{\cdclt \ell t}; \sigma^\infty\rp
            \right|
            \\
    &\le
        \f {2s} {(\nvar^{\cdc t})^{1-1/N}}
        \sqrt[N]{
            \f 1 {\nvar^{\cdc t}} 
            \sum_{i \in [\nvar^{\cdc t}]} 
            \left|
            \Phi\left(\mu^t_i,
                        \gvar_i^{\cdclt \ell t};
                    \sigma^t \sqrt{(\Pi^\perp_{H'{}^t})_{ii}} \right)
            \right|^N}
            \\
    &\qquad
        +
        \f {2s} {(\nvar^{\cdc t})^{1-1/N}}
        \sqrt[N]{
            \f 1 {\nvar^{\cdc t}}
            \sum_{i \in [\nvar^{\cdc t}]} 
            \left|
            \Phi\lp
                G^t_{i:}\Sigma^\infty{}^+ \omega^\infty, \gvar_i^{\cdclt \ell t}; \sigma^\infty\rp
            \right|^N
            }
        \numberthis\label{eqn:Ubound}
\end{align*}
for any $N > 0$.
Now $\f 1 {\nvar^{\cdc t}}
            \sum_{i \in [\nvar^{\cdc t}]} 
            \left|
            \Phi\lp
                G^t_{i:}\Sigma^\infty{}^+ \omega^\infty, \gvar_i^{\cdclt \ell t}; \sigma^\infty\rp
            \right|^N $ converges a.s. to a finite value, by \cref{claim:Aasto0}, so for large $N$,  the second term in \cref{eqn:Ubound} converges to 0.
Similarly, $\f 1 {\nvar^{\cdc t}} 
            \sum_{i \in [\nvar^{\cdc t}]} 
            \left|
            \Phi\left(\mu^t_i,
                        \gvar_i^{\cdclt \ell t};
                    \sigma^t \sqrt{(\Pi^\perp_{H'{}^t})_{ii}} \right)
            \right|^N$ converges a.s. to a finite value, as in the proof of \cref{claim:Casto0}, so for large $N$, the first term of \cref{eqn:Ubound} and thus \cref{eqn:Ubound} itself go to 0 almost surely.

We proceed with the second sum \eqref{eqn:Vsum}:
\begin{align*}
    &\phantomeq
        \f 1 {\nvar^{\cdc t}} \sum_{i\in V}
            \left|
            \Phi\left(\mu^t_i,
                        \gvar_i^{\cdclt \ell t};
                    \sigma^t \sqrt{(\Pi^\perp_{H'{}^t})_{ii}} \right)
            -
            \Phi\lp
                G^t_{i:}\Sigma^\infty{}^+ \omega^\infty, \gvar_i^{\cdclt \ell t}; \sigma^\infty\rp
            \right|
            \\
    &\le
        \tilde C'
        \lp 1 + \f{1}{\min(\sigma^\infty, \sigma^t/\sqrt 2)^2}\rp
        \f 1 {\nvar^{\cdc t}} \sum_{i \in V}
                (\varrho_i + |(\sigma^\infty)^2 - (\sigma^t)^2(\Pi^\perp_{H'{}^t})_{ii}|)
                (1 + |\mu^t_i|^{2p} 
                    + \varrho_i^{2p}
                    + \|\gvar^{\cdclt \ell}_i\|^{2p} + \max(\sigma^\infty, \sigma^t)^{2p})
\end{align*}
where $\varrho_i \defeq G^t_{i:} \varepsilon^t + H'{}^t_{i:} \varepsilon'{}^t$ with 
$\varepsilon^t, \varepsilon'{}^t = o(1)$ a.s. coming from \cref{claim:epsilon}.
Write $Y_i \defeq (1 + |\mu^t_i|^{2p} 
                    + \varrho_i^{2p}
                    + \|\gvar^{\cdclt \ell}_i\|^{2p} + \max(\sigma^\infty, \sigma^t)^{2p})$.
Since $\lp 1 + \f{1}{\min(\sigma^\infty, \sigma^t/\sqrt 2)^2}\rp$ is obviously uniformly bounded in $t$,
via Cauchy-Schwarz, the sum above is bounded by a constant multiple of
\begin{align*}
    \sqrt{
        \f 1 {\nvar^{\cdc t}} \sum_{i \in V}
                (\varrho_i + |(\sigma^\infty)^2 - (\sigma^t)^2(\Pi^\perp_{H'{}^t})_{ii}|)^2
    }
    \sqrt{
        \f 1 {\nvar^{\cdc t}} \sum_{i \in V}
                Y_i^2
    }.
\end{align*}
Using similar techniques as before, by applying induction hypothesis, $\sqrt{
        \f 1 {\nvar^{\cdc t}} \sum_{i \in V}
                Y_i^2
    }$ can be shown to be uniformly bounded in $t$, almost surely.
All it remains to show is that the first term in the product above converges to 0 a.s..
Now
\begin{align*}
    &\phantomeq
    \sqrt{
        \f 1 {\nvar^{\cdc t}} \sum_{i \in V}
                (\varrho_i + |(\sigma^\infty)^2 - (\sigma^t)^2(\Pi^\perp_{H'{}^t})_{ii}|)^2
    }
        \\
    &\le
        \sqrt{
            \f 1 {\nvar^{\cdc t}} \sum_{i \in V}
                    \varrho_i^2
        }
        +
        \sqrt{
            \f 1 {\nvar^{\cdc t}} \sum_{i \in V}
                    ((\sigma^\infty)^2 - (\sigma^t)^2)^2
        }
        +
        \sqrt{
            \f 1 {\nvar^{\cdc t}} \sum_{i \in V}
                    (\sigma^t)^4(1 - (\Pi^\perp_{H'{}^t})_{ii})^2
        }
        \\
    &\le
        \sqrt{
            \f 1 {\nvar^{\cdc t}} \sum_{i \in V}
                    \varrho_i^2
        }
        +
        |(\sigma^\infty)^2 - (\sigma^t)^2|
        +
        (\sigma^t)^2
        \sqrt{
            \f 1 {\nvar^{\cdc t}} \sum_{i \in V}
                    1 - (\Pi^\perp_{H'{}^t})_{ii}
        }
        \\
    &
        \pushright{\text{since $1 - (\Pi^\perp_{H'{}^t})_{ii} \in [0, 1/2]$}}
        \\
    &\le
        \sqrt{
            \f 1 {\nvar^{\cdc t}} \sum_{i \in V}
                    (G^t_{i:} \varepsilon^t)^2
        }
        +
        \sqrt{
            \f 1 {\nvar^{\cdc t}} \sum_{i \in V}
                    (H'{}^t_{i:} \varepsilon'{}^t)^2
        }
        +
        |(\sigma^\infty)^2 - (\sigma^t)^2|
        +
        (\sigma^t)^2
        \sqrt{
            \f {2 \rank H'{}^t} {\nvar^{\cdc t}}
        }
        \\
    &\le
        \|\varepsilon^t\|\sqrt{
            \f 1 {\nvar^{\cdc t}} \sum_{i \in V}
                    \|G^t_{i:}\|^2
        }
        +
        \|\varepsilon'{}^t\|
        \sqrt{
            \f 1 {\nvar^{\cdc t}} \sum_{i \in V}
                    \|H'{}^t_{i:}\|^2
        }
        +
        |(\sigma^\infty)^2 - (\sigma^t)^2|
        +
        (\sigma^t)^2
        \sqrt{
            \f {2 \rank H'{}^t} {\nvar^{\cdc t}}
        }
        .
\end{align*}
By induction hypothesis, $\sqrt{
            \f 1 {\nvar^{\cdc t}} \sum_{i \in V}
                    \|G^t_{i:}\|^2
        }$ converges and so is also uniformly bounded in $t$, almost surely.
Then, because $\|\varepsilon^t\| \asto 0$, the first term converges to 0 a.s..
Likewise for the second term.
The final two terms also obviously converge to 0 almost surely with $t$.
This completes the proof of our claim.

\end{claimproof}

\end{proof}

We are finally ready to prove \cref{thm:generalTensorP}, but a few lemmas first would help our effort significantly.

\begin{lemma}\label{lemma:steinMultivar}
Let $X = (X_1, \ldots, X_n)$ be a multivariate Gaussian with 0 mean and nondegenerate covariance.
For any $L_2(X)$-integrable function $f$,
\begin{align*}
    \EV X_i f(X) 
        &=
            \sum_{j=1}^n \f{\Cov(X_i, X_j)}{\Var(X_j | X_{\setminus j})} \EV (X_j - \EV[X_j | X_{\setminus j}]) f(X)\\
        &=
            \sum_{j=1}^n
            \f{\Cov(X_i, X_j) \EV (X_j - \Cov(X_j, X_{\setminus j}) \Cov(X_{\setminus j}, X_\setminus j)^{+} X_{\setminus j}) f(X)}
            {\Var(X_j) - \Cov(X_j, X_{\setminus j})\Cov(X_{\setminus j}, X_{\setminus J})^+ \Cov(X_{\setminus j}, X_j)}
            \\
\end{align*}

\end{lemma}
\begin{proof}
By a density argument, it suffices to consider only the case when $f$ is $C^\infty$.
Then by Stein's lemma,
\begin{align*}
    \EV X_i f(X)
        &=
            \sum_{j=1}^n \Cov(X_i, X_j) \EV \pd_j f(X)
            .
\end{align*}
By Stein's lemma again,
\begin{align*}
    \EV \pd_j f(X)
        &=
            \EV_{X_{\setminus j}} \EV_{X_j|X_{\setminus j}} \pd_j f(X)
            \\
        &=
            \EV_{X_{\setminus j}} \EV_{X_j|X_{\setminus j}} 
                \f{(X_j - \Cov(X_j, X_{\setminus j}) \Cov(X_{\setminus j}, X_\setminus j)^{+} X_{\setminus j}) f(X)}
                {\Var(X_j) - \Cov(X_j, X_{\setminus j})\Cov(X_{\setminus j}, X_{\setminus J})^+ \Cov(X_{\setminus j}, X_j)}
                \\
        &=
            \EV_{X}
            \f{(X_j - \Cov(X_j, X_{\setminus j}) \Cov(X_{\setminus j}, X_\setminus j)^{+} X_{\setminus j}) f(X)}
                {\Var(X_j) - \Cov(X_j, X_{\setminus j})\Cov(X_{\setminus j}, X_{\setminus J})^+ \Cov(X_{\setminus j}, X_j)}.
\end{align*}
\end{proof}

\begin{defn}\label{defn:varphih} 
Let $\gvar^l \defeq A \hvar^m$ be a line in $\pi$ and let $g^i = A' h^i, i = 1, \ldots, s$ be all previous \ref{linetype:GLinear} lines that involve $A'.$
Here $\Avar'$ is the A-var such that $A' := A^\trsp$ if $A$ is an input A-var, or $A := A'{}^\trsp$ if $A$ is a transposed var.
Then, by the construction of $\check \pi$, $\varphi(\gvar^l) \defeq \varphig(\gvar^l) + \sum_{j=1}^s \avar_j \varphi(h^j)$ for some coefficients $\avar$ defined in \cref{defn:detransposition}.
Define $\fvar^{\varphih(\gvar^l)} = \sum_{j=1}^s \avar_j \fvar^{\varphi(h^j)}$,
so $\fvar^{\varphi(\gvar^l)} = \fvar^{\varphig(\gvar^l)} + \fvar^{\varphih(\gvar^l)}.$
\end{defn}

Let $\gvar^l := A \hvar^m$ be a \ref{linetype:GLinear} line in $\pi$ and set $\check \cdc = \cdc(\varphi(\gvar^l)).$
Consider the Hilbert space $L_2(Z)$ of $L_2$ functions of $Z \sim \Gaus(\mu^{\check \cdc}, K^{\check \cdc})$ (equivalently, square-integrable random variables in the $\sigma$-algebra generated by $Z$).
Write $\la f, g \ra = \EV f(Z) g(Z)$ for its inner product.
\begin{lemma}\label{lemma:Hpart}
Fix a line number $\lambda \ge l$.
Let $g^i = A' h^i, i = 1, \ldots, s$ be all \ref{linetype:GLinear} lines strictly before line $\lambda$ that involve $\Avar^k{}^\trsp.$
Define $C \in \R^{s \times s}, v \in \R^s$ by
\begin{align*}
    C_{ij}
        &=
            \la \fvar^{\varphi(h^i)}, \fvar^{\varphi(h^j)} \ra
            \\
    v_i
        &=
            \la \fvar^{\varphig(g^i)}, \fvar^{\varphi(\hvar^{m})}\ra.
\end{align*}
Then for all $i \in [s]$,
\begin{align*}
    v_i &= 
        \inv \alpha \la \fvar^{\varphi(h^i)}, \fvar^{\varphih(\gvar^l)}\ra
        \\
    \fvar^{\varphih(\gvar^l)} 
    &= \alpha (\fvar^{\varphi(h^i)})_{i=1}^s C^+ v
    = \alpha \sum_{i, j \in [s]}
            (C^+)_{ij}
            v_j
            \fvar^{\varphi(h^i)}
\end{align*}
where $\alpha = \lim_{t \to \infty} \nvar_2(A^t)/\nvar_1(A^t) = \alpha_{\cdc_2(A), \cdc_1(A)}$.
\end{lemma}
\begin{proof}
\newcommand{\II}{{\mathcal{I}}}
\newcommand{\JJ}{{\mathcal{J}}}

Let $\II' \defeq \{i: \lno(\varphig(g^i)) < l\}$.

\paragraph{We show $v_i = \inv \alpha \la \fvar^{\varphi(h^i)}, \fvar^{\varphih(\gvar^l)}\ra$ for all $i \in \II'$.}
Note first of all that $\avar$ in \cref{defn:varphih} is defined by \cref{defn:detransposition} as $\alpha C_{\II'}^+ v_{\II'}$, where ,  $C_{\II'} = (C_{ij})_{i,j\in\II'}$, and $v_{\II'} = (v_i)_{i \in \II'}$.

Suppose $Z^\II \defeq \{Z^{\varphig(g^i)}\}_{i \in \II}$ is a maximal linearly independent set in $Z^{\II'} \defeq \{Z^{\varphig(g^i)}\}_{i \in \II'}$.
Note that $Z^{\II}$ (as well as $Z^{\II'}$) is also independent of all $Z^{\check g}$ where $\check g$ is produced by \ref{linetype:GLinear} involving a matrix that is not $\varphi(A')$ or where $\check g$ is an input var (by the construction of $K^{\check \cdc}$).
Let $Z'$ the collection of all such $Z^{\check g}$.
Then $\EV_{Z^\JJ} \fvar^{\varphi(\hvar^m)}(Z)$ is purely a function of $Z^\II$, by expressing other elements of $Z^{\II'}$ as linear combinations of $Z^{\II}$.
Thus, by \cref{lemma:steinMultivar} applied to $Z^\II$ and $\EV_{Z'} \fvar^{\varphi(\hvar^m)}$, there exist coefficients $\{a_j\}_{j \in \II}$ such that, for each $i \in \II$,
\begin{align*}
    v_i
        &=
            \EV_{Z^{\II}} Z^{\varphig(g^i)} \EV_{Z'| Z^{\II}} \fvar^{\varphi(\hvar^m)}(Z)
            =
            \EV_{Z^{\II}} Z^{\varphig(g^i)} \EV_{Z'} \fvar^{\varphi(\hvar^m)}(Z)
            \\
        &=
            \sum_{j \in \II} a_j K^{\check \cdc}(\varphig(g^i), \varphig(g^j))
            \\
        &=
            \la Z^{\varphig(g^i)}, \sum_{j \in \II} a_j Z^{\varphig(g^j)} \ra
            \\
        &=
            \f{(\sigma^{k\infty})^2} \alpha 
            \la \fvar^{\varphi(h^i)}, \sum_{j \in \II} a_j \fvar^{\varphi(h^j)} \ra
            \\
        & \pushright{\text{by construction of $K^{\cdc}$}}.
\end{align*}
This equality is extended to all $i \in \II'$ via linear combination.
Thus,
\begin{align*}
    (v_i)_{i \in \II'}
        &=
            \f{(\sigma^{k\infty})^2} \alpha 
            \la (\fvar^{\varphi(h^i)})_{i \in \II'}, \sum_{j \in \II} a_j \fvar^{\varphi(h^j)} \ra
            \\
    \fvar^{\varphih(\gvar^l)}
        &=
            \alpha \sum_{i,j \in \II'} \fvar^{\varphi(h^i)} (C_{\II'}^+)_{ij} v_j
            \\
        &=
            (\sigma^{k\infty})^2 
            \sum_{i,j \in \II'} \fvar^{\varphi(h^i)} (C_{\II'}^+)_{ij}
                \la \fvar^{\varphi(h^j)}, \sum_{j \in \II} a_j \fvar^{\varphi(h^j)} \ra
            \\
        &=
            (\sigma^{k\infty})^2
            \Pi_{\II'}
            \sum_{j \in \II} a_j \fvar^{\varphi(h^j)}
            \\
        &=
            (\sigma^{k\infty})^2\sum_{j \in \II} a_j \fvar^{\varphi(h^j)}
\end{align*}
where $\Pi_{\II'}$ is the projection operator on $L_2(Z)$ that projects to the linear span of $\{\fvar^{\varphi(h^i)}\}_{i \in \II'}$ (see \cref{defn:pseuodoinverse} and the basic facts underneath).

So all along, $v_i = \f 1 \alpha \la \fvar^{\varphi(h^i)}, \varphih(\gvar^l) \ra$ for all $i \in \II'$.

\paragraph{We show $v_i = \inv\alpha \la \fvar^{\varphi(h^i)}, \fvar^{\varphih(\gvar^l)}\ra$ for all $i \not\in \II'$.}

Suppose $g^i = A' h^i$ has line number greater than $l$.
Then, we have that, 
conditioned on $Z^{\II'}$, $\fvar^{\varphi(\hvar^m)}$ and $\fvar^{\varphig(g^i)}$ are independent.
Indeed, with the conditioning, the randomness in the former only comes from $\{Z^{\varphig(g^j)}\}_{j \not\in \II'}$ and the randomness in the latter only comes from $Z'$, and the two are independent.
Thus, 
\begin{align*}
    v_i 
        &=
            \EV_{Z^{\II'}}
                \lp \EV_{Z|Z^{\II'}} \fvar^{\varphig(g^i)}(Z) \rp
                \lp \EV_{Z|Z^{\II'}} \fvar^{\varphi(\hvar^m)}(Z) \rp
            \\
        &=
            \EV_{Z^{\II'}}
                \lp K^{\check \cdc}_{i \II'} (K^{\check \cdc}_{\II' \II'})^+ Z^{\II'} \rp
                \lp \EV_{Z|Z^{\II'}} \fvar^{\varphi(\hvar^m)}(Z) \rp
                \\
        &=
            K^{\check \cdc}_{i \II'} (K^{\check \cdc}_{\II' \II'})^+
            (\la Z^{\varphig(g^i)}, \fvar^{\varphi(\hvar^m)}\ra)_{i \in \II'}
\end{align*}
where $K^{\check \cdc}$ is the row vector $(K^{\check \cdc}(\varphig(g^i), \varphi(g^j)))_{j \in \II'}$ and $K^{\check \cdc}_{\II' \II'}$ is the submatrix $(K^{\check \cdc}(\varphig(g^i), \varphi(g^j)))_{i,j \in \II'}$.
Again by the construction of $K^{\check\cdc}$, this simplifies to
\begin{align*}
    v_i
        &=
            \sum_{j,k \in \II'} \la \fvar^{\varphi(h^i)}, \fvar^{\varphi(h^j)} \ra
                (C^+)_{jk} v_k
            \\
    v_i
        &=
            \la \fvar^{\varphi(h^i)}, \inv\alpha \fvar^{\varphih(\gvar^l)} \ra
\end{align*}

Therefore,
\begin{align*}
    &\phantomeq
        \sum_{i, j \in [s]}
            \fvar^{\varphi(h^i)} 
            (C^+)_{ij}
            v_j
            \\
    &=
        \inv\alpha\sum_{i, j \in [s]}
            \fvar^{\varphi(h^i)} 
            (C^+)_{ij}
            \la \fvar^{\varphi(h^j)}, \fvar^{\varphih(\gvar^l)} \ra
            \\
    &=
        \inv \alpha \Pi \fvar^{\varphih(\gvar^l)}
        \\
    &=
        \inv \alpha \fvar^{\varphih(\gvar^l)}.
\end{align*}
where $\Pi$ is the projection operator to the span of $\{\fvar^{\varphi(h^j)}\}_{j \in [s]}$, and the last equality follows because $\fvar^{\varphih(\gvar^l)}$ is already in this span.

\end{proof}

\generalTensorP*

\begin{proof}

We proceed by induction on line number of $\pi$.
All line types are trivial except \ref{linetype:GLinear}.
So suppose in $\pi$, line $l$ is $\gvar^l := A \hvar^m$, and 
the induction hypothesis holds for $l$ and $\cdc = \cdc(\gvar^l)$.
Set $\cdc_1 = \cdc$ and $\cdc_2 = \cdc(\hvar^m)$.
Let $A'$ be the A-var such that $A' := A^\trsp$ if $A$ is an input A-var, and $A := A'{}^\trsp$ otherwise.

Let $g^i := A h^i, i = 1, \ldots, r,$ be all \ref{linetype:GLinear} lines involving $A$ that appear before line $l$, and let $g'{}^i := A' h'{}^i, i = 1, \ldots, s,$ be all \ref{linetype:GLinear} lines involving $A'$ that appear before line $l$.
Note that $\cdc(g^i) = \cdc(h'{}^i) = \cdc_1$ and $\cdc(h^i) = \cdc(g'{}^i) = \cdc_2$.
Set $H^t = [h^{1t}|\cdots|h^{rt}] \in \R^{\nvar^{\cdc_2 t} \times r}$, and likewise for 
$G^t \in \R^{\nvar^{\cdc_1 t} \times r},
H'{}^t \in \R^{\nvar^{\cdc_1 t} \times s},
G'{}^t \in \R^{\nvar^{\cdc_2 t} \times r}.$
Let $\Aa$ be the $\sigma$-algebra generated by all vector vars before $\gvar^l$.

As in the proof of \cref{thm:gradIndep},
$$g^{l t} \disteq_{\Aa^t} \mu^t + \Pi^\perp_{H'{}^t} \tilde A^t \Pi^\perp_{H^t} \hvar^{mt}$$
where $\tilde A^t$ is random matrix sampled iid as $A$, and
$
\mu^t = 
    G^t \Sigma^{t}{}^+ \omega^t 
    + \f{\nvar^{\cdc_2 t}}{\nvar^{\cdc_1 t}} H'{}^t \Sigma'{}^t{}^+ \beta^t 
    - \f{\nvar^{\cdc_2 t}}{\nvar^{\cdc_1 t}} H'{}^t \Sigma'{}^t{}^+ \Upsilon^t \Sigma^t{}^+ \omega^t$ and
\begin{align*}
    \Sigma^t
        &\defeq H^t {}^\trsp H^t / \nvar^{\cdc_2 t}
            \in \R^{r \times r}
            &
    \Sigma'{}^t
        &\defeq H'{}^t{}^\trsp H'{}^t / \nvar^{\cdc_1 t}
            \in \R^{s \times s}
            &
    \Upsilon^t
        &\defeq
            G'{}^t{}^{\trsp} H^t / \nvar^{\cdc_2 t}
            \in \R^{s \times r}
            \\
    \omega^t
        &\defeq
            H^t{}^\trsp \hvar^m / \nvar^{\cdc_2 t}
            \in \R^{r}
            &
    \beta^t
        &\defeq
            G'{}^t{}^\trsp \hvar^m/ \nvar^{\cdc_2 t}
            \in \R^{s}.
\end{align*}

By induction hypothesis, $\Sigma^t, \Sigma'{}^t, \Upsilon^t, \omega^t, \beta^t$ all converge almost surely to corresponding limit values:
Let $\alpha = \alpha_{\cdc_2, \cdc_1} = \lim_{t \to \infty} \f{\nvar^{\cdc_2 t}}{\nvar^{\cdc_1 t}}$.
With $Z \sim \Gaus(\mu^{\check \cdc}, K^{\check \cdc})$,
\begin{align*}
    \Sigma^t_{ij} 
        &\asto
            \Sigma^\infty_{ij} \defeq
            \EV \fvar^{\varphi(h^i)}(Z)\fvar^{\varphi(h^j)}(Z)
            \\
    \Sigma'{}^t_{ij}
        &\asto
            \Sigma'{}^\infty_{ij} \defeq
            \EV \fvar^{\varphi(h'{}^i)}(Z)\fvar^{\varphi(h'{}^j)}(Z)
            \\
    \omega^t_i
        &\asto
            \omega^\infty_i \defeq
            \EV \fvar^{\varphi(h^i)}(Z) \fvar^{\varphi(\hvar^m)}(Z)
            \\
    \Upsilon^t_{ij}
        &\asto
            \Upsilon^\infty_i \defeq
            \EV \fvar^{\varphi(g'{}^i)}(Z) \fvar^{\varphi(h^j)}(Z)
            \\
    \beta^t_i
        &\asto
            \beta^\infty_i \defeq
            \EV \fvar^{\varphi(g'{}^i)}(Z) \fvar^{\varphi(\hvar^m)}(Z).
\end{align*}

As in the proof of \cref{thm:gradIndep}, we can apply \cref{thm:controlHighMoments} and Gaussian average smoothness to integrate out $\tilde A^t$ and obtain the following claim
\begin{claim}
With $w \sim \Gaus(0, 1),$
$$\f 1 {\nvar^{\cdc t}} \sum_{i=1}^{\nvar^{\cdc t}} \phi(\gvar^{lt}_i, \hvar^{\cdclt l t}_i) - \EV_w \phi(G^t_{i:} \avar + H'{}^t_{i:} \avar' + \sigma^\infty w, \hvar^{\cdclt l t}_i) \asto 0$$
where
\begin{align*}
\avar &\defeq \Sigma^\infty{}^+ \omega^\infty   \\
\avar' &\defeq \alpha \Sigma'{}^\infty {}^+ (\beta^\infty - \Upsilon^\infty \Sigma^\infty{}^+ \omega^\infty)\\
\sigma^\infty &\defeq K^{\check \cdc}(\varphig(\gvar^l), \varphig(\gvar^l)) - K^{\check \cdc}(\varphig(\gvar^l), \varphig(G)) K^{\check \cdc}|_{\varphig(G)}^+ K^{\check \cdc}(\varphig(G), \varphig(\gvar^l))\\
\end{align*}
with $\check \cdc = \cdc(\varphig(\gvar^l))$.
\end{claim}

Combining this with the induction hypothesis and \cref{thm:notransposeLimit}, we have
\begin{claim}
With $w \sim \Gaus(0, 1)$ and $Z \sim \Gaus(\mu^{\check \cdc}, K^{\check \cdc})$,
$$
    \f 1 {\nvar^{\cdc t}} \sum_{i=1}^{\nvar^{\cdc t}}
        \phi(\gvar^{lt}_i, \hvar^{\cdclt l t}_i)
        - \EV_{w, Z} \phi\lp \sum_{j=1}^r \avar_j \fvar^{\varphi(g^j)}(Z) 
                    + \sum_{j'=1}^s \avar'_{j'} \fvar^{\varphi(h'{}^{j'})}(Z)
                    + \sigma^\infty w,
                    \{\fvar^{\varphi(h)}(Z)\}_{h \in (\bar \cdc)_{< l}}
                    \rp
    \asto 0
$$
\end{claim}

In the following, we consider the inner product space of $L^2$-integrable functions of $Z$ (equivalently, square-integrable random variables in the $\sigma$-algebra generated by $Z$), with inner product $\la f, g \ra \defeq \EV f(Z) g(Z)$ with $Z \sim \Gaus(\mu^{\check \cdc}, K^{\check \cdc})$.
We abuse notation and let any vector-var in $\check \pi$ (e.g. $\check \hvar^m$) denote the corresponding function (e.g. $\fvar^{\check \hvar^m}$).
Note that we can rewrite
\begin{align*}
    &\phantomeq
        \beta^\infty - \Upsilon^\infty \Sigma^\infty{}^+ \omega^\infty
        \\
    &=
        \left(\la \varphi(g'{}^i), \varphi(\hvar^m) \ra
        - \sum_{j,k}\la \varphi(g'{}^i), \varphi(h^j) \ra (\Sigma^\infty{}^+)_{jk} \la \varphi(h^k), \varphi(\hvar^m) \ra    
        \right)_{i \in s} 
        \\
    &=
        (\la \Pi_{\varphi(H)}^\perp \varphi(g'{}^i), \varphi(\hvar^m)\ra)_{i \in s}
    =
        (\la \varphi(g'{}^i), \Pi_{\varphi(H)}^\perp \varphi(\hvar^m)\ra)_{i \in s}
\end{align*}
where $\Pi_{\varphi(H)}$ is the projection operator to the span of $\varphi(H) \defeq \{\varphi(h^i)\}_{i=1}^r$ and $\Pi_{\varphi(H)}^\perp$ is its orthogonal complement.

\begin{claim}
\begin{align*}
    \sum_{j=1}^r \avar_j \varphih(g^j)
        &=
            \alpha \sum_{i,j=1}^s \varphi(h'{}^i)(\Sigma'{}^\infty{}^+)_{ij} \la \Pi_{\varphi(H)} {\varphig(g'{}^j)}, {\varphi(\hvar^m)}\ra.
\end{align*}

\end{claim}

\begin{claimproof}

By \cref{lemma:Hpart},
\begin{align*}
    \sum_{j=1}^r \avar_j \varphih(g^j)
        &=
            \alpha \sum_{j,q=1}^r 
            (\Sigma^\infty{}^+)_{jq} \la \varphi(h^q), \varphi(\hvar^m) \ra
            \sum_{a, b \in [s]}
            \varphi(h'{}^a)
            (\Sigma'{}^\infty{}^+)_{ab}
            \la \varphig(g'{}^b), \varphi(h^j) \ra
            \\
        &=
            \alpha
            \sum_{a, b \in [s]}
            \sum_{j,q\in[r]}
            \varphi(h'{}^a)
            (\Sigma'{}^\infty{}^+)_{ab}
            \la \varphig(g'{}^b), \varphi(h^j) \ra
            (\Sigma^\infty{}^+)_{jq} \la \varphi(h^q), \varphi(\hvar^m) \ra
            \\
        &=
            \alpha
            \sum_{a, b \in [s]}
            \varphi(h'{}^a)
            (\Sigma'{}^\infty{}^+)_{ab}
            \la \varphig(g'{}^b), \Pi_{\varphi(H)} \varphi(\hvar^m) \ra
\end{align*}
as desired.
\end{claimproof}

\begin{claim}
$$
\sum_{j=1}^r \avar_j \varphih(g^j) + \sum_{j'=1}^s \avar'_{j'} \varphi(h'{}^{j'})
= \alpha \sum_{i,j=1}^s \varphi(h'{}^i)(\Sigma'{}^\infty{}^+)_{ij} 
    \la {\varphig(g'{}^j)}, {\varphi(\hvar^m)}\ra.
$$
\end{claim}
\begin{claimproof}
As noted above,
\begin{align*}
    \avar'_{j'}
        &=
            \alpha \sum_{k \in [s]} (\Sigma'{}^\infty{}^+)_{j' k}
            (\beta^\infty - \Upsilon^\infty \Sigma^\infty{}^+ \omega^\infty)_k
            \\
        &=
            \alpha \sum_{k \in [s]} (\Sigma'{}^\infty{}^+)_{j' k} 
            \la \Pi_{\varphi(H)}^\perp \varphi(g'{}^i), \varphi(\hvar^m)\ra
            \\
        &=
            \alpha \sum_{k \in [s]} (\Sigma'{}^\infty{}^+)_{j' k} 
            \la \Pi_{\varphi(H)}^\perp \varphig(g'{}^i), \varphi(\hvar^m)\ra
            \\
        &\pushright{\text{because $\varphi(g'{}^i) - \varphig(g'{}^i) \in \lspan \varphi(H)$}}\\
    \sum_{j'=1}^s \avar'_{j'} \varphi(h'{}^{j'})
        &=
            \alpha \sum_{j', k \in [s]} \varphi(h'{}^{j'}) (\Sigma'{}^\infty{}^+)_{j' k} 
            \la \Pi_{\varphi(H)}^\perp \varphig(g'{}^i), \varphi(\hvar^m)\ra
            .
\end{align*}
By the previous claim, adding $\sum_{j=1}^r \avar_j \varphih(g^j)$ cancels $\Pi_{\varphi(H)}^\perp$ and gives the desired result.
\end{claimproof}

Therefore,
\begin{align*}
    &\phantomeq
        \sum_{j=1}^r \avar_j \varphi(g^j) 
                    + \sum_{j'=1}^s \avar'_{j'} \varphi(h'{}^{j'})
                    + \sigma^\infty w
        \\
    &=
        (\sigma^\infty w + \sum_{j=1}^r \avar_j \varphig(g^j) )
        + 
        \alpha \sum_{i,j=1}^s \varphi(h'{}^i)(\Sigma'{}^\infty{}^+)_{ij} 
        \la {\varphig(g'{}^j)}, {\varphi(\hvar^m)}\ra
        \\
    &\disteq_{\mathcal H_{<l}}
        \varphig(\gvar^l) + \varphih(\gvar^l)
        = \varphi(\gvar^l)
\end{align*}
where $\mathcal H_{<l}$ is the $\sigma$-algebra generated by $\{\fvar^{\varphi(h)}(Z)\}_{h \in (\bar \cdc)_{< l}},$
and $(\bar \cdc)_{<l}$ is the collection of all vector vars in $\bar \cdc$ with line number $<l$.
So we can complete our induction by stating
$$\f 1 {\nvar^{\cdc t}} \sum_{i=1}^{\nvar^{\cdc t}}
        \phi(\gvar^{lt}_i, \hvar^{\cdclt l t}_i)
        - \EV_{w, Z} \phi(\fvar^{\varphi(\gvar^l)}(Z), \{\fvar^{\varphi(h)}(Z)\}_{h \in (\bar \cdc)_{< l}})
    \asto 0.$$
    
\begin{center}
----------------
\end{center}

For the second claim, let $C_{ij} = \la \varphi(h'{}^i), \varphi(h'{}^j) \ra$ for all $i, j \in [s]$.
We can compute, as in the proof of \cref{lemma:Hpart,lemma:steinMultivar},
\begin{align*}
    \varphih(\gvar^l)
        &=
            \alpha \sum_{i,j \in [s]} \varphi(h'{}^i) (C^+)_{ij} \la \varphig(g'{}^j), \hvar^m \ra
            \\
        &=
            \alpha \sum_{i,j, k \in [s]} \varphi(h'{}^i) (C^+)_{ij} K^{\check \cdc}(\varphig(g'{}^j), \varphig(g'{}^k)) \EV \pd_{Z^{\varphig(g'{}^k)}} \fvar^{\hvar^m}(Z)
            \\
        &\pushright{\text{by Stein's Lemma \cref{lemma:stein}}}
            \\
        &=
            \alpha \sum_{i,j,k \in [s]} \varphi(h'{}^i) (C^+)_{ij} \sigma^2 
                \la \varphi(h'{}^j), \varphi(h'{}^k)\ra
                \EV \pd_{Z^{\varphig(g'{}^k)}} \fvar^{\hvar^m}(Z)
            \\
        &=
            \alpha  \sigma^2 \Pi_{\varphi(H)} \sum_{k \in [s]} 
                \varphi(h'{}^k)
                \EV \pd_{Z^{\varphig(g'{}^k)}} \fvar^{\hvar^m}(Z)
            \\
        &=
            \alpha  \sigma^2 \sum_{k \in [s]} 
                \varphi(h'{}^k)
                \EV \pd_{Z^{\varphig(g'{}^k)}} \fvar^{\hvar^m}(Z)
\end{align*}
where $\sigma = \sigma^{r\infty}$ and $r = \lno(\varphi(A'))$.
This computation goes through as long as $\fvar^{\hvar^m}$ is differentiable, which is implied by all $\fvar^l$ in $\pi$ being differentiable, or if the covariance $K^{\check \cdc}(\varphig(G'), \varphig(G'))$ is nondegenerate (which allows us to consider $\fvar^{\hvar^m}$, a polynomially bounded function, as a tempered distribution, whose derivatives are also tempered distributions, giving a valid interpretation to the expectation).

\newcommand{\II}{\mathcal{I}}
When $K^{\check \cdc}(\varphig(G'), \varphig(G'))$ is singular, let $\II$ be a minimal subset of $[s]$ such that $\{\varphi(h'{}^i)\}_{i \in \II}$ is linearly independent.
We can compute similarly,
\begin{align*}
    \varphih(\gvar^l)
        &=
            \alpha \sum_{i,j \in \II} \varphi(h'{}^i) (C_\II^+)_{ij} \la \varphig(g'{}^j), \hvar^m \ra
            \\
        &=
            \alpha \sum_{i,j, k \in \II} \varphi(h'{}^i) (C_\II^+)_{ij} K^{\check \cdc}(\varphig(g'{}^j), \varphig(g'{}^k)) \EV \pd_{Z^{\varphig(g'{}^k)}} \fvar_\II^{\hvar^m}(Z)
            \\
        &\pushright{\text{by Stein's Lemma \cref{lemma:stein}}}
            \\
        &=
            \alpha \sum_{i,j,k \in \II} \varphi(h'{}^i) (C_\II^+)_{ij} \sigma^2 
                \la \varphi(h'{}^j), \varphi(h'{}^k)\ra
                \EV \pd_{Z^{\varphig(g'{}^k)}} \fvar_\II^{\hvar^m}(Z)
            \\
        &=
            \alpha  \sigma^2 \Pi_{\varphi(H)} \sum_{k \in \II} 
                \varphi(h'{}^k)
                \EV \pd_{Z^{\varphig(g'{}^k)}} \fvar_\II^{\hvar^m}(Z)
            \\
        &=
            \alpha  \sigma^2 \sum_{k \in \II} 
                \varphi(h'{}^k)
                \EV \pd_{Z^{\varphig(g'{}^k)}} \fvar_\II^{\hvar^m}(Z)
\end{align*}
where $c_\II$ is the restriction of $C$ to $\II$, and $\fvar^{\hvar^m}_\II$ is the version of $\fvar^{\hvar^m}$ that expands $Z^{\varphig(g'{}^j)}, j \not\in \II$ to linear combinations of $\{Z^{\varphig(g'{}^i)}\}_{i \in \II}$.
Then this computation goes through always, since $\{Z^{\varphig(g'{}^i)}\}_{i \in \II}$ has a density.

\end{proof}

\end{document}